\documentclass[11 pt]{article}
\usepackage{style}

\title{\LARGE\bfseries Concentration of Contractive Stochastic Approximation: Additive and Multiplicative Noise}
\author{Zaiwei Chen\textsuperscript{1} Siva Theja Maguluri\textsuperscript{2}, and Martin Zubeldia\textsuperscript{3}\\
{\small
\textsuperscript{1}\textit{School of Industrial Engineering, Purdue University,} \href{mailto:chen5252@purdue.edu}{\textit{chen5252@purdue.edu}}}\\
{\small\textsuperscript{2}\textit{School of Industrial \& Systems Engineering, Georgia Institute of Technology,} \href{mailto:siva.theja@gatech.edu}{\textit{siva.theja@gatech.edu}}}\\
{\small\textsuperscript{3}\textit{Department of Industrial and Systems Engineering, University of Minnesota,} \href{mailto:zubeldia@umn.edu}{\textit{zubeldia@umn.edu}}}
}
\date{\vspace{-0.4 in}}

\begin{document}
\maketitle

\begin{abstract}
In this paper, we establish maximal concentration bounds for the iterates generated by a stochastic approximation (SA) algorithm under a contractive operator with respect to some arbitrary norm (for example, the $\ell_\infty$-norm). We consider two settings where the iterates are potentially unbounded: SA with bounded multiplicative noise and SA with sub-Gaussian additive noise. Our maximal concentration inequalities state that the convergence error has a sub-Gaussian tail in the additive noise setting and a Weibull tail (which is faster than polynomial decay but could be slower than exponential decay) in the multiplicative noise setting. In addition, we provide an impossibility result showing that it is generally impossible to have sub-exponential tails under multiplicative noise. To establish the maximal concentration bounds, we develop a novel bootstrapping argument that involves bounding the moment-generating function of a modified version of the generalized Moreau envelope of the convergence error and constructing an exponential supermartingale to enable using Ville's maximal inequality. We demonstrate the applicability of our theoretical results in the context of linear SA and reinforcement learning.
\end{abstract}

\tableofcontents

\section{Introduction}\label{sec:intro}
The stochastic approximation (SA) method, first proposed in \cite{robbins1951stochastic}, has become a foundational tool for modern large-scale optimization and machine learning, which have achieved great success in solving many practical problems across different domains \cite{kober2013reinforcement,silver2017mastering,jumper2021highly,ouyang2022training}. More formally, SA is an iterative method for solving systems of equations, in particular, fixed-point equations of the form $\bar{F}(x)=x$, where $\bar{F}:\mathbb{R}^d\mapsto\mathbb{R}^d$ is a (possibly nonlinear) operator. When the explicit expression of the operator $\bar{F}(\cdot)$ is accessible, such an equation can be efficiently solved with the fixed-point iteration $x_{k+1}=\bar{F}(x_k)$, provided that $\bar{F}(\cdot)$ is a contraction mapping \cite{banach1922operations}, i.e., there exist $\gamma_c\in (0,1)$ and a norm $\|\cdot\|_c$ such that $\|\bar{F}(x_1)-\bar{F}(x_2)\|_c\leq \gamma_c\|x_1-x_2\|_c$ for all $x_1,x_2\in\mathbb{R}^d$. However, in many practical applications, such as in reinforcement learning (RL) or large-scale optimization, we do not have enough knowledge or enough computational power to accurately compute $\bar{F}(x)$ for a given $x$, which is needed to carry out the fixed-point iteration. To overcome this challenge, SA was proposed as a small-stepsize and data-driven stochastic variant of the fixed-point iteration, which takes the following form: 
\begin{align}\label{eq:stochastic-approximation}
x_{k+1} = x_k + \alpha_k(F(x_k,Y_k) - x_k),
\end{align}
where $Y_k$ is a random variable taking values in a probability space $\mathcal{Y}$, $F:\mathbb{R}^d\times\mathcal{Y}\mapsto\mathbb{R}^d$ is another operator (which can be viewed as a noisy version of $\bar{F}(\cdot)$), and $\alpha_k>0$ is the stepsize. The stepsize is usually chosen as $\alpha_k=\alpha/(k+h)^z$, where $\alpha,h>0$ and $z\in (0,1]$. We assume that the noisy estimate $F(x_k,Y_k)$ of $\bar{F}(x_k)$ is conditionally unbiased, i.e., $\mathbb{E}[F(x_k,Y_k)\mid \mathcal{F}_k]=\bar{F}(x_k)$ for all $k\geq 0$, where $\mathcal{F}_k$ is the $\sigma$-algebra generated by $\{x_0,Y_0,Y_1,\cdots,Y_{k-1}\}$.

The SA algorithm presented in Eq. (\ref{eq:stochastic-approximation}) covers many existing popular algorithms as its special cases. For example, when $\bar{F}(x)=-c\nabla J(x)+x$ for some strongly convex and smooth objective function $J:\mathbb{R}^d\mapsto \mathbb{R}$, where $c>0$ can be arbitrary, the algorithm reduces to the stochastic gradient descent (SGD) used to minimize $J(\cdot)$ \cite{lan2020first,bottou2018optimization}, which can also be modeled as a contractive SA \cite{ryu2016primer}. In the context of RL, popular algorithms such as $Q$-learning \cite{watkins1992q} and temporal-difference (TD)-learning \cite{sutton1988learning} can all be modeled in the form of Eq.(\ref{eq:stochastic-approximation}) \cite{bertsekas1996neuro}, where the operator $\bar{F}(\cdot)$ is closely related to the Bellman operator \cite{bellman1957dynamic}. Due to the wide applications of SA, theoretically understanding the convergence behavior of the sequence $\{x_k\}$ generated by Eq. (\ref{eq:stochastic-approximation}) is of fundamental interest.

Early literature on SA focused on the asymptotic convergence \cite{robbins1951stochastic,borkar2009stochastic,tsitsiklis1994asynchronous,kushner2012stochastic}. In recent years, finite-sample analysis has received considerable attention \cite{bhandari2018finite,srikant2019finite,chen2020finite}. In finite-sample analysis, the goal is to bound the error between the stochastic iterate $x_k$ and its limit $x^*$ (which satisfies $\bar{F}(x^*)=x^*$) as a function of the number of iterations $k$ and to study its decay rate. Compared with asymptotic analysis, finite-sample analysis not only provides a more refined characterization of the behavior of the SA iterates but also can be used as a guide in implementation. 

Due to the stochastic nature of the iterates, there are multiple ways of measuring the distance between the iterates and the limit point. One natural way is to use the mean-square distance $\mathbb{E}[\|x_k-x^*\|_c^2]$, which has been extensively studied in the literature \cite{srikant2019finite,bhandari2018finite,chen2019finitesample,chen2020finite,wainwrightHDS}. Another popular way is to use the probability that $\|x_k-x^*\|_c\leq \epsilon$ for some $\epsilon>0$. A bound on this probability is called a ``high-probability bound'' and is sometimes preferable over a mean-square bound as it not only provides the convergence rate but also the confidence level. However, high-probability bounds are, in general, more challenging to establish. For example, consider the convergence rate of the law of large numbers\footnote{The average of a sequence of random variables $\frac{1}{k}\sum_{i=0}^{k-1}Y_k$ can be computed iteratively as $x_{k+1}=x_k+\frac{1}{k+1}(-x_k+Y_k)$ with $x_0=0$, which is a special case of the SA algorithm presented in Eq. (\ref{eq:stochastic-approximation}).}. The establishment of the $\mathcal{O}(1/k)$ mean-square bound is significantly easier than establishing exponential tail bounds such as Hoeffding's inequality \cite{hoeffding1994probability}, Chernoff bound \cite{chernoff1952measure}, and Bernstein's inequality \cite{bernstein1924modification}, etc.

To establish high-probability bounds of $\{x_k\}$ generated by Eq. (\ref{eq:stochastic-approximation}), the properties of the random process $\{F(x_k,Y_k)\}$ (in addition to being conditionally unbiased) play important roles in the analysis. In this work, we consider two types of noise sequences where the iterates $\{x_k\}$ could be unbounded: bounded multiplicative noise (i.e., there exists $\sigma>0$ such that $\| F(x_k,Y_k) - \mathbb{E}[F(x_k,Y_k)\mid \mathcal{F}_k] \|_c \leq \sigma(1+\|x_k\|_c)$ for all $k\geq 0$) or sub-Gaussian additive noise (i.e., the random vector $F(x_k,Y_k)-\mathbb{E}[F(x_k,Y_k)\mid \mathcal{F}_k]$ is norm sub-Gaussian). In the existing literature, most results focused on the setting where the noise in the SA algorithm is additive and is almost surely (a.s.) bounded. To illustrate, consider a linear SA of the form $x_{k+1}=x_k+\alpha_k(A(Y_k)x_k-b(Y_k))$, where $A:\mathcal{Y}\mapsto\mathbb{R}^{d\times d }$ and $b:\mathcal{Y}\mapsto\mathbb{R}^d$ are deterministic functions. This corresponds to Eq. (\ref{eq:stochastic-approximation}) with $F(x,y)=c(A(y)x-b(y))+x$ for any $c>0$.
When $A(Y_k)$ is not random, i.e., $A(Y_k)=\mathbb{E}[A(Y_k)]$, the noise is purely additive. In the multiplicative noise setting, which corresponds to $A(Y_k)$ being random, the analysis is much more challenging. Existing high-probability bounds either have tails that do not decay faster than polynomials or require strong assumptions, such as $A(Y_k)$ being Hurwitz a.s. See Section \ref{subsec:literature} for a more detailed literature review. 

In this paper, we develop maximal concentration bounds with Weibull tails for contractive SA algorithms in both the additive and the multiplicative noise setting. The main contributions of this work are summarized below. See also Table \ref{tab:summary} for a concrete summary of our results.

\begin{table}[ht!]
\setlength{\arrayrulewidth}{0.3mm}
\renewcommand{\arraystretch}{1.9}
\begin{center}
\begin{tabular}{ |m{1.8cm}||m{2cm}|m{5.5cm}|m{3.7cm}|} 
 \hline
\multirow{2}{*}{} &   \multicolumn{2}{c|}{SA with Multiplicative Noise} &  \multirow{2}{*}{SA with Additive Noise}\\ 
 \cline{2-3}
&   \centering $D\leq 0$ 
&  \centering $D>0$ &\\  
 \hline\hline
 \multirow{ 2}{*}{$\quad z=1$}   & \multirow{2}{*}{$\tilde{\mathcal{O}}\left(\frac{\log(1/\delta)}{k}\right)$}   & \centering Bound: $\tilde{\mathcal{O}}\left(\frac{\log^{\lceil2\alpha D\rceil+1}(1/\delta)}{k}\right)$& \multirow{2}{*}{$\quad \;\;\tilde{\mathcal{O}}\left(\frac{\log(1/\delta)}{k}\right)$}\\
 && \centering Impossibility Result: $\tilde{\Omega}\left(\frac{\log^{2\alpha D+1}(1/\delta)}{k}\right)$&\\[2.5 ex]
 \hline
 $z\in (0,1)$   & \multicolumn{2}{c|}{Impossibility Result: Weibull tail is not achievable.} &  $\quad \;\;\tilde{\mathcal{O}}\left(\frac{\log(1/\delta)}{k^z}\right)$\\[0.8ex]
 \hline
\end{tabular}
\end{center}
\caption{Summary of our results: Bounds on $\|x_k-x^*\|_c^2$ with probability at least $1-\delta$}
\label{tab:summary}
\begin{flushleft}
    \textit{Remark: The parameter $z$ is the exponent in the stepsize $\alpha_k=\alpha/(k+h)^z$. The parameter $D=\sigma+\gamma_c-1$ depends on the contraction factor $\gamma_c$ of the operator $\bar{F}(\cdot)$ and the parameter $\sigma$ of the multiplicative noise. We use $\lceil x\rceil$ to denote the smallest integer larger than $x$. In $\tilde{\mathcal{O}}(\cdot)$ and $\tilde{\Omega}(\cdot)$, logarithmic factors are ignored.}
\end{flushleft}
\end{table}

\begin{itemize}
    \item \textit{Multiplicative Noise Setting.} We establish high-probability bounds with Weibull tails when using diminishing stepsizes of the form $\alpha_k=\alpha/(k+h)$, where $\alpha,h>0$. Importantly, our result provides a bound on the entire tail of the iterates, as our stepsizes do not depend on either the desired accuracy level or the probability tolerance level. Moreover, our bound is ``maximal'' in the sense that it is a bound on the concentration behavior of the entire trajectory of the iterates $\{x_k\}$. As a complement of the concentration bounds, we provide impossibility results showing that concentration bounds with sub-exponential tails are, in general, not achievable when using $\alpha_k=\alpha/(k+h)$. In addition, when using $\alpha_k=\alpha/(k+h)^z$ with $z\in (0,1)$ as the stepsize, even Weibull tails are not achievable.  To our knowledge, this is the first maximal concentration bound (with a Weibull tail) for SA with multiplicative noise. Even in the simple setting of linear SA (with a random $A(Y_k)$ that is not a.s. Hurwitz), such a concentration result is unknown in the literature.
    \item \textit{Additive Noise Setting.} We also consider the case of purely additive noise. We allow the noise to be unbounded, albeit sub-Gaussian. In this case, we establish maximal concentration bounds with sub-Gaussian tails for the SA algorithm when using stepsizes of the form $\alpha_k=\alpha/(k+h)^z$ for any $z\in (0,1]$. To our knowledge, except for the special case of SGD, such concentration results in the case of additive but unbounded noise are unknown in the literature.
    \item \textit{Methodological Contributions.} To overcome the challenge of having multiplicative noise and potentially unbounded iterates in the SA algorithm, we develop a novel bootstrapping argument that involves (1) the establishment of a bound on the moment-generating function (MGF) of a properly modified variant of the generalized Moreau envelope of the convergence error, which serves as a Lyapunov function in our analysis, and (2) the construction of an exponential supermartingale and the use of Ville's maximal inequality. Next, we provide more details about the challenges and our technical contributions.
\end{itemize}

\subsection{Challenges \& Our Techniques}\label{subsec:challenge}

We use SA with multiplicative noise as an example to illustrate the challenges and our techniques. The analysis of SA with additive noise follows a similar approach.

The main challenge of obtaining high-probability bounds with tails decaying faster than polynomials is due to the combination of \textit{unbounded iterates} and \textit{multiplicative noise}. While these two components are not too problematic in isolation, the combination of both creates a vicious circle where the variance of the noise can be unbounded. In this case, while the expected operator $\bar{F}(x_k)$ is contracting, since the ``noisy'' operator $F(x_k,Y_k)$ can be expansive with a positive probability, the error can grow extremely fast with a significant probability. This creates a challenge that no approach in the literature can deal with in general. 

To overcome this challenge, we develop a bootstrapping argument, which is in spirit to the mathematical induction proof technique. Specifically, we first show that the iterates of the SA algorithm, while not uniformly bounded, admit a time-varying a.s. bound, which could be polynomially increasing. This is similar to the base case in an induction argument. To proceed with the induction step, suppose that a \textit{non-decreasing} bound holds with some probability (which is the induction hypothesis). Then, we show that a tighter bound must hold with a slightly larger probability. This serves as a blueprint for the iterative refinement of the bounds. Finally, we start with the worst-case a.s. bound and repeatedly use the induction blueprint (for finitely many times) to finish the proof. Next, we elaborate on the $3$-step proof idea in more detail.

\textit{Step $1$: Time-Varying Worst-Case Bounds.} Although the iterates of SA with multiplicative noise are not uniformly bounded by a constant, we show that they do admit a time-varying bound. The behavior of such a time-varying bound depends on the contraction effect in the expected operator and the expansive effect in the multiplicative noise. In general, the bound can be polynomially \textit{increasing} with time. This time-varying worst-case bound serves as the base case in our bootstrapping argument. 

\textit{Step $2$: An Iterative Framework to Improve the Bound.} The key to our bootstrapping argument is to establish the following proposition for induction.

\begin{proposition}\label{prop:blueprint_intro}
    Given a tolerance level $\delta\in (0,1)$, suppose that there exists a non-decreasing sequence $\{T_k(\delta)\}$ such that $\mathbb{P}(\|x_k-x^*\|_c^2\leq T_k(\delta),\forall\;k\geq 0)\geq 1-\delta$. Then, for any $\delta'\in (0,1-\delta)$, there must exist a sequence $\{T_k(\delta,\delta')\}$ with $T_k(\delta,\delta')=\tilde{\mathcal{O}}(T_k(\delta)/k)$) such that $\mathbb{P}(\|x_k-x^*\|_c^2\leq T_k(\delta,\delta'),\forall\;k\geq 0)\geq 1-\delta-\delta'$.
\end{proposition}

This result enables us to start with the time-varying worst-case bound for the error (which can be polynomially increasing) and iteratively improve it to obtain our concentration bound with a Weibull tail and the desired convergence rate. To establish Proposition \ref{prop:blueprint_intro}, we use a Lyapunov approach, which consists of the following two steps. 
\begin{itemize}
    \item \textit{Step 2.1: A Recursive Bound on the Log-MGFs.} The first step is to obtain a recursive upper bound on the log-MGF of a modified variant of the generalized Moreau envelope of the norm-square function $\|x_k-x^*\|_c^2$. Opening this recursion, we also obtain an outright bound on $\|x_k-x^*\|_c^2$ that only depends on $\|x_0-x^*\|_c^2$ and other model parameters. These bounds are valid for all $k\geq 0$ and give us a tight grasp on the effect of the noise on the error.
    \item \textit{Step 2.2: The Construction of an Exponential Supermartingale.} We construct a supermartingale $\{\overline{M}_k\}_{k\geq 0}$ of the form $\overline{M}_k=\exp(\|x_k-x^*\|_c^2 \alpha_k^{-1}T_k(\delta)^{-1}- C \sum_{i=0}^{k-1} \alpha_k ) $, where $C>0$ is a properly chosen constant, and then use Ville's maximal inequality \cite{ville1939etude} to obtain a maximal bound on the iterates. In particular, this maximal bound states that $\|x_k-x^*\|_c^2=\tilde{\mathcal{O}}(\alpha_k T_k(\delta))$ for all $k \geq 0$ with high probability. Since we use $\alpha_k=\mathcal{O}(1/k)$ in the multiplicative noise setting, the induction blueprint is established. 
\end{itemize}

\textit{Step $3$: Completing the Bootstrapping Argument.} The final step in proving our maximal concentration bounds is to use the worst-case bound of the convergence error as a starting point and repeatedly apply the induction step to iteratively improve the bound. Note that in contrast to the classical induction argument, which can be applied infinitely many times, since our induction blueprint requires the initial bound $\{T_k(\delta)\}$ to be non-decreasing, it can only be applied for finitely many times.

\subsection{Related Literature}\label{subsec:literature}

Before presenting our problem setting and our main results, we first summarize related work on establishing concentration bounds of SA algorithms in the form of SGD, linear SA, and RL algorithms.

\textit{Stochastic Gradient Descent.} There is a large body of work on exponential high-probability bounds for SGD and its variants. In \cite{rakhlin2012,Hazan14}, the authors obtain exponential high-probability bounds for SGD with non-smooth but strongly convex objective functions when the noise is conditionally unbiased and the iterates are in a compact set. This was later generalized to the case of sub-Gaussian noise and unbounded iterates in \cite{Harvey19}, making it one of the rare cases where exponential high-probability bounds are obtained with unbounded noise. Exponential high-probability bounds were also obtained for the ergodic mirror descent (under Markovian, conditionally biased noise with uniformly bounded variance) in \cite{duchi2012ergodic}, under the additional assumption that the iterates are in a compact set. 
More recently, polynomial high-probability bounds have been obtained in \cite{heavySGD} for SGD on linear models when the noise is heavy-tailed. Finally, in \cite{telgarsky22}, the authors analyze mirror descent with constant stepsize and independent and identically distributed (i.i.d.) noise with a uniformly bounded variance that is a.s. bounded or sub-Gaussian. By choosing the constant stepsize appropriately, they obtain exponential high-probability bounds in this setting.

\textit{Linear Stochastic Approximation.}
For linear SA, the first moment bounds for the $\ell_2$-norm of the error with constant stepsize were given in \cite{lakshminarayanan2018linear,srikant2019finite}. Based on these, one could obtain high-probability bounds, albeit with polynomial tails instead of exponential ones. To our knowledge, the strongest result on exponential high-probability bounds for linear SA is given in \cite{dalal2018general}. There, the authors analyze a two-timescale linear SA with decreasing stepsizes, and with multiplicative, a.s. bounded, martingale-difference noise. In this setting, they obtain maximal exponential high-probability bounds for all iterates large enough by choosing stepsizes that depend on both the runtime and the confidence level. On the other hand, there is a line of work that focuses primarily on the product of random matrices and then applies these results to linear SA. In \cite{durmus2021tight}, the authors consider a linear SA with constant stepsize, where the noise is Markovian and a.s. bounded. In this setting, they develop high-probability bounds on the product of random matrices to obtain sub-exponential high-probability bounds when the random matrices are a.s. Hurwitz, and polynomial high-probability bounds when the random matrices are only Hurwitz in expectation. This was later extended to the case of Polyak-Ruppert averaged iterates in \cite{durmusAveraged,Mou20}.

\textit{TD-Learning in RL.} TD-learning was proposed as an SA algorithm for solving the policy evaluation problem in RL \cite{sutton1988learning,sutton2018reinforcement}. The mean-square bounds of TD-learning were established in \cite{bhandari2018finite,srikant2019finite}, and high-probability bounds in \cite{dalal2018finite,LSTD17}. Specifically, the authors of \cite{LSTD17} consider the least-square temporal difference (LSTD) algorithm (which includes a projection step onto a compact set) and obtain exponential high-probability bounds for the $\ell_2$-norm of the error when the stepsizes are $\mathcal{O}(k^{-1})$.  The authors of \cite{dalal2018finite} analyze TD-learning with linear function approximation with i.i.d. sampling and obtain maximal exponential high-probability bounds for the $\ell_2$-norm of the error in the last iterate for iterates beyond some point that is of order $\log(1/\delta)$. In the off-policy setting, finite-sample mean-square bounds of TD-learning were established in \cite{chen2021GB,chen2022sample}. To our knowledge, there are no results on high-probability bounds (with tails decaying faster than polynomials) of off-policy TD-learning, with or without function approximation.

\textit{$Q$-Learning in RL.} 
In one of the earliest works on exponential high-probability bounds in RL \cite{even2003learning}, the authors analyze synchronous $Q$-learning algorithm when the stepsizes are $\mathcal{O}(k^{-z})$ for $z\in(1/2,1)$. In this setting, they obtain exponential high-probability bounds for all iterates large enough. Recently in \cite{li2020sample,li2021q,li2024q}, the authors analyze the popular $Q$-learning algorithm with constant stepsize and uniformly bounded, Markovian, possibly conditionally biased noise. In this setting, given a runtime and a performance guarantee, they obtain exponential high-probability bounds at the end of the runtime, provided that it is large enough.

\textit{General Stochastic Approximation.}
For general nonlinear SA under arbitrary norms and decreasing stepsizes, the authors of \cite{chen2020finite,chen2019finitesample} obtain bounds on the second moment of the error. These moment bounds can be used to obtain high-probability bounds, albeit without exponential tails. 

In \cite{thoppe2019concentration}, the authors consider an SA with decreasing stepsizes and martingale-difference sub-exponential noise. In this setting, they obtain maximal exponential high-probability bounds conditioned on the event that the iterates are close enough to the fixed point after some time. In follow-up work \cite{chandak2022concentration}, they assume that their noise is multiplicative, a.s. bounded, and Markovian. In this setting, they obtain maximal exponential high-probability bounds without conditioning on an unknown event. However, their high-probability bounds only hold after some time, and both the bound and the probabilities depend on the unknown norm of the iterate after some time (which is random, with unknown distribution). 

In a separate line of work \cite{qu2020finite}, the authors consider a general SA under the infinity norm, where the noise has an a.s. uniformly bounded martingale-difference part, and a Markovian part that only determines which coordinate of the iterate gets updated. Due to this structured noise, the random operator is a conditionally biased estimator of the original operator. In this setting, assuming that the iterates are always in a compact set, they obtain exponential high-probability bounds. Finally, in \cite{infiniteDim22}, the authors consider a variance-reduced version of the general SA in arbitrary Banach spaces, with constant stepsize and i.i.d., multiplicative, a.s. bounded noise. By appropriately choosing the stepsize and the averaging used to reduce the variance, they obtain exponential high-probability bounds for the error.

\textit{In summary}, all of the previous high-probability bounds for SA in the literature have one or more of these limitations: (1) they force the iterates to belong to a compact set via a projection, or they introduce stringent assumptions on their noise so that their iterates belong a.s. to some compact set; (2) they tune the parameters of the algorithm according to the probability guarantee and total runtime; (3) they do not allow for multiplicative noise; (4) they are only valid for a particular iterate, or for a limited range of iterates, which can depend on the probability guarantee itself.

The remainder of the paper is organized as follows. In Section \ref{sec:main-results}, we present our main results on maximal concentration bounds of SA under bounded multiplicative noise and sub-Gaussian additive noise, as well as the impossibility results. In Sections \ref{sec:proof-multi}, \ref{sec:proof-impossibility}, and \ref{sec:proof-additive}, we present the proof of our main theoretical results. In Section \ref{sec:applications}, we showcase the applicability of our main results in the context of linear SA and RL. Finally, we conclude this work in Section \ref{sec:conclusion}.

\section{Main Results}\label{sec:main-results}
Consider solving the fixed-point equation $\bar{F}(x)=x$ with the SA algorithm presented in Eq. (\ref{eq:stochastic-approximation}). The following assumption is imposed on the operator $\bar{F}(\cdot)$.

\begin{assumption}\label{as:contraction}
	There exist a constant $\gamma_c\in[0,1)$ and a norm $\|\cdot\|_c$ such that $\|\Bar{F}(x_1)-\Bar{F}(x_2)\|_c\leq \gamma_c\|x_1-x_2\|_c$ for all $x_1,x_2\in\mathbb{R}^d$.
\end{assumption}
Using the Banach fixed-point theorem \cite{banach1922operations}, Assumption \ref{as:contraction} implies that $\Bar{F}(x)=x$ has a unique solution $x^*$. Our results also hold when $\Bar{F}(\cdot)$ is a pseudo-contractive operator, i.e., $\|\Bar{F}(x)-x^*\|_c\leq \gamma_c\|x-x^*\|_c$ for all $x\in\mathbb{R}^d$ \cite{bertsekas1996neuro}. However, in this case, the existence of $x^*$ must be assumed. A contraction mapping is always a pseudo-contraction mapping.

Our next assumption states that the noisy estimate $F(x_k,Y_k)$ of $\bar{F}(x_k)$ has an unbiased perturbation.

\begin{assumption}\label{as:unbiased}
	It holds that $\mathbb{E}[F(x_k,Y_k) \mid \mathcal{F}_k ] = \Bar{F}(x_k)$ a.s. for all $k\geq 0$, where $\mathcal{F}_k$ is the $\sigma$-algebra generated by $\{x_0,Y_0,Y_1,\cdots,Y_{k-1}\}$.
\end{assumption}

A special case where Assumption \ref{as:unbiased} is satisfied is when $\{Y_k\}$ is a sequence of i.i.d. random variables. Assumption \ref{as:unbiased} can be relaxed to $\| \mathbb{E}[F(x_k,Y_k) \mid \mathcal{F}_k] - \Bar{F}(x_k) \|_c \leq L \left\|x_k - x^*\right\|_c$ a.s. for all $k\geq 0$, for some small enough constant $L>0$. For SA with generally biased perturbation (a typical example of which is when $\{Y_k\}$ is a Markov chain), establishing high-probability bounds is a future direction. That being said, existing results that allow biased perturbation all require $\{x_k\}$ being bounded a.s. by a deterministic constant, such as $Q$-learning and ergodic mirror descent. 

\subsection{Stochastic Approximation with Multiplicative Noise}\label{subsec:multiplicative}

To begin with, we explain in the following assumption what we mean by multiplicative noise.

\begin{assumption}\label{as:multi}
	There exists $\sigma>0$ such that $\| F(x_k,Y_k) - \mathbb{E}[F(x_k,Y_k)\mid \mathcal{F}_k] \|_c \leq \sigma(1+\|x_k\|_c)$ a.s. for all $k\geq 0$.
\end{assumption}

One special case where Assumption \ref{as:multi} is satisfied is when the operator $F(x,y)$ is Lipschitz continuous in $x$ uniformly for all $y$, which is formally stated in the following. 

\begin{assumptionp}{\ref*{as:multi}$'$}\label{as:Lip}
    There exists $L_c>0$ such that $\sup_{y\in\mathcal{Y}}\|F(x_1,y)-F(x_2,y)\|_c\leq L_c\|x_1-x_2\|_c$ for all $x_1,x_2\in\mathbb{R}^d$, and $\sup_{y\in\mathcal{Y}}\|F(\bm{0},y)\|_c<\infty$.
\end{assumptionp}

To see the implication, under Assumption \ref{as:Lip}, we have by triangle inequality that
\begin{align}
    \| F(x_k,Y_k) \|_c\leq\;& \| F(x_k,Y_k) - F(\bm{0},Y_k)\|_c+\|F(\bm{0},Y_k) \|_c\nonumber\\
    \leq \;&L_c\|x_k\|_c+{\sup}_{y\in\mathcal{Y}}\|F(\bm{0},y)\|_c\nonumber\\
    \leq\;& \sigma(1+\|x_k\|_c),\label{lip1}
\end{align}
where $\sigma:=\max(L_c,\sup_{y\in\mathcal{Y}}\|F(\bm{0},y)\|_c)<\infty$. Moreover, Jensen's inequality implies that
\begin{align}\label{lip2}
    \|\mathbb{E}[F(x_k,Y_k)\mid \mathcal{F}_k]\|_c\leq\mathbb{E}[\|F(x_k,Y_k)\|_c\mid \mathcal{F}_k] \leq \sigma(1+\|x_k\|_c).
\end{align}
Assumption \ref{as:multi} then follows from combining Eqs. (\ref{lip1}) and (\ref{lip2}) with triangle inequality. Note that Assumption \ref{as:Lip} is automatically satisfied in linear SA, which has the update equation $ x_{k+1} = x_k + \alpha_k(A(Y_k) x_k - b(Y_k))$,
where $A(\cdot)$ and $b(\cdot)$ are bounded functions. In fact, the terminology ``multiplicative noise'' is inspired by linear SA. 

We next state our maximal concentration bounds. Let $D=\sigma+\gamma_c-1$, where $\gamma_c$ is the contraction factor from Assumption \ref{as:contraction} and $\sigma$ is the parameter from Assumption \ref{as:multi}.
The other parameters $c_1,c_1',c_1''$, $\{c_i\}_{2\leq i\leq 4}$, and $D_0\in (0,1)$ we use to state the following theorem are (problem-dependent) constants, the expressions of which will be revealed in Section \ref{sec:proof-multi}, where we present the complete proof of Theorem \ref{thm:multi}.

\begin{theorem}\label{thm:multi}
Consider $\{x_k\}$ generated by Eq. (\ref{eq:stochastic-approximation}). Suppose that Assumptions \ref{as:contraction} -- \ref{as:multi} are satisfied and $\alpha_k=\alpha/(k+h)$. Then we have the following results.
\begin{enumerate}[(1)]
\item When $D>0$, by choosing $\alpha>2/D_0$, and $h$ large enough, for any $\delta>0$ and $K\geq 0$, with probability at least $1-\delta$, we have for all $k\geq K$ that\footnote{It is possible to remove the product of logarithmic terms (i.e., $\mathcal{O}(\log(k)^{m-1})$) at the cost of slightly compromising the tail. The result is presented in Appendix \ref{ap:removing_log}.}
    \begin{align*}
	\|x_k-x^*\|_c^2
	\leq \;&\frac{c_1\alpha \|x_0-x^*\|_c^2}{k+h} \left[\log\left(\frac{m}{\delta}\right)+c_2+c_3+c_4\log\left(\frac{k-1+h}{h-1}\right)\right]^{m-1}\nonumber\\
	&\times \left[\log\left(\frac{m}{\delta}\right)+c_2\left(\frac{h}{K+h}\right)^{\alpha D_0/2-1}+c_3+c_4\log\left(\frac{k-1+h}{K-1+h}\right)\right],
\end{align*}
where $m=\lceil 2\alpha D\rceil +1$.
\item When $D=0$, by choosing $\alpha>2/D_0$, and $h$ large enough, for any $\delta>0$ and $K\geq 0$, with probability at least $1-\delta$, we have for all $k\geq K$ that
		\begin{align*}
	\|x_k-x^*\|_c^2
	\leq \;&\frac{c_1'\alpha\|x_0-x^*\|_c^2}{k+h}\left[\log\left(\frac{k-1+h}{h-1}\right)\right]^2\left[\log\left(\frac{1}{\delta}\right)\right.\\
	&\left.+\;c_2\left(\frac{h}{K+h}\right)^{\alpha D_0/2-1}+c_3+c_4 \log\left(\frac{k-1+h}{K-1+h}\right)\right].
\end{align*}
\item When $D<0$, by choosing $\alpha>2/D_0$ and $h$ large enough, for any $\delta>0$ and $K\geq 0$, with probability at least $1-\delta$, we have for all $k\geq K$ that
    \begin{align*}
        \|x_k-x^*\|_c^2\leq \frac{c_1''\alpha\|x_0-x^*\|_c^2}{k+h}\left[\log\left(\frac{1}{\delta}\right)+c_2\left(\frac{h}{K+h}\right)^{\alpha D_0/2-1}+c_3+c_4 \log\left(\frac{k-1+h}{K-1+h}\right)\right].
    \end{align*}
\end{enumerate}
\end{theorem}
Several remarks are in order.
We begin by discussing the tail, which is determined by the parameter $D$. In Theorem \ref{thm:multi} (2) and (3), where $D\leq 0$, since $\delta$ appears as $\log(1/\delta)$ in the squared error $\|x_k-x^*\|_c^2$, the convergence error $\|x_k-x^*\|_c$ has a sub-Gaussian tail. The case where $D>0$ (cf. Theorem \ref{thm:multi} (1)) is more subtle. In this case, the tail depends on the parameter $m$. Since $m=\lceil 2\alpha D\rceil +1$ and $D>0$, in general, we only have a Weibull tail. The fact that $m$ is affine in $D$ (up to a ceiling function) makes intuitive sense since a larger $D$ implies a noisier update, which in turn implies a heavier tail\footnote{Note that if Assumptions \ref{as:contraction} and \ref{as:multi} are satisfied with some $\gamma_c$ and $\sigma$, then they must also be satisfied with any $\gamma'_c\in (\gamma_c,1)$ and $\sigma'>\sigma$. Therefore, we can always make $D$ positive. However, in view of Theorem \ref{thm:multi}, the tail gets heavier (from a sub-Gaussian tail to a Weibull tail) as $D$ increases. Therefore, to obtain a concentration bound with the best tail decay rate, the parameters $\gamma_c$ and $\sigma$ should be viewed as the smallest ones so that Assumptions \ref{as:contraction} and \ref{as:multi} are satisfied. }.

Next, we discuss the convergence rate in terms of $k$ and $K$. We only consider the case where $D>0$, which is the most interesting case. Theorem \ref{thm:multi} (1) states that,
with probability at least $1-\delta$, all iterates lie in a cone that starts with radius $\Theta((1+\log^{m/2}(1/\delta))K^{-1/2})$, which corresponds to $k=K$. This matches with the rate obtained for the mean-square error in \cite{chen2021finite}. Then, for any $k>K$, the radius of the cone is of order $\Theta((\log^{m/2}(1/\delta)+\log^{m/2}(k))k^{-1/2})$.

Theorem \ref{thm:multi} has several implications. Specifically, maximal concentration bounds immediately imply concentration bounds for a fixed iteration number, which in turn gives the full tail bound. Here, we present only the results when $D>0$. The case where $D\leq 0$ follows a similar approach.

\begin{corollary}\label{corollary:multi}
Suppose that the same assumptions for Theorem \ref{thm:multi} (1) are satisfied. 
\begin{enumerate}[(1)]
    \item For any $\delta>0$ and $k\geq 0$, we have with probability at least $1-\delta$ that
\begin{align*}
	\|x_k-x^*\|_c^2
	\leq \;&\frac{c_1\alpha \|x_0-x^*\|_c^2}{k+h} \left[\log\left(\frac{m}{\delta}\right)+c_2+c_3+c_4\log\left(\frac{k-1+h}{h-1}\right)\right]^{m}.
\end{align*}
As a result, the sample complexity to achieve $\|x_k-x^*\|_c\leq \epsilon$ is $\tilde{\mathcal{O}}(\epsilon^{-2}\log^m(1/\delta))$
\item There exists $C_1>0$ such that the following inequality holds for all $\epsilon>0$ and $k\geq 0$: 
\begin{align*}
    \mathbb{P}\left(\frac{\sqrt{k+h}\;\| x_k - x^* \|_c}{\log(k)^{m/2}}> \epsilon \right) < m\exp\left(-C_1\epsilon^{2/m} \right).
\end{align*}
\end{enumerate}
\end{corollary}

Corollary \ref{corollary:multi} (1) follows by setting $K=k$ in Theorem \ref{thm:multi} (1), and Corollary \ref{corollary:multi} (2) follows by representing the tolerance level $\delta$ from Corollary \ref{corollary:multi} (1) as a function of the accuracy level $\epsilon$. Observe that Corollary \ref{corollary:multi} (2) provides an upper bound for the whole complementary cumulative distribution function (CDF) of the error $\|x_k-x^*\|_c$ for any $k\geq 0$. Therefore, we can use the formula $\mathbb{E}[\|x_k-x^*\|_c^r]=\int_0^\infty \mathbb{P}(\|x_k-x^*\|_c^r>x)dx$ (where $r$ is any positive integer) to integrate this bound to obtain bounds for any moment of the error at any point in time.

\subsubsection{Tightening the Bounds for Small Values of $k$} Note that, for $k=0$ and other small values of $k$, the bound in Theorem \ref{thm:multi} can be far from the true initial error $\|x_0-x^*\|_c$. In order to obtain tighter bounds for small values of $k$, we use a time-varying worst-case bound (cf. Proposition \ref{prop:worst_case_bound}). This bound turns out to be polynomially increasing in $k$ when $D>0$, and it is, in fact, our first step in proving Theorem \ref{thm:multi}. By combining these two bounds, we obtain a non-monotone bound as in Figure \ref{fig:anytimeBounded}.

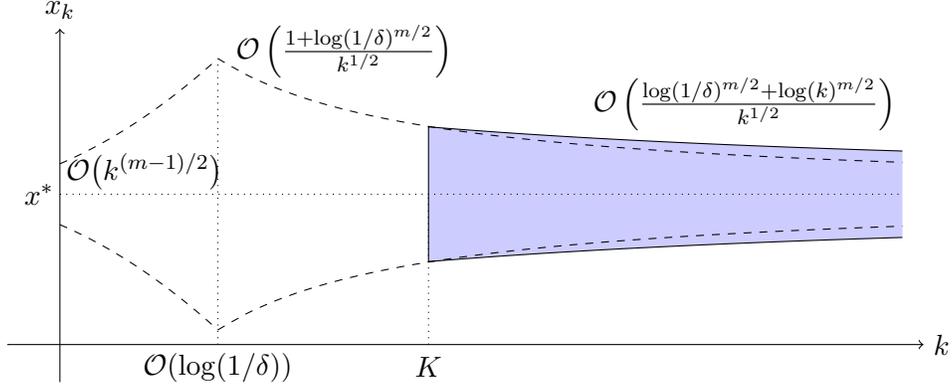
\begin{figure}[h]
    \centering
    \begin{tikzpicture}[xscale=1.4]
  \draw[->] (-0.5, 0) -- (8.2, 0) node[right] {$k$};
  \draw[->] (0, -0.5) -- (0, 4.2) node[above] {$x_k$};

  \draw[fill=blue!20!white] plot[scale=0.5, domain=16:7, smooth, variable=\x, blue] ({\x},  {4+4.4*(1.9)^2/(\x+1)*ln(\x+1)/ln(10)}) -- plot[scale=0.5, domain=7:16, smooth, variable=\x, blue] ({\x}, {4-4.4*(1.9)^2/(\x+1)*ln(\x+1)/ln(10)});

  \draw[dashed,scale=0.5, domain=0:3, smooth, variable=\x] plot ({\x}, {4+(\x/3+0.9)^2});
  \draw[dashed,scale=0.5, domain=3:16, smooth, variable=\x] plot ({\x}, {4+4*(1.9)^2/(\x+1)});
  \draw[dashed,scale=0.5, domain=0:3, smooth, variable=\x] plot ({\x}, {4-(\x/3+0.9)^2});
  \draw[dashed,scale=0.5, domain=3:16, smooth, variable=\x] plot ({\x}, {4-4*(1.9)^2/(\x+1)});
  
  \draw (6.5,3.2) node {$\mathcal{O}\left(\frac{\log(1/\delta)^{m/2}+\log(k)^{m/2}}{k^{1/2}}\right)$};
  
  \draw (2.7,3.9) node {$\mathcal{O}\left(\frac{1+\log(1/\delta)^{m/2}}{k^{1/2}}\right)$};
  \draw (0.8,2.3) node {$\mathcal{O}\big(k^{(m-1)/2}\big)$};

  \draw[dotted] (0,2) -- (8,2);
  \draw (-0.2,2) node {$x^*$};

  \draw[dotted] (3.5,0) -- (3.5,2);
  \draw (3.5,-0.3) node {$K$};
  
  \draw[dotted] (1.5,0) -- (1.5,3.8);
  \draw (1.5,-0.3) node {$\mathcal{O}(\log(1/\delta))$};
\end{tikzpicture}
    \caption{For $D>0$, all the iterates lie in the blue shaded area with probability at least $1-\delta$.}
    \label{fig:anytimeBounded}
\end{figure}

Observe from Figure \ref{fig:anytimeBounded} that the combination of the high-probability bound of Theorem \ref{thm:multi} and the worst-case bound of Proposition  \ref{prop:worst_case_bound} presents a surprising behavior: the combined bound is increasing up to a point, and then it is decreasing. In particular, the bound is increasing up to a time of order $\mathcal{O}(\log(1/\delta))$, at which point the bound on the squared error is of order $\mathcal{O}(\log(1/\delta)^{m-1})$. To demonstrate that this is not an artifact of our proof, we next present numerical simulations based on the following example.

\begin{example}\label{example:increasing_decreasing}
    Consider a $1$-dimensional SA given by \begin{align}\label{sa:1d}
        x_{k+1}=x_k+\alpha_k (Y_k x_k - x_k),
    \end{align}
    where $\{Y_k\}$ is a sequence of i.i.d. random variables such that $\mathbb{P}(Y_k=0.65)=\mathbb{P}(Y_k=1.15)=1/2$, and $\alpha_k=60/(k+120)$. 
\end{example}

Note that the linear SA in Eq. (\ref{sa:1d}) is a special case of the contractive SA  in Eq. (\ref{eq:stochastic-approximation}) with $F(x,y)=xy$. In addition, it is easy to verify that  Assumption \ref{as:contraction} holds with the minimum $\gamma_c$ being $0.9$, Assumption \ref{as:unbiased} holds (due to $\{Y_k\}$ being i.i.d.), and Assumption \ref{as:multi} holds with the minimum $\sigma$ being $0.25$. As a result, we have $D=\sigma+\gamma_c-1=0.15>0$, which corresponds to Theorem \ref{thm:multi} (1) and the figure depicted in Figure~\ref{fig:anytimeBounded}. In our numerical simulations, we start from the initial condition $x_0=10$ and run $20,000$ instances of Eq. (\ref{sa:1d}) and plot some percentiles in Figure~\ref{fig:LSApercentiles}. It is clear that the percentiles first increase before they decrease, which agrees with Figure \ref{fig:anytimeBounded}.

\begin{figure}[h]
    \centering
    \includegraphics[width=0.7\textwidth]{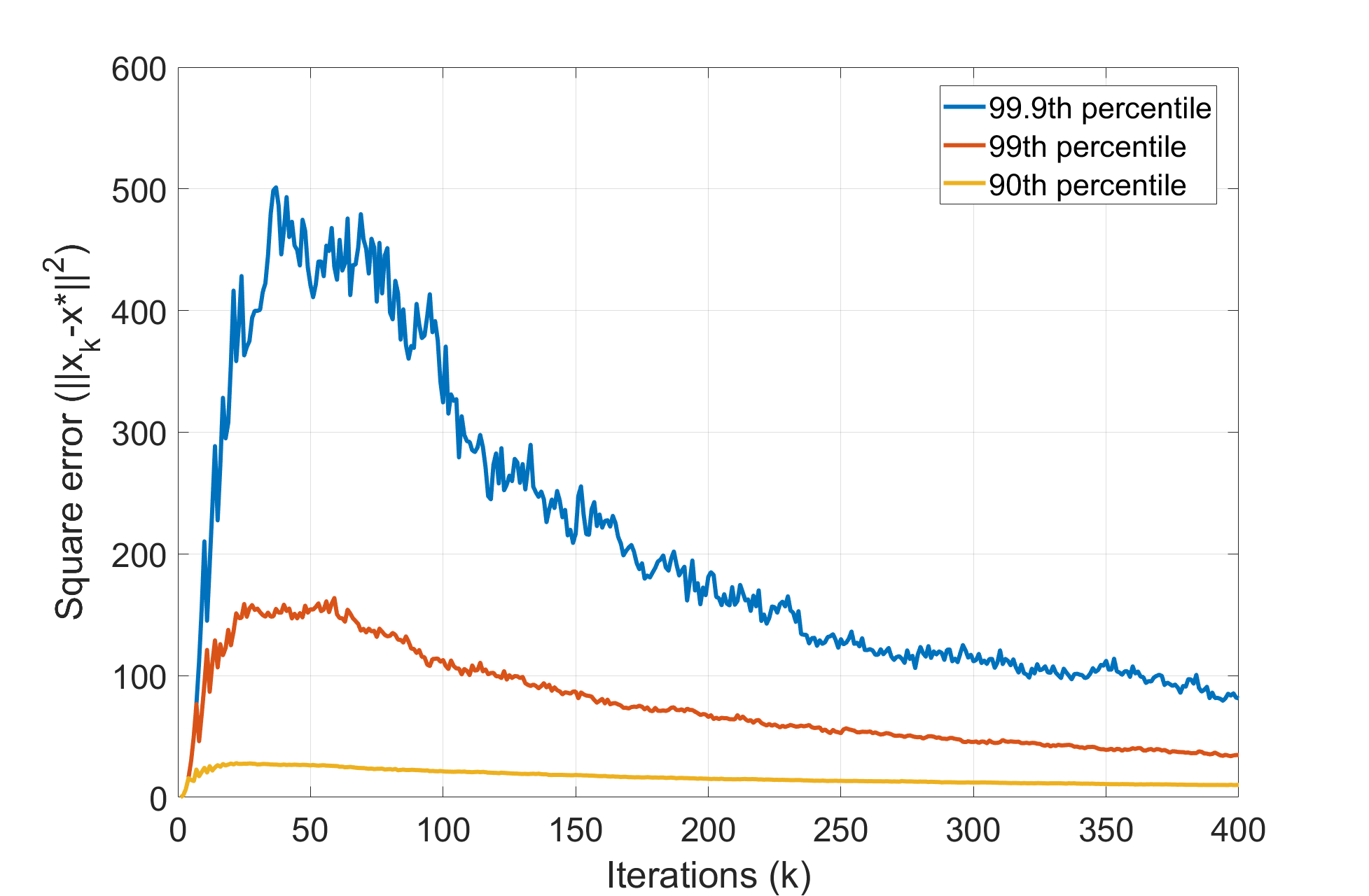}
    \caption{Percentiles of the squared error for a 1-dimensional SA with $D>0$.}
    \label{fig:LSApercentiles}
\end{figure}

\subsection{An Impossibility Result on the Tail Decay Rate}\label{sec:hard_example}
Theorem \ref{thm:multi} shows that SA with multiplicative noise in general is able to achieve an $\tilde{\mathcal{O}}(1/k)$ convergence rate with a Weibull tail. One may ask if a sub-Gaussian (or a sub-exponential) tail is achievable. In this section, we show in the following example that it is impossible to obtain a general sub-exponential tail bound whenever we only obtain a Weibull one.

\begin{example}\label{example:impossibility_result}
Consider the $1$-dimensional linear SA presented in Eq. (\ref{sa:1d}). In this case,
    let $\{Y_k\}$ be an i.i.d. sequence of real-valued random variables such that 
$\mathbb{P}\left(Y_k=a+N\right)=1/(N+1)$ and $\mathbb{P}\left(Y_k=a-1\right)=N/(N+1)$, where $a\in (0,1)$ and $N\geq 1$ are tunable parameters.
Note that the update equation can be equivalently written as
\begin{align}\label{algo:example}
    x_{k+1} =(1+(Y_k-1)\alpha_k)x_k.
\end{align}
\end{example}
In the above example, it can be easily verified that Assumption \ref{as:contraction} is satisfied with the minimum $\gamma_c$ being $a$, Assumption \ref{as:unbiased} holds due to $\{Y_k\}$ being an i.i.d. sequence, and Assumption \ref{as:multi} is satisfied with the minimum $\sigma$ being $N$. As a result, we have $D=a+N-1$. Next, we apply Theorem \ref{thm:multi} to obtain high-probability bounds for the SA algorithm presented in Eq. \eqref{algo:example}. Suppose that $x_0>0$, $\alpha_k=\alpha/(k+h)$, where $\alpha>1/(1-a)$ and $h$ is large enough so that $\alpha_0<1/2$. Then, there exist $K_1,K_2>0$ such that the following inequality holds for all $\epsilon>0$ and $k\geq 0$:
\begin{align}
    \mathbb{P}\left(\frac{\sqrt{k+h} \;x_k}{\log(k)^{m_e/2}} > \epsilon \right) <K_1 \exp\left(-K_2\epsilon^{\frac{2}{m_e}}\right),\label{eq:example_bound1}
\end{align}
where $m_e=\lceil 2\alpha D \rceil +1$.

We next investigate the lower bound of the SA algorithm in Eq. (\ref{algo:example}) through the following theorem, the proof of which is presented in Section \ref{sec:proof-impossibility}.

\begin{theorem}\label{thm:impossibility}
Consider $\{x_k\}$ generated by Eq. (\ref{algo:example}). Suppose that $\alpha_k=\alpha/(k+h)^z$, where $z\in (0,1]$, $\alpha>0$, and $h$ is chosen such that $\alpha_0<1/2$.
\begin{enumerate}[(1)]
    \item When $z=1$, for any ${\tilde{\beta}}>2/(1+2\alpha D)$, we have $\liminf\limits_{k\to\infty} \mathbb{E}[\exp(\lambda [(k+h)^{1/2} x_k]^{\tilde{\beta}})]=\infty$ for all $\lambda>0$. As a result, there do not exist $K_1',K_2'>0$ such that $\mathbb{P}\left((k+h)^{1/2}\;x_k\geq \epsilon\right)\leq K_1'\exp(-K_2'\epsilon^{\tilde{\beta}}) $ for any $\epsilon>0$ and $k\geq 0$.
\item When $z\in (0,1)$, for any ${\tilde{\beta}},{\tilde{\beta}}'>0$, we have $\liminf\limits_{k\to\infty} \mathbb{E}[\exp(\lambda (k+h)^{{\tilde{\beta}}'} x_k^{\tilde{\beta}})]=\infty$ for all $\lambda>0$. As a result, there do not exist $\bar{K}_1',\bar{K}_2'>0$ such that $\mathbb{P}\left((k+h)^{{\tilde{\beta}}'/{\tilde{\beta}}}\;x_k\geq \epsilon\right)\leq \bar{K}_1'\exp\left(-\bar{K}_2'\epsilon^{\tilde{\beta}}\right) $ for any $\epsilon>0$ and $k\geq 0$.
\end{enumerate}
\end{theorem}
In Theorem \ref{thm:impossibility} (1), since ${\tilde{\beta}} > 2/(1+2\alpha D)\geq 2/(1+\lceil 2\alpha D\rceil) = 2/m_e$, our concentration bound is almost tight in the sense that it has the best tail decay rate (at least when $2\alpha D$ is an integer) but with a slightly worse decay rate in $k$. In particular, this means that we obtain a sub-exponential tail upper bound whenever such bound is achievable. This is depicted in Figure \ref{fig:upper_and_lower_bounds}. Note that Theorem \ref{thm:impossibility} (2) implies that not even Weibull tail bounds are possible when $\alpha_k=\alpha/(k+h)^z$ (with $z\in (0,1)$) for any polynomial rate of convergence.

\begin{figure}[t]
    \centering
    \begin{tikzpicture}[xscale=1.5, yscale=1.5]
  \draw[->] (1.5, 0) -- (8.5, 0) node[right] {$D$};
  \draw[->] (2, -0.4) -- (2, 3.7) node[above] {$\beta$};
  
  \draw[thick, dashed, color=blue] plot[scale=1, domain=1.75:8.4, smooth, variable=\x] ({\x},  {4.5/(1.5+(\x-2))});
  \draw[thick] (1.5,3) -- (2,3);
  \draw[thick] (2,1.5) -- (3.5,1.5);
  \draw[thick] (3.5,1) -- (5,1);
  \draw[thick] (5,3/4) -- (6.5,3/4);
  \draw[thick] (6.5,3/5) -- (8,3/5);
  \draw[thick] (8,1/2) -- (8.4,1/2);

  \draw[dotted] (3.5,0) -- (3.5,1.5);
  \draw[dotted] (5,0) -- (5,1);
  \draw[dotted] (6.5,0) -- (6.5,3/4);
  \draw[dotted] (8,0) -- (8,3/5);
  
  \draw (3.5,-0.2) node {$\frac{1}{2\alpha}$};
  \draw (5,-0.2) node {$\frac{1}{\alpha}$};
  \draw (6.5,-0.2) node {$\frac{3}{2\alpha}$};
  \draw (8,-0.2) node {$\frac{2}{\alpha}$};
  
  \draw (1.85,3) node {$2$};
  \draw (1.85,1.5) node {$1$};

  \draw (3,2.3) node {{\color{blue}$\frac{2}{1+2\alpha D}$}};
  \draw (2.75,1.25) node {$\frac{2}{1+\lceil 2\alpha D \rceil}$};

  \draw (1.8,1) node {\tiny $2/3$};
  \draw (1.8,3/4) node {\tiny $2/4$};
  \draw (1.8,3/5) node {\tiny $2/5$};
  \draw (1.8,1/2) node {\tiny $2/6$};
  \draw[dotted] (2,1) -- (3.5,1);
  \draw[dotted] (2,3/4) -- (5,3/4);
  \draw[dotted] (2,3/5) -- (6.5,3/5);
  \draw[dotted] (2,1/2) -- (8,1/2);
\end{tikzpicture}
    \caption{Best tail exponent in Eq. (\ref{eq:example_bound1}) (black) vs. upper bound on best possible tail exponent given by Theorem \ref{thm:impossibility} (1) (dashed blue)}
    \label{fig:upper_and_lower_bounds}
\end{figure}
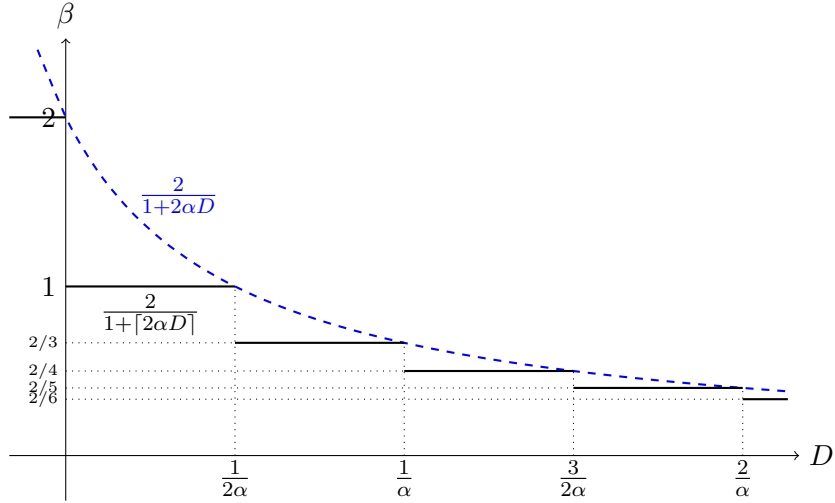

\subsection{Stochastic Approximation with Sub-Gaussian Additive Noise}\label{sec:additive}

In this section, we also consider $\{x_k\}$ generated by the SA algorithm presented in Eq. (\ref{eq:stochastic-approximation}), but with additive sub-Gaussian noise, which is explained in the following assumption. 

\begin{assumption}\label{ass:sub-Gaussian}
There exist $\Bar{\sigma}>0$ and a (possibly dimension-dependent) constant $c_d>0$ such that 
the following two inequalities hold for any $k\geq 0$ and $\mathcal{F}_k$-measurable random vector $v$:
\begin{align}
    \mathbb{E}\left[\exp\left(\lambda \langle F(x_k,Y_k) - \Bar{F}(x_k), v \rangle \right) \middle| \mathcal{F}_k \right] \leq& \exp\left(\lambda^2\Bar{\sigma}^2 \|v\|_{c,*}^2/2\right)\label{eq:sub-Gaussian_1}
\end{align}
for all $\lambda>0$, and
\begin{align}
    \mathbb{E}\left[ \exp\left( \lambda \left\| F(x_k,Y_k) - \Bar{F}(x_k) \right\|_c^2 \right) \middle| \mathcal{F}_k \right] \leq&  \left(1- 2 \lambda \Bar{\sigma}^2\right)^{-\frac{c_d}{2}}\label{eq:sub-Gaussian_2}
\end{align}
for all $\lambda\in \left(0,1/2\Bar{\sigma}^2\right)$, where $\|\cdot\|_c^*$ is the dual norm \cite{beck2017first} of the contraction norm $\|\cdot\|_c$.
\end{assumption}

Assumption \ref{ass:sub-Gaussian} can be viewed as a generalization of the standard definition of a random vector being norm sub-Gaussian \cite{jin2019short} to the case where we use an arbitrary norm $\|\cdot\|_c$ instead of $\|\cdot\|_2$. In fact, when $\|\cdot\|_c=\|\cdot\|_2$ and $c_d=d$, Eqs. (\ref{eq:sub-Gaussian_1}) and (\ref{eq:sub-Gaussian_2}) are exactly the equivalent definitions of sub-Gaussian random vectors \cite{jin2019short,wainwrightHDS}. Since we use an arbitrary norm, we allow for a (possibly) different dimension-dependent constant $c_d$. One special case where Assumptions \ref{as:contraction}, \ref{as:unbiased}, and \ref{ass:sub-Gaussian} are satisfied is when the noise $Y_k$ is purely additive and is either a martingale-difference sequence or an i.i.d. mean zero sequence with sub-Gaussian tails.

Next, we state the concentration bound. In the following theorem, the parameters $\{\Bar{c}_i\}_{1\leq i\leq 5}$ and $\bar{D}_1\in (0,1)$ are (problem-dependent) constants, the explicit expressions of which and the complete proof of the theorem are presented in Section \ref{sec:proof-additive}.

\begin{theorem}\label{thm:additive}
Consider $\{x_k\}$ generated by Eq. (\ref{eq:stochastic-approximation}). Suppose that Assumptions \ref{as:contraction}, \ref{as:unbiased}, and \ref{ass:sub-Gaussian} are satisfied, and $\alpha_k=\alpha/(k+h)^z$, where $z\in (0,1]$ and $\alpha,h>0$ are appropriately chosen. Then we have the following results.
\begin{enumerate}[(1)]
\item When $z=1$, by choosing $\alpha>2/\Bar{D}_1$, for any $\delta>0$ and $K\geq 0$, with probability at least $1-\delta$, we have for all $k\geq K$ that
    \begin{align*}
        \|x_k-x^*\|_c^2\leq\;&\frac{\bar{c}_1\log(1/\delta)}{k+h}+\bar{c}_2\|x_0\!-\!x^*\|_c^2\left(\frac{h}{k+h}\right)^{\Bar{D}_1\alpha/2}\!+\!\frac{\bar{c}_3+\bar{c}_4\log((k+1)/K^{1/2})}{k+h}.
    \end{align*}
\item When $z\in (0,1)$, by choosing $\alpha>0$ and $h\geq (4z/(\Bar{D}_1\alpha))^{1/(1-z)}$, 
 for any $\delta>0$ and $K\geq 0$, with probability at least $1-\delta$, we have for all $k\geq K$ that
    \begin{align*}
        \|x_k-x^*\|_c^2\leq\;&\frac{\Bar{c}_1\log(1/\delta)}{(k+h)^z}+\bar{c}_2\|x_0-x^*\|_c^2\exp\left(-\frac{\Bar{D}_1\alpha ((k+h)^{1-z}-h^{1-z})}{2(1-z)}\right)\\
        &+\frac{\bar{c}_5+\bar{c}_4\log((k+1)/K^{1/2})}{(k+h)^z}.
    \end{align*}
\end{enumerate}
\end{theorem}

We will discuss the implications of Theorem  \ref{thm:additive} in terms of its dependence on $\delta$, $K$, and $k$. Since the tolerance level $\delta$ appears as $\log(1/\delta)$ in the norm-square bound, the norm error $\|x_k-x^*\|_c$ has a sub-Gaussian tail. As for the dependence on $K$ and $k$, Theorem \ref{thm:additive} implies that, with probability at least $1-\delta$, all the iterates lie in a cone with a radius $\tilde{\Theta}((\log(1/\delta)^{1/2}+\log(k/K^{1/2}))k^{-z/2})$
for all $k\geq K$.

As a side note, observe that when $z=1$, to achieve the $\tilde{\mathcal{O}}(1/k)$ rate of convergence, the parameter $\alpha$ in the stepsize must be bounded away from zero ($\alpha>2/\Bar{D}_1$ to be precise). However, when $z<1$, the parameter $\alpha$ only needs to be positive, suggesting that the convergence rate for using the $\alpha_k=\alpha/(k+h)^z$ stepsize, while being sub-optimal, is more robust.  This coincides with what was observed in the literature studying the mean-square error \cite{chen2020finite,bhandari2018finite}. In addition, recall that in the multiplicative noise setting, if we use $\alpha_k=\alpha/(k+h)^z$ as the stepsize, it is not possible to achieve a convergence bound with a polynomial rate of convergence and a Weibull tail (cf. Theorem \ref{thm:impossibility} (2)), suggesting a fundamental difference in the behavior of the stochastic iterates generated by SA with multiplicative noise versus additive noise.

Similarly to Corollary \ref{corollary:multi} in Section \ref{subsec:multiplicative}, in the additive noise setting, one can use the maximal concentration bound to obtain the fixed-time concentration bound of order $k^{-z/2}$, which in turn gives us the full tail bound. The results are omitted here.

\section{Proof of Theorem \ref{thm:multi}}\label{sec:proof-multi}
We will follow the high-level idea presented in Section \ref{subsec:challenge} to prove the result. Specifically, in Section \ref{subsec:proof-multi-initialization}, we show that $\{x_k\}$ generated by Eq. (\ref{eq:stochastic-approximation}), while is not uniformly bounded, admits a time-varying worst-case bound. In Section \ref{subsec:proof-multi-blueprint}, we prove Proposition \ref{prop:blueprint_intro}, which serves as a blueprint for our bootstrapping argument. Finally, in Section \ref{subsec:proof-multi-bootstrapping}, we implement the bootstrapping argument to finish the proof, where we use the worse-case bound as an initialization and then use Proposition \ref{prop:blueprint_intro} to iteratively improve the bound.

\subsection{Time-Varying Worst-Case Bounds}\label{subsec:proof-multi-initialization}

The following proposition establishes the time-varying worst-case bounds of $\{x_k\}$ generated by the SA algorithm presented in Eq. (\ref{eq:stochastic-approximation}). Recall that $D=\sigma+\gamma_c-1$.

\begin{proposition}\label{prop:worst_case_bound}
Suppose that Assumptions \ref{as:contraction} and \ref{as:multi} are satisfied and $\alpha_k=\alpha/(k+h)$, where $\alpha,h>0$. Then, we have $\|x_k-x^*\|_c\leq B_k(D)$ a.s. for all $k\geq 0$, where
\begin{align*}
	B_k(D)=\begin{dcases}
	\left(\frac{k-1+h}{h-1}\right)^{\alpha D}\left(\|x_0-x^*\|_c+\frac{\sigma(1+\|x^*\|_c)}{D}\right)
	,&\text{if }D> 0,\\
	\|x_0-x^*\|_c+\sigma(1+\|x^*\|_c)\alpha \log\left(\frac{k-1+h}{h-1}\right),&\text{if }D=0,\\
	\|x_0-x^*\|_c-\frac{\sigma(1+\|x^*\|_c)}{D},&\text{if }D< 0.
\end{dcases} 
\end{align*}
\end{proposition}
\begin{remark}
    Intuitively, the parameter $\gamma_c$ captures the contraction effect of the expected operator, and the parameter $\sigma$ captures the expansive effect of the noise. The combined effect is captured by the parameter $D=\sigma+\gamma_c-1$. Proposition \ref{prop:worst_case_bound} states that $\|x_k\|_c$ is uniformly bounded by a deterministic constant when $D<0$, grows at most logarithmically when $D=0$, and can grow at a polynomial rate of $\mathcal{O}(k^{\alpha D})$ when $D>0$. 
\end{remark}
\begin{proof}[Proof of Proposition \ref{prop:worst_case_bound}]
    Using the update equation (\ref{eq:stochastic-approximation}) and the fact that $\Bar{F}(x^*)=x^*$, we have for all $k\geq 0$ that
\begin{align*}
	x_{k+1}-x^*=x_k-x^*+\alpha_k(F(x_k,Y_k)-\Bar{F}(x_k)+\Bar{F}(x_k)-\Bar{F}(x^*)+x^*-x_k).
\end{align*}
It follows that
\begin{align}
	\|x_{k+1}-x^*\|_c
	\leq\;& (1-\alpha_k)\|x_k-x^*\|_c+\alpha_k(\|F(x_k,Y_k)-\Bar{F}(x_k)\|_c+\|\Bar{F}(x_k)-\Bar{F}(x^*)\|_c)\nonumber\\
	\leq\;& (1-\alpha_k)\|x_k-x^*\|_c+\alpha_k(\sigma(1+\|x_k\|_c)+\gamma_c\|x_k-x^*\|_c)\nonumber\\
	\leq\;& (1-\alpha_k)\|x_k-x^*\|_c+\alpha_k(\sigma\|x_k-x^*\|_c+\sigma(1+\|x^*\|_c)+\gamma_c\|x_k-x^*\|_c)\nonumber\\
	=\;& (1+(\sigma+\gamma_c-1)\alpha_k)\|x_k-x^*\|_c+\alpha_k\sigma(1+\|x^*\|_c)\nonumber\\
	=\;& (1+D\alpha_k)\|x_k-x^*\|_c+\alpha_k\sigma(1+\|x^*\|_c),\label{eq:proposition:worst-case bound:before_recursion}
\end{align}
where the second inequality follows from Assumption \ref{as:contraction} and Assumption \ref{as:multi}. To proceed, we need the following lemma, which is a general result of solving recursive inequalities.

\begin{lemma}\label{le:help1}
	Consider a scalar-valued sequence $\{w_k\}$ that satisfies
	\begin{align*}
		w_{k+1}\leq (1+\beta_1\alpha_k)w_k+\beta_2 \alpha_k,\quad \forall\;k\geq 0,
	\end{align*}
	where $w_0\geq 0$ and $\{\alpha_k\}$ is a sequence of positive real numbers. Suppose that $\beta_1>-1/\alpha_0$ and $\beta_2>0$. Then, we have for all $k\geq 0$ that
	\begin{align*}
		w_k\leq \begin{dcases}
			e^{\beta_1\sum_{i=0}^{k-1}\alpha_i}w_0+\frac{\beta_2}{\beta_1}(e^{\beta_1\sum_{i=0}^{k-1}\alpha_i}-1),& \textit{if }\beta_1>0,\\
			w_0+\beta_2\sum_{i=0}^{k-1}\alpha_i,&\textit{if }\beta_1=0,\\
			w_0-\frac{\beta_2}{\beta_1},&\textit{if }\beta_1<0.
		\end{dcases}
	\end{align*}
\end{lemma}

The proof of Lemma \ref{le:help1} is presented in Appendix \ref{pf:le:help1}. Applying Lemma \ref{le:help1} to Eq. (\ref{eq:proposition:worst-case bound:before_recursion}), we have 
\begin{align}\label{eq:proposition-worst-case-bound-last-inequality}
	\|x_k-x^*\|_c\leq \begin{dcases}
		e^{D\sum_{i=0}^{k-1}\alpha_i}\|x_0-x^*\|_c+\frac{\sigma(1+\|x^*\|_c)}{D}(e^{D\sum_{i=0}^{k-1}\alpha_i}-1),&\textit{if }D> 0,\\
		\|x_0-x^*\|_c+\sigma(1+\|x^*\|_c)\sum_{i=0}^{k-1}\alpha_i,&\textit{if }D=0,\\
		\|x_0-x^*\|_c-\frac{\sigma(1+\|x^*\|_c)}{D}&\textit{if }D< 0.
	\end{dcases} 
\end{align}
The rest of the proof is to evaluate $\sum_{i=0}^{k-1}\alpha_i$ and $e^{D\sum_{i=0}^{k-1}\alpha_i}$ when $\alpha_k=\alpha/(k+h)$.
Observe that
\begin{align*}
	\sum_{i=0}^{k-1}\alpha_i=\sum_{i=0}^{k-1}\frac{\alpha}{i+h}\leq \alpha\int_{-1}^{k-1}\frac{1}{x+h}dx=\alpha\log\left(\frac{k-1+h}{h-1}\right),
\end{align*}
which further implies
\begin{align*}
    \exp\left(D\sum_{i=0}^{k-1}\alpha_i\right)\leq \left(\frac{k-1+h}{h-1}\right)^{\alpha D}.
\end{align*}
The final result follows by combining the upper bounds we obtained for $\sum_{i=0}^{k-1}\alpha_i$ and $e^{D\sum_{i=0}^{k-1}\alpha_i}$ with Eq. (\ref{eq:proposition-worst-case-bound-last-inequality}).
\end{proof}

\subsection{An Iterative Framework to Improve the Bound}\label{subsec:proof-multi-blueprint}

Now that we have established a time-varying worst-case bound of the SA iterates, the next step is to establish Proposition \ref{prop:blueprint_intro}, a blueprint for the iterative refinement of bounds, which is restated in the following for ease of presentation.

\begin{proposition}\label{prop:blueprint}
    Given a probability tolerance level $\delta\in (0,1)$, suppose that there exists a \textit{non-decreasing} sequence $\{T_k(\delta)\}$ such that 
\begin{align}\label{eq:bootstrapping_recipe_before}
    \mathbb{P}(\|x_k-x^*\|_c^2\leq  T_k(\delta),\forall\;k\geq 0)\geq 1-\delta.
\end{align}
Then, for any $\delta'\in (0,1-\delta)$, there exists $T_k(\delta,\delta')=\mathcal{O}(T_k(\delta)\alpha_k)$ such that 
\begin{align}\label{eq:bootstrapping_recipe}
    \mathbb{P}(\|x_k-x^*\|_c^2\leq  T_k(\delta,\delta'),\forall\;k\geq 0)\geq 1-\delta-\delta'.
\end{align}
\end{proposition}

Since $x_0$ is initialized deterministically, we must have $T_0(\delta)\geq \|x_0-x^*\|_c^2$ a.s. 
Once Proposition \ref{prop:blueprint} is established, we can use Proposition \ref{prop:worst_case_bound} as initialization and iteratively improve the bound using Proposition \ref{prop:blueprint} to prove Theorem \ref{thm:multi}. To prove Proposition \ref{prop:blueprint}, we use a Lyapunov argument. The construction of our Lyapunov function needs the following definition.
\begin{definition}\label{def:smoothness}
  Let $f:\mathbb{R}^d\mapsto\mathbb{R}$ be a convex differentiable function. Then, $f(\cdot)$ is said to be $L$-smooth with respect to some norm $\|\cdot\|$ if and only if $f(y) \leq  f(x) + \nabla f(x)^\top (y-x) + \frac{L}{2}\|x-y\|^2$ for all $x,y\in\mathbb{R}^d$.
\end{definition}
Inspired by \cite{chen2020finite}, we will use the generalized Moreau envelope defined as
\begin{align*}
	M(x) = \min_{u\in\mathbb{R}^d} \left\{ \frac{1}{2}\|u\|_c^2 + \frac{1}{2\mu} \|x-u\|_s^2 \right\}
\end{align*}
as our Lyapunov function,
where $\|\cdot\|_s$ is a smoothing norm chosen such that $\frac{1}{2}\|\cdot\|_s^2$ is an $L$-smooth function with respect to $\| \cdot \|_s$, and $\mu>0$ is a tunable constant. Intuitively, the generalized Moreau envelope is constructed as a smooth approximation of the norm-square function $\|x\|_c^2$, which itself is in general not smooth (for example, $\|\cdot\|_\infty^2$). This was formally established in \cite[Lemma 2.1]{chen2020finite} and is presented in the following lemma for completeness. Let $\ell_{cs}$ and $u_{cs}$ be two positive constants such that $\ell_{cs}\|x\|_s\leq  \|x\|_c\leq u_{cs}\|x\|_s$ for all $x\in\mathbb{R}^d$, which is always possible due to the equivalence between norms in finite-dimensional spaces. We assume without loss of generality that $\ell_{cs}\in (0,1]$ and $u_{cs}\in [1,+\infty)$.

\begin{lemma}[Lemma 2.1 of \cite{chen2020finite}]\label{prop:Moreau}
	The generalized Moreau envelope $M(\cdot)$ has the following properties: (1) The function $M(\cdot)$ is convex and is $L/\mu$ -- smooth with respect to $\|\cdot\|_s$. (2) There exists a norm, denoted by $\|\cdot\|_M$, such that $M(x)=\frac{1}{2}\|x\|_M^2$ for all $x\in\mathbb{R}^d$. (3) It holds that $(1+\mu \ell_{cs}^2)^{1/2}\|x\|_M\leq \|x\|_c\leq (1+\mu u_{cs}^2)^{1/2}\|x\|_M$ for all $x\in\mathbb{R}^d$.
\end{lemma}

For simplicity of notation, denote $\ell_{cM}=(1+\mu \ell_{cs}^2)^{1/2}$, $u_{cM}=(1+\mu u_{cs}^2)^{1/2}$, and $\tilde{\gamma}_c=\gamma_c u_{cM}/\ell_{cM}$. The tunable parameter $\mu>0$ is chosen such that $\tilde{\gamma}_c<1$, which is always possible since $\gamma_c\in (0,1)$ and $\lim_{\mu\rightarrow 0}u_{cM}/\ell_{cM}=1$.

To proceed, recall that in the development of classical concentration inequalities such as Hoeffding's inequality and Chernoff bound, an important step is to bound the MGF of the random variable of interest. Once that is done, the concentration bound can be derived using the Markov inequality together with the bound on the MGF. Inspired by this, we will bound the MGF of a modified version of the generalized Moreau envelope. Such a modification is introduced to address the potential unboundedness issue of the iterates $\{x_k\}$.

Let $E_k(\delta)=\{\|x_t-x^*\|_c^2\leq  T_t(\delta),\forall\;t=0,1,\cdots,k\}$. Note that $\{E_k(\delta)\}_{k\geq 0}$ is by definition a sequence of decreasing events, i.e., $E_{k+1}(\delta)\subseteq E_k(\delta)$ for all $k\geq 0$. In addition, according to Eq. (\ref{eq:bootstrapping_recipe_before}), we have $\mathbb{P}(E_k(\delta))\geq 1-\delta$ for any $k\geq 0$. Let $\lambda_k=\theta\alpha_k^{-1} T_k(\delta)^{-1}$, where $\theta$ is a tunable constant yet to be chosen and $\alpha_k$ is the stepsize. For any $k\geq 0$, let
\begin{align}\label{def:Z_k}
	Z_k=\log\left(\mathbb{E}\left[\exp\left(\lambda_k\mathds{1}_{E_k(\delta)} M(x_k-x^*)\right)\right]\right),
\end{align}
which is the log-MGF of a modified version of the generalized Moreau envelope that will be frequently used in our analysis. To understand the intuition behind the definition of $Z_k$, suppose that $\|x_k-x^*\|_c$ is uniformly bounded by a deterministic constant, i.e., the case $D<0$ in Proposition \ref{prop:worst_case_bound}. Then we can choose $T_k(\delta)$ to be the constant worst case bound provided in Proposition \ref{prop:worst_case_bound} (3), which implies $\mathds{1}_{E_k(\delta)}=1$ a.s. In this case, $Z_k$ becomes the standard log-MGF. The fact that we do not have such a strong boundedness property motivates us to introduce the additional parameters $T_k(\delta)$ and $\mathds{1}_{E_k(\delta)}$, which are crucial for the development of our bootstrapping argument that is used to overcome the challenge of having unbounded iterates.

\subsubsection{Bounding the Log-MGFs}\label{subsec:step1}
To bound $Z_k$, we first derive a recursive bound connecting $Z_k$ and $Z_{k+1}$, which further implies an outright bound on $Z_k$ by solving the recursion. To state the result, we need the following notation. Let $D_0=2(1-\tilde{\gamma}_c)\in (0,1)$, $D_1=4\sigma^2/\ell_{cM}^2$, and $D_2=2L(2+\sigma)^2u_{cM}^2/(\mu\ell_{cs}^2)$. The parameter $\theta$ is chosen as $\theta=D_0\|x_0-x^*\|_c^2/[8D_1((1+\|x^*\|_c)^2+\|x_0-x^*\|_c^2)]$. The stepsizes are chosen according to the following condition.
\begin{condition}\label{condition:stepsize_multi1}
Let $\alpha_k=\alpha/(k+h)$ for all $k\geq 0$, where $\alpha>2/D_0$ and $h>1$ are chosen such that $\alpha_0\leq \min(1,D_0,D_0/(4D_2))$.
\end{condition}

Now, we are ready to state a recursive bound of $Z_k$ in the following proposition.
\begin{proposition}\label{prop:Bound-log-MGF}
	It holds that
	\begin{align}\label{prop:Bound-log-MGF-eq1}
		Z_{k+1}\leq \exp\left(-\frac{\alpha D_0/2-1}{\alpha}\alpha_k\right)Z_k+2\alpha_k^2 \lambda_k D_2(1+\|x^*\|_c)^2,\quad \forall\,k\geq 0.
	\end{align}
\end{proposition}
\begin{proof}[Proof of Proposition \ref{prop:Bound-log-MGF}]
Using the smoothness property of the generalized Moreau envelope $M(\cdot)$ (cf. Lemma \ref{prop:Moreau} (1)) and the update equation (\ref{eq:stochastic-approximation}), we have for all $k\geq 0$ that
\begin{align*}
	M(x_{k+1}\!-\!x^*) 
	\leq\;& M(x_k\!-\!x^*) \!+\! \alpha_k \nabla M(x_k\!-\!x^*)^\top ( F(x_k,Y_k) \!-\! x_k )\!+\!\frac{L\alpha_k^2}{2\mu} \| F(x_k,Y_k) \!-\! x_k \|_s^2\\
	=\;& M(x_k\!-\!x^*) \!+\! \alpha_k \nabla M(x_k\!-\!x^*)^\top ( \Bar{F}(x_k)\!-\!x_k) +\frac{L\alpha_k^2}{2\mu} \left\| F(x_k,Y_k) \!-\! x_k \right\|_s^2\\
	&+ \alpha_k \nabla M(x_k-x^*)^\top (F(x_k,Y_k) - \Bar{F}(x_k))\\
	\leq \;& (1- 2\alpha_k(1-\tilde{\gamma}_c)) M(x_k-x^*) +\alpha_k \nabla M(x_k-x^*)^\top ( F(x_k,Y_k) - \Bar{F}(x_k) )\\
	&+\frac{L\alpha_k^2}{2\mu} \left\| F(x_k,Y_k) \!-\! x_k \right\|_s^2,
\end{align*}
where the last inequality follows from \cite[Lemma A.1]{chen2021finite}. Therefore, by first multiplying $\lambda_{k+1}\mathds{1}_{E_{k+1}(\delta)}$ and then taking exponential on both sides of the previous inequality, we have
\begin{align*}
    &\exp\left(\lambda_{k+1}\mathds{1}_{E_{k+1}(\delta)}M(x_{k+1}-x^*)\right)\\
    \leq \;&\exp\left(\lambda_{k+1}\mathds{1}_{E_{k+1}(\delta)}(1- 2\alpha_k(1-\tilde{\gamma}_c)) M(x_k-x^*)\right)\\ 
    &\times \exp\left(\alpha_k\lambda_{k+1}\mathds{1}_{E_{k+1}(\delta)} \nabla M(x_k-x^*)^\top ( F(x_k,Y_k) - \Bar{F}(x_k) )\right)\\
	&\times \exp\left(\frac{L\alpha_k^2}{2\mu} \lambda_{k+1}\mathds{1}_{E_{k+1}(\delta)}\left\| F(x_k,Y_k) - x_k \right\|_s^2\right)\\
	\leq \;&\exp\left(\lambda_{k+1}\mathds{1}_{E_k(\delta)}(1- 2\alpha_k(1-\tilde{\gamma}_c)) M(x_k-x^*)\right)\\ &\times \exp\left(\alpha_k\lambda_{k+1}\mathds{1}_{E_k(\delta)} \nabla M(x_k-x^*)^\top ( F(x_k,Y_k) - \Bar{F}(x_k) )\right)\\
	&\times \exp\left(\frac{L\alpha_k^2}{2\mu} \lambda_{k+1}\mathds{1}_{E_k(\delta)}\left\| F(x_k,Y_k) - x_k \right\|_s^2\right),
\end{align*}
where in the last inequality we used the fact that $\{E_k(\delta)\}$ is a decreasing sequence of events (which implies $\mathds{1}_{E_{k+1}(\delta)}\leq \mathds{1}_{E_k(\delta)}$). Taking expectation conditioning on $\mathcal{F}_k$ on both sides of the previous inequality , we obtain
\begin{align}
	&\mathbb{E}\left[\exp\left(\lambda_{k+1}\mathds{1}_{E_{k+1}(\delta)}M(x_{k+1}-x^*)\right)\mid\mathcal{F}_k\right]\nonumber\\
	\leq \;&\exp\left(\lambda_{k+1}\mathds{1}_{E_k(\delta)}(1- 2\alpha_k(1-\tilde{\gamma}_c))M(x_k-x^*)\right)\nonumber\\
	&\times \mathbb{E}\left[\exp\left(\alpha_k\lambda_{k+1}\mathds{1}_{E_k(\delta)} \nabla M(x_k-x^*)^\top ( F(x_k,Y_k) - \Bar{F}(x_k))\right)\right.\nonumber\\
	&\left.\times \exp\left(\frac{L\alpha_k^2}{2\mu} \lambda_{k+1}\mathds{1}_{E_k(\delta)}\left\| F(x_k,Y_k) - x_k \right\|_s^2\right)\;\bigg|\;\mathcal{F}_k\right]\nonumber\\
	\leq \;& \exp\left(\lambda_{k+1}\mathds{1}_{E_k(\delta)}(1- 2\alpha_k(1-\tilde{\gamma}_c))M(x_k-x^*)\right)\nonumber\\
	&\times \underbrace{\mathbb{E}\left[\exp\left(2\alpha_k\lambda_{k+1}\mathds{1}_{E_k(\delta)} \nabla M(x_k-x^*)^\top ( F(x_k,Y_k) - \Bar{F}(x_k) )\right)\mid\mathcal{F}_k\right]^{1/2}}_{T_1}\nonumber\\
	&\times \underbrace{\mathbb{E}\left[\exp\left(\frac{L\alpha_k^2}{\mu} \lambda_{k+1}\mathds{1}_{E_k(\delta)}\left\| F(x_k,Y_k) - x_k \right\|_s^2\right)\;\middle|\;\mathcal{F}_k\right]^{1/2}}_{T_2},\label{eq:multi_T1T2}
\end{align}
where the last line follows from the conditional Cauchy–Schwarz inequality. Next, we bound the terms $T_1$ and $T_2$. To bound the term $T_1$, we will use the conditional Hoeffding's lemma. Observe that Assumption \ref{as:unbiased} implies
\begin{align}\label{eq:Hoeffding_unbiased}	&\mathbb{E}\left[2\alpha_k\lambda_{k+1}\mathds{1}_{E_k(\delta)} \nabla M(x_k-x^*)^\top ( F(x_k,Y_k) - \Bar{F}(x_k) )\mid\mathcal{F}_k\right]\nonumber\\
	=\;&2\alpha_k\lambda_{k+1}\mathds{1}_{E_k(\delta)}\nabla M(x_k-x^*)^\top( \mathbb{E}\left[F(x_k,Y_k)\mid\mathcal{F}_k\right] - \Bar{F}(x_k) )\nonumber\\
	=\;&0.
\end{align}
In addition, we have
\begin{align}
	&\left|2\alpha_k\lambda_{k+1}\mathds{1}_{E_k(\delta)} M(x_k-x^*)^\top (F(x_k,Y_k) - \Bar{F}(x_k))\right|\nonumber\\
	\leq \;&2\alpha_k\lambda_{k+1}\mathds{1}_{E_k(\delta)}\| \nabla M(x_k-x^*)\|_{M}^*\|F(x_k,Y_k) - \Bar{F}(x_k)\|_M\label{eq:Hoeffding_bound1}\\
	\leq \;&2\alpha_k\lambda_{k+1}\mathds{1}_{E_k(\delta)}\|x_k-x^*\|_M\| \nabla \|x_k-x^*\|_M\|_{M}^*\|F(x_k,Y_k) - \Bar{F}(x_k)\|_M\label{eq:Hoeffding_bound2}\\
	\leq \;&2\alpha_k\lambda_{k+1}\mathds{1}_{E_k(\delta)}\|x_k-x^*\|_M\|F(x_k,Y_k) - \Bar{F}(x_k)\|_M\label{eq:Hoeffding_bound3}\\
	\leq \;&\frac{2\alpha_k\lambda_{k+1}\mathds{1}_{E_k(\delta)}}{\ell_{cM}}\|x_k-x^*\|_M\|F(x_k,Y_k) - \Bar{F}(x_k)\|_c\label{eq:Hoeffding_bound4}\\
	\leq \;&\frac{2\sigma\alpha_k\lambda_{k+1}\mathds{1}_{E_k(\delta)}}{\ell_{cM}}\|x_k-x^*\|_M(1+\|x_k\|_c)\label{eq:Hoeffding_bound5}\\
	\leq \;&\frac{2\sigma\alpha_k\lambda_{k+1}\mathds{1}_{E_k(\delta)}}{\ell_{cM}}\|x_k-x^*\|_M(1+\|x_k-x^*\|_c+\|x^*\|_c)\nonumber\\
	\leq \;&\frac{2\sigma\alpha_k\lambda_{k+1}\mathds{1}_{E_k(\delta)}}{\ell_{cM}}\|x_k-x^*\|_M(1+\|x^*\|_c+T_k^{1/2}(\delta))\label{eq:Hoeffding_bound_last},
\end{align}
where Eq. (\ref{eq:Hoeffding_bound1}) follows from H\"{o}lder's inequality with $\|\cdot\|_M^*$ being the dual norm of $\|\cdot\|_M$, Eq. (\ref{eq:Hoeffding_bound2}) follows from $M(x)=\frac{1}{2}\|x\|_M^2$ (cf. Lemma \ref{prop:Moreau}) and the chain rule of calculus, Eq. (\ref{eq:Hoeffding_bound3}) follows from $\| \nabla \|x\|_M\|_M^*\leq 1$ for all $x\in\mathbb{R}^d$ \cite[Lemma 2.6]{shalev2012online}, Eq. (\ref{eq:Hoeffding_bound4}) follows from Lemma \ref{prop:Moreau}, Eq. (\ref{eq:Hoeffding_bound5}) follows from Assumption \ref{as:multi}, and Eq. (\ref{eq:Hoeffding_bound_last}) follows from $\|x_k-x^*\|_c^2\leq T_k(\delta)$ on $E_k(\delta)$. Eqs. (\ref{eq:Hoeffding_unbiased}) and (\ref{eq:Hoeffding_bound_last}) together enable us to use the conditional Hoeffding's lemma on the random variable $2\alpha_k\lambda_{k+1}\mathds{1}_{E_k(\delta)} M(x_k-x^*)^\top (F(x_k,Y_k) - \Bar{F}(x_k))$, which gives us
\begin{align}
	T_1=\;&\mathbb{E}\left[\exp\left(2\alpha_k\lambda_{k+1}\mathds{1}_{E_k(\delta)} M(x_k-x^*)^\top (F(x_k,Y_k) - \Bar{F}(x_k))\right)\,\middle|\,\mathcal{F}_k\right]^{1/2}\nonumber\\
	\leq \;&\exp\left( \frac{\sigma^2\alpha_k^2\lambda_{k+1}^2\mathds{1}_{E_k(\delta)}}{\ell_{cM}^2}\|x_k-x^*\|_M^2(1+\|x^*\|_c+T_k^{1/2}(\delta))^2 \right)\nonumber\\
	\leq  \;&\exp\left(\frac{4\sigma^2\alpha_k^2\lambda_{k+1}^2\mathds{1}_{E_k(\delta)}}{\ell_{cM}^2}M(x_k-x^*)[(1+\|x^*\|_c)^2+T_k(\delta))]\right),\label{eq:bound_on_T1}
\end{align}
where the last inequality follows from $M(x)=\frac{1}{2}\|x\|_M^2$ (cf. Lemma \ref{prop:Moreau}) and $(a+b)^2\leq 2(a^2+b^2)$ for all $a,b\in\mathbb{R}$.

Next, we bound the term $T_2$ in Eq. (\ref{eq:multi_T1T2}). Observe that
\begin{align}
	\left\| F(x_k,Y_k) - x_k \right\|_s
	\leq \;&\frac{1}{\ell_{cs}}\left\| F(x_k,Y_k) - x_k \right\|_c\nonumber\\
	=\;&\frac{1}{\ell_{cs}}\left\| F(x_k,Y_k) - \Bar{F}(x_k)+\Bar{F}(x_k)-\Bar{F}(x^*)+x^*-x_k \right\|_c\nonumber\\
	\leq \;&\frac{1}{\ell_{cs}}\left(\| F(x_k,Y_k) - \Bar{F}(x_k)\|_c+\|\Bar{F}(x_k)-\Bar{F}(x^*)\|_c+\|x^*-x_k \|_c\right)\nonumber\\
	\leq \;&\frac{1}{\ell_{cs}}\left(\sigma(1+\|x_k\|_c)+(\gamma_c+1)\|x_k-x^*\|_c\right)\label{eq:the_only_one}\\
	\leq \;&\frac{1}{\ell_{cs}}\left(\sigma(1+\|x_k-x^*\|_c+\|x^*\|_c)+2\|x_k-x^*\|_c\right)\nonumber\\
	\leq \;&\frac{1}{\ell_{cs}}\left((2+\sigma)\|x_k-x^*\|_c+\sigma(1+\|x^*\|_c)\right),\nonumber
\end{align}
where Eq. (\ref{eq:the_only_one}) follows from Assumptions \ref{as:contraction} and \ref{as:multi}.
Therefore, we have
\begin{align*}
	T_2
	=\;&\mathbb{E}\left[\exp\left(\frac{L\alpha_k^2}{\mu} \lambda_{k+1}\mathds{1}_{E_k(\delta)}\left\| F(x_k,Y_k) - x_k \right\|_s^2\right)\;\middle|\;\mathcal{F}_k\right]^{1/2}\\
	\leq \;&\exp\left(\frac{L\alpha_k^2\lambda_{k+1}\mathds{1}_{E_k(\delta)}}{2\mu\ell_{cs}^2} \left((2+\sigma)\|x_k-x^*\|_c+\sigma(1+\|x^*\|_c)\right)^2\right)\\
	\leq \;&\exp\left(\frac{L\alpha_k^2\lambda_{k+1}\mathds{1}_{E_k(\delta)}}{\mu\ell_{cs}^2} \left((2+\sigma)^2\|x_k-x^*\|_c^2+\sigma^2(1+\|x^*\|_c)^2\right)\right)\\
	\leq \;&\exp\left(\frac{L\alpha_k^2\lambda_{k+1}\mathds{1}_{E_k(\delta)}}{\mu\ell_{cs}^2} \left((2+\sigma)^2u_{cM}^2\|x_k-x^*\|_M^2+\sigma^2(1+\|x^*\|_c)^2\right)\right)\\
	\leq  \;&\exp\left(\frac{2L(2+\sigma)^2u_{cM}^2\alpha_k^2\lambda_{k+1}\mathds{1}_{E_k(\delta)}}{\mu\ell_{cs}^2}M(x_k-x^*)+ \frac{L\sigma^2\alpha_k^2\lambda_{k+1}}{\mu\ell_{cs}^2}(1+\|x^*\|_c)^2\right).
\end{align*}
Using the upper bounds we obtained for $T_1$ (cf. Eq. (\ref{eq:bound_on_T1})) and $T_2$ (cf. the previous inequality) in Eq. (\ref{eq:multi_T1T2}) , we have
\begin{align}
	&\mathbb{E}\left[\exp\left(\lambda_{k+1}\mathds{1}_{E_{k+1}(\delta)}M(x_{k+1}-x^*)\right)\mid\mathcal{F}_k\right]\nonumber\\
	\leq \;& \exp\left(\lambda_{k+1}\mathds{1}_{E_k(\delta)}(1- 2\alpha_k(1-\tilde{\gamma}_c))M(x_k-x^*)\right)\nonumber\\
	&\times \exp\left(\frac{4\sigma^2\alpha_k^2\lambda_{k+1}^2\mathds{1}_{E_k(\delta)}}{\ell_{cM}^2}M(x_k-x^*)[(1+\|x^*\|_c)^2+T_k(\delta))] \right)\nonumber\\
	&\times \exp\left(\frac{2L(2\!+\!\sigma)^2u_{cM}^2\alpha_k^2\lambda_{k+1}\mathds{1}_{E_k(\delta)}}{\mu\ell_{cs}^2}M(x_k\!-\!x^*)\!+\! \frac{L\sigma^2\alpha_k^2\lambda_{k+1}}{\mu\ell_{cs}^2}(1\!+\!\|x^*\|_c)^2\right)\nonumber\\
	=\;&\exp\bigg(\lambda_{k+1}\mathds{1}_{E_k(\delta)}M(x_k-x^*)\bigg(1- 2\alpha_k(1-\tilde{\gamma}_c)\nonumber\\
	&\left.+\frac{4\sigma^2\alpha_k^2\lambda_{k+1}[(1+\|x^*\|_c)^2+T_k(\delta))]}{\ell_{cM}^2}+\frac{2L(2+\sigma)^2u_{cM}^2\alpha_k^2}{\mu\ell_{cs}^2}\right)\bigg)\nonumber\\
	&\times \exp\left( \frac{L\sigma^2\alpha_k^2\lambda_{k+1}}{\mu\ell_{cs}^2}(1+\|x^*\|_c)^2\right)\nonumber.
\end{align}
Recall that we have denoted $D_0=2(1-\tilde{\gamma}_c)$, $D_1=4\sigma^2/\ell_{cM}^2$, and $D_2=2L(2+\sigma)^2u_{cM}^2/(\mu\ell_{cs}^2)$. In addition, to simplify the notation, let
\begin{align}
    T_3=\;&\frac{\lambda_{k+1}}{\lambda_k}\left(1- \alpha_kD_0+\alpha_k^2\lambda_{k+1}D_1((1+\|x^*\|_c)^2+T_k(\delta))+\alpha_k^2D_2\right),\label{eq:def:T3}\\
    T_4=\;&\alpha_k^2\frac{\lambda_{k+1}}{\lambda_k}.\label{eq:def:T4}
\end{align}
Then, the previous inequality reads 
\begin{align}  &\mathbb{E}\left[\exp\left(\lambda_{k+1}\mathds{1}_{E_{k+1}(\delta)}M(x_{k+1}-x^*)\right)\mid\mathcal{F}_k\right]\nonumber\\
\leq \;&\exp\left(T_3\lambda_k\mathds{1}_{E_k(\delta)}M(x_k-x^*)\right)\exp\left( T_4 \lambda_k D_2(1+\|x^*\|_c)^2\right).\label{eq:multi_T3T4}
\end{align}
Next, we bound the terms $T_3$ and $T_4$ defined in Eqs. (\ref{eq:def:T3}) and (\ref{eq:def:T4}), respectively. On the one hand, it is clear that $T_3\geq 0$ because $\alpha_0\leq D_0$. On the other hand, since $\lambda_k=\theta\alpha_k^{-1} T_k(\delta)^{-1}$, we have
\begin{align}
 T_3=\;&\frac{\lambda_{k+1}}{\lambda_k}\left(1- \alpha_kD_0+\alpha_k^2\lambda_{k+1}D_1((1+\|x^*\|_c)^2+T_k(\delta))+\alpha_k^2D_2\right)\nonumber\\
=\;&\frac{\alpha_k T_k(\delta)}{\alpha_{k+1}T_{k+1}(\delta)}\left(1- \alpha_kD_0+\frac{\theta \alpha_k^2D_1((1+\|x^*\|_c)^2+T_k(\delta))}{\alpha_{k+1}T_{k+1}(\delta)}+\alpha_k^2D_2\right)\nonumber\\
\leq \;&\frac{\alpha_k }{\alpha_{k+1}}\left(1- \alpha_kD_0+\frac{\theta \alpha_k^2D_1((1+\|x^*\|_c)^2+T_k(\delta))}{\alpha_{k+1}T_{k+1}(\delta)}+\alpha_k^2D_2\right),\label{eq:bound_T3_first}
\end{align}
where the last line follows from $\{T_k(\delta)\}$ being a non-decreasing sequence. To proceed, observe that
\begin{align}
    \frac{\theta \alpha_k^2D_1((1+\|x^*\|_c)^2+T_k(\delta))}{\alpha_{k+1}T_{k+1}(\delta)}\leq \;&\frac{\theta \alpha_k^2D_1}{\alpha_{k+1}}\frac{(1+\|x^*\|_c)^2+T_{k+1}(\delta)}{T_{k+1}(\delta)}\label{eq:bound_T3_1}\\
    \leq \;&\frac{\theta \alpha_k^2D_1}{\alpha_{k+1}}\frac{(1+\|x^*\|_c)^2+\|x_0-x^*\|_c^2}{\|x_0-x^*\|_c^2}\label{eq:bound_T3_3}\\
    =\;&\frac{ \alpha_k^2D_0}{8\alpha_{k+1}}\label{eq:bound_T3_4}\\
    \leq \;&\frac{ \alpha_kD_0}{4},\label{eq:bound_T3_last}
\end{align}
where Eq. (\ref{eq:bound_T3_1}) follows from $\{T_k(\delta)\}$ being a non-decreasing sequence, Eq. (\ref{eq:bound_T3_3}) follows from $\|x_0-x^*\|_c^2\leq T_0(\delta)\leq T_{k+1}(\delta)$ and the numerical inequality $(a+c)/(b+c)\leq  a/b$ for any $a,b>0$ with $a\geq b$ and $c>0$, Eq. (\ref{eq:bound_T3_4}) follows from choosing $\theta=D_0\|x_0-x^*\|_c^2/[8D_1((1+\|x^*\|_c)^2+\|x_0-x^*\|_c^2)]$, and Eq. (\ref{eq:bound_T3_last}) follows from $\alpha_k/\alpha_{k+1}\leq (h+1)/h\leq 2$ (cf. Condition \ref{condition:stepsize_multi1}). Therefore, using the previous inequality in Eq. (\ref{eq:bound_T3_first}), we have
\begin{align}
    T_3\leq \;&\frac{\alpha_k }{\alpha_{k+1}}\left(1- \alpha_kD_0+\frac{ \alpha_kD_0}{4}+\alpha_k^2D_2\right)\nonumber\\
    \leq \;&\frac{\alpha_k }{\alpha_{k+1}}\left(1- \frac{\alpha_kD_0}{2}\right)\label{eq:final_bound_T3_1}\\
    =\;&\frac{k+1+h}{k+h}\left(1-\frac{\alpha D_0}{2(k+h)}\right)\nonumber\\
    \leq \;&\left(\frac{k+h+1}{k+h}\right)\exp\left(-\frac{D_0\alpha}{2(k+h)}\right)\label{eq:final_bound_T3_2}\\
    = \;&\left[\left(1+\frac{1}{k+h}\right)^{k+h}\right]^{1/(k+h)}\exp\left(-\frac{D_0\alpha}{2(k+h)}\right)\nonumber\\
    \leq \;&\exp\left(\frac{1}{k+h}-\frac{D_0\alpha}{2(k+h)}\right)\label{eq:final_bound_T3_3}\\
    =\;&\exp\left(-\frac{\alpha D_0/2-1}{k+h}\right)\nonumber\\
    =\;&\exp\left(-\frac{\alpha D_0/2-1}{\alpha}\alpha_k\right),\label{final_bound_on_T3}
\end{align}
where Eq. (\ref{eq:final_bound_T3_1}) follows from $\alpha_k\leq D_0/(4D_2)$ (cf. Condition \ref{condition:stepsize_multi1}), Eq. (\ref{eq:final_bound_T3_2}) follows from $1+x\leq e^x$ for all $x\in\mathbb{R}$, and Eq. (\ref{eq:final_bound_T3_3}) follows from $(1+1/x)^x\leq e$ for all $x>0$.
Note that $T_3<1$ because $\alpha>2/D_0$ (cf. Condition \ref{condition:stepsize_multi1}). 

Now, consider the term $T_4$ from Eq. (\ref{eq:def:T4}). We have by definition of $\lambda_k$ that
\begin{align*}
	T_4=\alpha_k^2\frac{\lambda_{k+1}}{\lambda_k}
	=\alpha_k^2\frac{\alpha_kT_k(\delta)}{\alpha_{k+1}T_{k+1}(\delta)}
	\leq \frac{\alpha_k^3}{\alpha_{k+1}}
	\leq 2\alpha_k^2,
\end{align*}
where we used $\{T_k(\delta)\}$ being a non-decreasing sequence and $\alpha_k/\alpha_{k+1}\leq (h+1)/h\leq 2$ (cf. Condition \ref{condition:stepsize_multi1}).
Using the upper bounds we obtained for the terms $T_3$ (cf. Eq. (\ref{final_bound_on_T3})) and $T_4$ (cf. the previous inequality) in Eq. (\ref{eq:multi_T3T4}) , we have
\begin{align}\label{eq10:MGF_recursive}
	&\mathbb{E}\left[\exp\left(\lambda_{k+1}\mathds{1}_{E_{k+1}(\delta)}M(x_{k+1}-x^*)\right)\mid\mathcal{F}_k\right]\nonumber\\
	\leq  \;&\exp\!\left(\!\exp\!\left(\!-\!\frac{\alpha D_0/2\!-\!1}{\alpha}\alpha_k\right)\lambda_k\mathds{1}_{E_k(\delta)}M(x_k\!-\!x^*)\!\right)\!\exp\left( 2\alpha_k^2 \lambda_k D_2(1\!+\!\|x^*\|_c)^2\right).
\end{align}
In view of the definition of $Z_k$ in Eq. (\ref{def:Z_k}), a recursive bound of $Z_k$ can be obtained by first taking total expectation, then taking the logarithm and finally using Jensen's inequality on both sides of Eq. (\ref{eq10:MGF_recursive}). This proves Proposition \ref{prop:Bound-log-MGF}. 
\end{proof}

Repeatedly using Eq. (\ref{prop:Bound-log-MGF-eq1}) of Proposition \ref{prop:Bound-log-MGF} yields an overall bound on $Z_k$. The result is stated in the following lemma, the proof of which involves only standard algebra manipulation and is deferred to Appendix \ref{pf:le:multi_solving_recursion}.

\begin{lemma}\label{le:MGFrecursive}
	It holds for all $k\geq 0$ that
	\begin{align}\label{prop:Bound-log-MGF-eq2}
		Z_k\leq  Z_0\left(\frac{h}{k+h}\right)^{\alpha D_0/2-1}+\frac{4\alpha e D_2 \theta}{\alpha D_0/2-1}\frac{(1+\|x^*\|_c)^2}{\|x_0-x^*\|_c^2},\quad \forall\,k\geq 0.
	\end{align}
\end{lemma}

Now that we have successfully established a recursive bound and an overall bound of our modified log-MGF $Z_k$, the next step is to construct a supermartingale using $Z_k$ and apply Ville's maximal inequality to finish proving Proposition \ref{prop:blueprint}.

\subsubsection{An Exponential Supermartingale}

For any $k\geq 0$, let $\overline{M}_k$ be defined as $\overline{M}_k=\exp(\lambda_k\mathds{1}_{E_k(\delta)} M(x_k-x^*)-D_3\sum_{i=0}^{k-1}\alpha_i)$,
where $D_3= D_0D_2/(4D_1) $. We next show that $\{\overline{M}_k\}$ is a supermartingale with respect to the filtration $\{\mathcal{F}_k\}$.  It is clear that $\{\overline{M}_k\}$ is adapted to $\{\mathcal{F}_k\}$, and is finite in expectation (cf. Lemma \ref{le:MGFrecursive}). In addition, for any $k\geq 0$, we have by Eq. (\ref{eq10:MGF_recursive}) that
\begin{align}
	&\mathbb{E}\left[\exp\left(\lambda_{k+1}\mathds{1}_{E_{k+1}(\delta)}M(x_{k+1}-x^*)\right)\mid\mathcal{F}_k\right]\nonumber\\
	\leq  \;&\exp\left(\lambda_k\mathds{1}_{E_k(\delta)}M(x_k-x^*)\right)\exp\left( 2\alpha_k^2 \lambda_k D_2(1+\|x^*\|_c)^2\right)\nonumber\\
	= \;&\exp\left(\lambda_k\mathds{1}_{E_k(\delta)}M(x_k-x^*)\right)\exp\left(  \frac{2\alpha_k\theta D_2(1+\|x^*\|_c)^2}{T_k(\delta)} \right)\label{eq:SM-1}\\
	\leq  \;&\exp\left(\lambda_k\mathds{1}_{E_k(\delta)}M(x_k-x^*)\right)\exp\left(  \frac{2\alpha_k\theta D_2(1+\|x^*\|_c)^2}{\|x_0-x^*\|_c^2} \right)\label{eq:SM-2}\\
	\leq   \;&\exp\left(\lambda_k\mathds{1}_{E_k(\delta)}M(x_k-x^*)\right)\exp\left(  \frac{\alpha_k D_0D_2}{4D_1}\right)\label{eq:SM-3}\\
	= \;&\exp\left(\lambda_k\mathds{1}_{E_k(\delta)}M(x_k-x^*)\right)\exp\left( D_3\alpha_k\right),\label{eq:SM-4}
\end{align}
where Eq. (\ref{eq:SM-1}) follows from $\lambda_k=\theta\alpha_k^{-1}T_k(\delta)^{-1}$, Eq. (\ref{eq:SM-2}) follows from $\|x_0-x_k\|_c^2\leq T_0(\delta)\leq T_k(\delta)$, Eq. (\ref{eq:SM-3}) follows from choosing $\theta=D_0\|x_0-x^*\|_c^2/[8D_1((1+\|x^*\|_c)^2+\|x_0-x^*\|_c^2)]$, and Eq. (\ref{eq:SM-4}) follows from our notation $D_3= D_0D_2/(4D_1) $. The previous inequality implies $\mathbb{E}\left[\overline{M}_{k+1}\mid\mathcal{F}_k\right]\leq \overline{M}_k$. Therefore,
the random process $\{\overline{M}_k\}$ is a supermartingale adapted to $\{\mathcal{F}_k\}$.

To this end, we have shown that $\{\overline{M}_k\}$ is a supermartingale, and provided a bound on the expectation of $\overline{M}_k$ (cf. Lemma \ref{le:MGFrecursive}). Our next step is to use Ville's maximal inequality to establish the first maximal concentration bound. The result is stated in the following proposition, the proof of which is a standard application of Ville's maximal inequality, and is deferred to Appendix \ref{pf:prop:ville}.

\begin{proposition}\label{prop:ville}
	For any $\delta'\in (0,1)$ and $K\geq 0$, the following inequality holds with probability at least $1-\delta'$:
	\begin{align}
		\sup_{k\geq K}\{\lambda_k\mathds{1}_{E_k(\delta)} \|x_k-x^*\|_c^2\}
	\leq\,& 2u_{cM}^2\log(1/\delta')+\frac{2u_{cM}^2D_0}{16\alpha_0D_1\ell_{cM}^2}\left(\frac{h}{K+h}\right)^{\alpha D_0/2-1}\nonumber\\
 &+\frac{16u_{cM}^2\alpha e D_2 \theta}{\alpha D_0-2}\frac{(1+\|x^*\|_c)^2}{\|x_0-x^*\|_c^2}+2\alpha D_3 u_{cM}^2\log\left(\frac{k-1+h}{K-1+h}\right).\label{eq:first_concentration}
	\end{align}
\end{proposition}

\begin{remark}
Suppose that the iterates $\{x_k\}$ are uniformly bounded by a deterministic constant. Then, by choosing $T_k(\delta)$ as the uniform norm-square bound, we have $\lambda_k=c/\alpha_k$ for some constant $c$ and $\mathds{1}_{\{E_k(\delta)\}}=1$ a.s. In this case, multiplying $\alpha_k$ on both sides of Eq. (\ref{eq:first_concentration}), we obtain the desired concentration bound, which has an $\tilde{\mathcal{O}}(1/k)$ rate of convergence and an exponentially small tail $\mathcal{O}(\log(1/\delta'))$. However, in the case of multiplicative noise, the iterates $\{x_k\}$ in general do not admit a uniform bound. Since $\lambda_k=\theta \alpha_k^{-1} T_k(\delta)^{-1}$ and $T_k(\delta)$ can be an increasing function, Eq. (\ref{eq:first_concentration}) does not provide us with the desired rate. 
\end{remark}

Next, we carry out the final step in proving Proposition \ref{prop:blueprint}.
For simplicity of notation, denote the right-hand side of Eq. (\ref{eq:first_concentration}) by $\epsilon(k,K,\delta')$. For any $K\geq 0$, observe that
\begin{align*}
	&\mathbb{P}(\lambda_k \|x_k-x^*\|_c^2 \leq \epsilon(k,K,\delta'), \; \forall\, k\geq K ) \\
	= \;&\mathbb{P}\left( \bigcap_{k=K}^\infty \{\lambda_k \|x_k-x^*\|_c^2 \leq \epsilon(k,K,\delta') \} \right)\\
	\geq \;& \mathbb{P}\left( \bigcap_{k=K}^\infty \{\lambda_k \mathds{1}_{E_k(\delta)} \|x_k-x^*\|_c^2  \leq \epsilon(k,K,\delta')\}\cap E_k(\delta) \right).
\end{align*}
To proceed, note that for any two events, $A$ and $B$, we have
\begin{align*}
    \mathbb{P}(A\cap B)=1-\mathbb{P}(A^c\cup B^c)\geq 1-\mathbb{P}(A^c)-\mathbb{P}(B^c)=\mathbb{P}(A)+\mathbb{P}(B)-1.
\end{align*}
Therefore, we have
\begin{align*}
	&\mathbb{P}(\|x_k-x^*\|_c^2 \leq \epsilon(k,K,\delta')/\lambda_k , \; \forall\, k\geq K )\\
 \geq \; &\mathbb{P}\left( \bigcap_{k=K}^\infty \{\lambda_k \mathds{1}_{E_k(\delta)} \|x_k-x^*\|_c^2  \leq \epsilon(k,K,\delta')\}\cap E_k(\delta) \right)\\
 \geq \;&\mathbb{P}\left( \bigcap_{k=K}^\infty \{\lambda_k \mathds{1}_{E_k(\delta)} \|x_k-x^*\|_c^2  \leq \epsilon(k,K,\delta')\} \right)+\mathbb{P} \left(\bigcap_{k=0}^\infty E_k(\delta)\right)-1\\
 = \;&\mathbb{P} \left( \lambda_k \mathds{1}_{E_k(\delta)} \|x_k-x^*\|_c^2  \leq  \epsilon(k,K,\delta'),\;\forall\;k\geq K \right)+\lim_{k\rightarrow\infty}\mathbb{P} \left( E_k(\delta)\right)-1\\
	\geq \;& (1-\delta')+(1-\delta)-1\\
 =\;&1-\delta-\delta'
\end{align*}
Using the definitions of $\epsilon(k,K,\delta')$ and $\lambda_k$, we arrive at the following result.

\begin{proposition}\label{prop:bootstrapping}
	For any $ \delta'\in (0,1-\delta)$ and $K\geq 0$, with probability at least $1-\delta-\delta'$, we have
	\begin{align*}
	\|x_k-x^*\|_c^2\leq \;&\frac{\alpha_k T_k(\delta)}{\theta} \bigg[2u_{cM}^2\log(1/\delta')+\frac{2u_{cM}^2D_0}{16\alpha_0D_1\ell_{cM}^2}\left(\frac{h}{K+h}\right)^{\alpha D_0/2-1}\\
 &+\frac{16u_{cM}^2\alpha e D_2 \theta}{\alpha D_0-2}\frac{(1+\|x^*\|_c)^2}{\|x_0-x^*\|_c^2}+2\alpha D_3 u_{cM}^2\log\left(\frac{k-1+h}{h-1}\right)\bigg],\;\forall\,k\geq K.
\end{align*}
\end{proposition}

Setting $K=0$ in this proposition yields Proposition \ref{prop:blueprint}, which is our bootstrapping blueprint. Note that Proposition \ref{prop:bootstrapping} is a stronger version of Proposition \ref{prop:blueprint} because the result holds for all $K\geq 0$. 

\subsection{Completing the Bootstrapping Argument}\label{subsec:proof-multi-bootstrapping}
We start with the worst-case bound derived in Proposition \ref{prop:worst_case_bound} and then iteratively improve the bound using Proposition \ref{prop:blueprint}, except in the last step, where we use Proposition \ref{prop:bootstrapping} to obtain the maximal bound for any $K\geq 0$. Since every time we apply Proposition \ref{prop:blueprint} the bound gets improved by a factor of roughly $1/k$, at some point we will arrive at a bound that is decreasing. After that, Proposition \ref{prop:blueprint} is no longer applicable because we require the bound to be non-decreasing to initialize the bootstrapping argument. To carry out the details, in the case where $D>0$, we assume that $2\alpha D$ is a positive integer, which is indeed without loss of generality because $D=\gamma_c+\sigma-1$ and if Assumption \ref{as:multi} holds with some $\sigma>0$, it also holds for all $\sigma'>\sigma$. In view of Proposition \ref{prop:worst_case_bound}, $\|x_k-x^*\|_c^2$ can be polynomially increasing at a rate of $\mathcal{O}(k^{2\alpha D})$. Therefore, we need to bootstrap $m:=2\alpha D+1$ times. When $D\leq 0$, since $\|x_k-x^*\|_c^2$ is either bounded a.s. by a deterministic constant or can grow at most logarithmically, we only need to bootstrap once. 
\begin{claim}\label{claim:finish-bootstrapping}
    The proof of Theorem \ref{thm:multi} is complete after the bootstrapping. 
\end{claim}
The proof of the above claim involves only standard algebra manipulation. See Appendix \ref{pf:le:first_bootstrap} for more details.

\section{Proof of Theorem \ref{thm:impossibility}}\label{sec:proof-impossibility}
    The following $2$ lemmas are needed to prove this theorem.
    \begin{lemma}\label{le:example4}
    Given $c_1,c_2>0$, suppose that there exist $C_1,C_2>0$ such that
    \begin{align*}
        \mathbb{P}\left((k+h)^{c_2}x_k^{c_1}\geq \epsilon^{c_1}\right)\leq C_1\exp\left(-C_2\epsilon^{c_1}\right),\quad \forall\;\epsilon>0,k\geq 0.
    \end{align*}
    Then, we have $\limsup_{k\to\infty} \mathbb{E}\left[\exp\left(\lambda (k+h)^{c_2} x_k^{c_1}\right)\right]<\infty$ for any $\lambda\in (0,C_2)$.
    \end{lemma}
    The proof of Lemma \ref{le:example4} (presented in Appendix \ref{pf:le:example4}) follows from using the formula $\mathbb{E}[X]=\int_0^\infty \mathbb{P}(X>x)dx$ for any positive random variable $X$.
    \begin{lemma}\label{le:numerical}
    Consider the function $\ell(x)=e^x-(1+c x)$ defined on $[0,\infty)$, where $c>1$. There exists $x_c>0$ such that $\ell(x)\leq 0$ for all $x\in [0,x_c]$ and $\ell(x)\geq 0$ for all $x\in [x_c,\infty)$.
\end{lemma}
The proof of Lemma \ref{le:numerical} follows from investigating the monotonicity of the function $\ell(\cdot)$ using its derivative. See Appendix \ref{pf:le:numerical} for more details. 

Next, we proceed to prove Theorem \ref{thm:impossibility}. 

\subsection{Proof of Theorem \ref{thm:impossibility} Part (1)} Since ${\tilde{\beta}}>2/(1+2\alpha(a+N-1))$, there exists $\epsilon>0$ such that
    \begin{align*}
        {\tilde{\beta}}>\frac{2}{1+2\alpha(a+N-1)/(1+\epsilon)}.
    \end{align*}
Applying Lemma \ref{le:numerical}, there exists $k=k(\epsilon)>0$ such that 
\begin{equation}\label{eq:tighter}
    \exp(\alpha_k(a+N-1)/(1+\epsilon)) \leq 1+\alpha_k(a+N-1),\quad \forall\; k\geq k(\epsilon).
\end{equation}
Now, for any $\lambda>0$, we have for any $k\geq  k(\epsilon)$ that
\begin{align}
    &\mathbb{E}\left[\exp\left(\lambda \left[(k+h)^{1/2} x_k\right]^{\tilde{\beta}}\right)\right]\nonumber\\ =\;&\mathbb{E}\left[\exp\left(\lambda x_0^{\tilde{\beta}} (k+h)^{{\tilde{\beta}}/2}\prod_{i=0}^{k-1}(1+\alpha_i( Y_i-1))^{\tilde{\beta}}\right)\right]\label{eq1:thm:impossibility}\\
    \geq \;&\frac{1}{(N+1)^k}\exp\left(\lambda x_0^{\tilde{\beta}} (k+h)^{{\tilde{\beta}}/2}\prod_{i=0}^{k-1}(1+\alpha_i (a+N-1))^{\tilde{\beta}}\right)\label{eq2:thm:impossibility}\\
    \geq \;&\frac{1}{(N+1)^k}\exp\left(\lambda x_0^{\tilde{\beta}} (k+h)^{{\tilde{\beta}}/2}\prod_{i=k(\epsilon)}^{k-1}(1+\alpha_i (a+N-1))^{\tilde{\beta}}\right)\nonumber\\
    \geq \;&\frac{1}{(N+1)^k}\exp\left(\lambda x_0^{\tilde{\beta}} (k+h)^{{\tilde{\beta}}/2}\exp\left({\tilde{\beta}} \sum_{i=k(\epsilon)}^{k-1}\frac{\alpha_i (a+N-1)}{(1+\epsilon)}\right)\right),\label{eq:example_before}
\end{align}
where Eq. (\ref{eq1:thm:impossibility}) follows from the update equation (\ref{algo:example}), Eq. (\ref{eq2:thm:impossibility}) follows from the distribution of $Y_k$ (cf. Example \ref{example:impossibility_result}), and Eq. (\ref{eq:example_before}) follows from Eq. (\ref{eq:tighter}).
Observe that
\begin{align*}
    \sum_{i=k(\epsilon)}^{k-1}\alpha_i= \sum_{i=k(\epsilon)}^{k-1}\frac{\alpha}{i+h}
    \geq \int_{k(\epsilon)}^k\frac{\alpha}{x+h}dx
    = 
    \alpha \ln\left(\frac{k+h}{k(\epsilon)+h}\right).
\end{align*}
Therefore, we have from the previous inequality and Eq. (\ref{eq:example_before}) that
\begin{align*}
    \mathbb{E}\left[\exp\left(\lambda \left[(k+h)^{1/2}\, x_k\right]^{\tilde{\beta}}\right)\right]
    \geq \exp\left(\lambda x_0^{\tilde{\beta}} \left(\frac{k+h}{k(\epsilon)+h}\right)^{\frac{{\tilde{\beta}}}{2}\left(1+\frac{2\alpha (a+N-1)}{1+\epsilon}\right)}-k\ln(N+1)\right).
\end{align*}
Since $\epsilon$ is chosen such that ${\tilde{\beta}}(1+2\alpha  (a+N-1)/(1+\epsilon))>2$, we have from the previous inequality that
\begin{align*}
        \liminf\limits_{k\to\infty} \mathbb{E}\left[\exp\left(\lambda \left[(k+h)^{1/2}\, x_k\right]^{\tilde{\beta}}\right)\right] = \infty,\quad \forall\;\lambda>0.
\end{align*}
As a result, according to Lemma \ref{le:example4},
there do not exist $K_1',K_2'>0$ such that 
\begin{align*}
    \mathbb{P}\left((k+h)^{1/2}\;x_k\geq \epsilon\right)\leq K_1'\exp\left(-K_2'\epsilon^{\tilde{\beta}}\right),\quad \forall\,\epsilon>0,k\geq 0. 
\end{align*}

\subsection{Proof of Theorem \ref{thm:impossibility} Part (2)} To begin with, according to Lemma \ref{le:numerical}, there exists $k_0>0$ such that
\begin{align*}
    \exp(\alpha_k(a+N-1)) \leq 1+\alpha_k(a+N-1),\quad \forall\; k\geq k_0.
\end{align*}
As a result, using the same analysis as in the proof of Part (1) of this theorem, we have for any $\lambda>0$ and $k\geq  k_0$ that
\begin{align*}
    &\mathbb{E}\left[\exp\left(\lambda (k+h)^{{\tilde{\beta}}'} x_k^{\tilde{\beta}}\right)\right]\\
    \geq\;& \frac{\exp\left(\lambda x_0^{\tilde{\beta}} (k+h)^{{\tilde{\beta}}'}\exp\left({\tilde{\beta}} \sum_{i=k_0}^{k-1}\alpha_i (a+N-1)\right)\right)}{(N+1)^k}\\
    \geq \;&\exp\left(\lambda x_0^{\tilde{\beta}} (k\!+\!h)^{{\tilde{\beta}}'}\!\exp\left(\frac{\alpha {\tilde{\beta}} (a\!+\!N\!-\!1)}{(1\!-\!z)}((k\!+\!h)^{1-z}\!-\!(k_0\!+\!h)^{1-z})\right)\!-\!k\ln(N\!+\!1)\right),
\end{align*}
where the last line follows from
\begin{align*}
    \sum_{i=k_0}^{k-1}\alpha_i
    \geq \int_{k_0}^k\frac{\alpha}{(x+h)^z}dx
    = 
    \frac{\alpha}{1-z}((k+h)^{1-z}-(k_0+h)^{1-z}).
\end{align*}
Therefore, we always have
\begin{align*}
        \liminf\limits_{k\to\infty} \mathbb{E}\left[\exp\left(\lambda (k+h)^{{\tilde{\beta}}'} x_k^{\tilde{\beta}}\right)\right] = \infty,\quad \forall\;\lambda>0.
\end{align*}
As a result, according to Lemma \ref{le:example4}, there do not exist $\bar{K}_1',\bar{K}_2'>0$ such that 
\begin{align*}
    \mathbb{P}\left((k+h)^{{\tilde{\beta}}'/{\tilde{\beta}}}\;x_k\geq \epsilon\right)\leq \bar{K}_1'\exp\left(-\bar{K}_2'\epsilon^{\tilde{\beta}}\right),\quad \forall\,\epsilon>0,k\geq 0.
\end{align*}

\section{Proof of Theorem \ref{thm:additive}}\label{sec:proof-additive}
The high-level idea for the proof is similar to that of Theorem \ref{thm:multi}. Specifically, we first establish a bound on the MGF of the generalized Moreau envelope and then use Ville's maximal inequality to derive a maximal concentration bound. The difference here is that we do not need to use the bootstrapping argument due to the additive nature of the noise.

\subsection{Bounding the log-MGF of the Generalized Moreau Envelope}
Recall that we defined the generalized Moreau envelope as $M(x) = \min_{u\in\mathbb{R}^d} \{ \frac{1}{2}\|u\|_c^2 + \frac{1}{2\mu} \|x-u\|_s^2\}$ for all $x\in\mathbb{R}^d$. The definition of the constants $\ell_{cs}$, $u_{cs}$, $\ell_{cM}$, and $u_{c,M}$, and the requirement on choosing $\mu$ are identical to that in Section \ref{sec:proof-multi}. The properties of $M(\cdot)$ were summarized in Lemma \ref{prop:Moreau}. Specifically, $M(\cdot)$ is an $L/\mu$ -- smooth function with respect to $\|\cdot\|_s$, and can be written as $M(x)=\frac{1}{2}\|x\|_M^2$ for some norm $\|\cdot\|_M$. 
The following constants will be frequently used in our derivation. The relation between them and $\ell_{cs}$ and $u_{cs}$ are also presented.
\begin{align*}
\ell_{M2}\|\cdot\|_2\leq &\|\cdot\|_M\leq u_{M2}\|\cdot\|_2,\;& \ell_{M2}=\;&\ell_{c2}(1+\mu u_{cs}^2)^{-1/2},& u_{M2}=u_{c2}(1+\mu \ell_{cs}^2)^{-1/2},\\
    \ell_{sM}\|\cdot\|_M\leq &\|\cdot\|_s\leq u_{sM}\|\cdot\|_M,\; & \ell_{sM}=\;&(1+\mu \ell_{cs}^2)^{1/2}u_{cs}^{-1}, &u_{sM}=(1+\mu u_{cs}^2)^{1/2}\ell_{cs}^{-1}.
\end{align*}
We also define $u_{cM,*}$ such that $\|x\|_M\leq  u_{cM,*}\|x\|_{c,*}$ for all $x\in\mathbb{R}^d$, where $\|\cdot\|_c^*$ is the dual norm of $\|\cdot\|_c$. Let $\Bar{D}_0=\mu \ell_{cs}^2/(8\Bar{\sigma}^2L)$, $\Bar{D}_1=2(1-\tilde{\gamma}_c)$, $\Bar{D}_2=8 Lu_{cM}^2/(\mu\ell_{cs}^2)$, $\Bar{D}_3=2\Bar{\sigma}^2u_{cM,*}^2$, and $\Bar{D}_4=2c_d\Bar{\sigma}^2 L/(\mu\ell_{cs}^2)$.
Let $\lambda_k=\theta/\alpha_k$ with $\theta=\Bar{D}_1/(8\Bar{D}_3)$. The stepsizes are chosen according to the following condition. 

\begin{condition}\label{con:stepsize}
We use $\alpha_k=\alpha/(k+h)^z$, where $z\in (0,1]$, and $\alpha>0,h\geq 1$ are chosen such that $\alpha_0\leq \min(4\Bar{D}_0\Bar{D}_3/\Bar{D}_1,1/\Bar{D}_1,D/(4\Bar{D}_2))$.
In addition, when $z=1$, we choose $\alpha>2/\Bar{D}_1$, and when $z\in (0,1)$, we choose $h\geq (2z/(\Bar{D}_1\alpha))^{1/(1-z)}$.
\end{condition}

The following proposition establishes a recursive inequality for the log-MGF of the generalized Moreau envelope. 

\begin{proposition}\label{prop:additive_recursion}
It holds for all $k\geq 0$ that
\begin{align}\label{prop:additive_recursion-1}
    Z_{k+1}\leq \frac{\alpha_k}{\alpha_{k+1}}(1-\Bar{D}_1\alpha_k/2) Z_k+\frac{\Bar{D}_1\Bar{D}_4}{4\Bar{D}_3}\alpha_k,
\end{align}
where $Z_k=\log(\mathbb{E}[\exp(\lambda_k M(x_k-x^*))])$.

\end{proposition}
\begin{proof}[Proof of Proposition \ref{prop:additive_recursion}]
Since $M(\cdot)$ is $L/\mu$ -- smooth with respect to $\|\cdot\|_s$, we have by the update equation (\ref{eq:stochastic-approximation}) that
\begin{align}
    M(x_{k+1}-x^*) 
    \leq\;& M(x_k-x^*) + \alpha_k \nabla M(x_k-x^*)^\top ( \Bar{F}(x_k) - x_k ) \nonumber\\
     &+  \alpha_k  \nabla M(x_k\!-\!x^*)^\top ( F(x_k,Y_k) \!-\! \Bar{F}(x_k) )\!+\!\frac{L\alpha_k^2}{2\mu} \left\| F(x_k,Y_k) \!-\! x_k \right\|_s^2. \label{eq:firstOGterm1}
\end{align}
Next, we bound all the terms on the right-hand side of the previous inequality. First, it follows from \cite[Lemma A.1]{chen2021finite} that
\begin{align}\label{eq:secondOGterm1}
    \nabla M(x_k-x^*)^\top( \Bar{F}(x_k) - x_k ) \leq -2\left( 1- \tilde{\gamma}_c \right)  M(x_k-x^*).
\end{align}
Next, for the quadratic term on the right-hand side of Eq. (\ref{eq:firstOGterm1}), we have
\begin{align*}
    \left\| F(x_k,Y_k) - x_k \right\|_s^2
    \leq \left( \left\| F(x_k,Y_k) - \Bar{F} (x_k) \right\|_s + \left\| \Bar{F}(x_k) - \Bar{F}(x^*) \right\|_s + \left\| x_k - x^* \right\|_s \right)^2.
\end{align*}
To proceed, observe that
\begin{align*}
    &\|F(x_k,Y_k) - \Bar{F} (x_k)\|_s\leq \frac{1}{\ell_{cs}}\|F(x_k,Y_k) - \Bar{F} (x_k)\|_c,\\
    &\| \Bar{F}(x_k) - \Bar{F}(x^*) \|_s\leq \frac{1}{\ell_{cs}}\|\Bar{F}(x_k) - \Bar{F}(x^*)\|_c\leq \frac{\gamma_c}{\ell_{cs}}\|x_k-x^*\|_c\leq \frac{\gamma_c u_{cM}}{\ell_{cs}}\|x_k-x^*\|_M,\\
    &\|x_k - x^*\|_s\leq u_{sM}\|x_k - x^*\|_M\leq \frac{u_{cM}}{\ell_{cs}}\|x_k - x^*\|_M.
\end{align*}
Therefore, we have
\begin{align}
    \left\| F(x_k,Y_k) - x_k \right\|_s^2
    \leq \;&\left( \frac{1}{\ell_{cs}}\|F(x_k,Y_k) - \Bar{F} (x_k)\|_c + \frac{2u_{cM}}{\ell_{cs}}\|x_k-x^*\|_M\right)^2\nonumber\\
    \leq \;&\frac{2}{\ell_{cs}^2}\|F(x_k,Y_k) - \Bar{F} (x_k)\|_c^2 + \frac{16u_{cM}^2}{\ell_{cs}^2}M(x_k-x^*).\label{eq:secondOGterm}
\end{align}
where the last line follows from $(a+b)^2\leq 2(a^2+b^2)$ for all $a,b\in\mathbb{R}$ and $M(x)=\frac{1}{2}\|x\|_M^2$ (cf. Lemma \ref{prop:Moreau}). Using the upper bounds we obtained in Eqs. \eqref{eq:secondOGterm1} and \eqref{eq:secondOGterm} in Eq. \eqref{eq:firstOGterm1}, we obtain
\begin{align}\label{eq:initialMoreauInequality}
    M(x_{k+1}-x^*) \leq\;& \left( 1- 2\alpha_k (1-\tilde{\gamma}_c) + \frac{8 Lu_{cM}^2\alpha_k^2}{\mu\ell_{cs}^2} \right)M(x_k-x^*) \nonumber \\
    &+ \frac{L\alpha_k^2 }{\mu\ell_{cs}^2} \left\| F(x_k,Y_k) - \Bar{F} (x_k) \right\|_c^2+\alpha_k M(x_k-x^*)^\top (F(x_k,Y_k) - \Bar{F}(x_k)).
\end{align}
For any $\lambda >0$, Eq. \eqref{eq:initialMoreauInequality} implies
\begin{align}\label{eq:twoExpectations}
    \mathbb{E}[ \exp(\lambda M(x_{k+1}\!-\!x^*)) \mid \mathcal{F}_k] 
    \leq \;& \exp\left[\lambda\left( 1- 2\alpha_k (1-\tilde{\gamma}_c) + \frac{8 Lu_{cM}^2\alpha_k^2}{\mu\ell_{cs}^2} \right)M(x_k-x^*)\right]\nonumber\\
    &\times \mathbb{E}\bigg[\exp\left(\frac{\lambda L\alpha_k^2 }{\mu \ell_{cs}^2} \left\| F(x_k,Y_k) - \Bar{F} (x_k) \right\|_c^2\right)\nonumber\\
    &\times \exp\left(\alpha_k\lambda M(x_k-x^*)^\top (F(x_k,Y_k) - \Bar{F}(x_k))\right)\;\bigg|\; \mathcal{F}_k\bigg]\nonumber\\
    \leq \;& \exp\left[\lambda\left( 1- 2\alpha_k (1-\tilde{\gamma}_c) + \frac{8 Lu_{cM}^2\alpha_k^2}{\mu\ell_{cs}^2} \right)M(x_k-x^*)\right]\nonumber\\
    &\times\! \underbrace{\mathbb{E}\left[\exp\left(2\alpha_k\lambda M(x_k\!-\!x^*)^\top\! (F(x_k,Y_k) \!-\! \Bar{F}(x_k))\right)\middle| \mathcal{F}_k\right]^{1/2}}_{N_1}\nonumber\\
    &\times\! \underbrace{\mathbb{E}\left[\exp\left(\frac{2\lambda L\alpha_k^2 }{\mu \ell_{cs}^2} \left\| F(x_k,Y_k) \!-\! \Bar{F} (x_k) \right\|_c^2\right)\middle|\mathcal{F}_k\right]^{1/2}}_{N_2},
\end{align}
where the last line follows from the conditional Cauchy–Schwarz inequality. 
We next bound the two conditional expectations in Eq. (\ref{eq:twoExpectations}).

For the term $N_1$, using Assumption \ref{ass:sub-Gaussian} (or Eq. (\ref{eq:sub-Gaussian_1})), we have
\begin{align}
    N_1
    &\leq  \exp\left(\lambda^2\Bar{\sigma}^2\alpha_k^2 \left\| \nabla M(x_{k}-x^*) \right\|_{c,*}^2\right) \nonumber\\   &=\exp\left(\lambda^2\Bar{\sigma}^2\alpha_k^2 \left\| \nabla \|x_k-x^*\|_M \right\|_{c,*}^2 \| x_k - x^* \|_M^2\right) \label{eq:middle1}\\
    &=\exp\left(2\lambda^2\Bar{\sigma}^2\alpha_k^2 \left\| \nabla \|x_k-x^*\|_M \right\|_{c,*}^2 M (x_k - x^* )\right),\nonumber\\
    &\leq  \exp\left(2\lambda^2\Bar{\sigma}^2\alpha_k^2 u_{cM,*}^2 M (x_k - x^* )\right)\label{eq:boundFirstConditional}
\end{align}
where Eq. (\ref{eq:middle1}) follows from $M(x)=\frac{1}{2}\|x\|_M^2$ (cf. Lemma \ref{prop:Moreau}) and Eq. (\ref{eq:boundFirstConditional}) follows from \cite[Lemma 2.6]{shalev2012online}.

For the term $N_2$, we have by Assumption \ref{ass:sub-Gaussian} that
\begin{align}
    N_2 \leq \left(1-\frac{4\lambda \Bar{\sigma}^2 L\alpha_k^2 }{\mu \ell_{cs}^2}\right)^{-c_d/4}\leq \exp\left(\frac{2c_d\lambda \Bar{\sigma}^2 L\alpha_k^2 }{\mu \ell_{cs}^2}\right). \label{eq:boundSecondConditional},
\end{align}
where the second inequality follows from choosing $\lambda\leq \bar{D}_0/\alpha_k^2$ so that $4\lambda \Bar{\sigma}^2 L\alpha_k^2 /(\mu \ell_{cs}^2)\leq 1/2$ and the numerical inequality  $1/(1-x)\leq e^{2x}$ for any $x\in [0,1/2]$.

Using the upper bounds we obtained for the terms $N_1$ (cf. Eq. (\ref{eq:boundFirstConditional})) and $N_2$ (cf. Eq. (\ref{eq:boundSecondConditional})) in Eq. (\ref{eq:twoExpectations}) , we have
\begin{align}
    \mathbb{E}[ \exp(\lambda M(x_{k+1}\!-\!x^*)) \mid \mathcal{F}_k ] 
    \leq \;&\exp\left[\lambda\left( 1\!-\! 2\alpha_k (1\!-\!\tilde{\gamma}_c) + \frac{8 Lu_{cM}^2\alpha_k^2}{\mu\ell_{cs}^2} \right)M(x_k\!-\!x^*)\right]\nonumber\\
    &\times \exp\left(2\lambda^2\Bar{\sigma}^2\alpha_k^2 u_{cM,*}^2 M (x_k - x^* )\right)\exp\left(\frac{2c_d\lambda \Bar{\sigma}^2 L\alpha_k^2 }{\mu \ell_{cs}^2}\right)\nonumber\\
    \leq  \;&\exp\left[\lambda M(x_k\!-\!x^*)\left( 1\!-\! \alpha_k \Bar{D}_1 \!+\! \Bar{D}_2\alpha_k^2\!+\! \Bar{D}_3\alpha_k^2\lambda \right)\!+\!\Bar{D}_4\alpha_k^2\lambda\right],\label{eq:fixlambda}
\end{align}
where we used $\{\Bar{D}_i\}_{1\leq i\leq 4}$ to simplify the notation in the last line.

Since the previous inequality holds for all $\lambda\in (0,\bar{D}_0/\alpha_k^2]$, we next replace $\lambda$ with the time-varying $\lambda_{k+1}=\theta/\alpha_{k+1}$. Before going forward, we first verify that this choice of $\lambda_{k+1}$ satisfies $\lambda_{k+1}\in (0,\bar{D}_0/\alpha_k^2]$ for all $k\geq 0$. 

Since $a/b\geq (a+c)/(b+c) $ for all $a\geq b>0$ and $c\geq 0$, we have for any $z\in (0,1]$ and $h\geq 1$ that
\begin{align*}
    \frac{\alpha_k}{\alpha_{k+1}}=\left(\frac{k+h+1}{k+h}\right)^z\leq \left(\frac{h+1}{h}\right)^z\leq 1+\frac{1}{h}\leq 2.
\end{align*}
It follows that
\begin{align*}
    \lambda_{k+1}=\frac{\theta}{\alpha_{k+1}}=\frac{\Bar{D}_0}{\alpha_k^2}\frac{\theta \alpha_k^2}{\Bar{D}_0\alpha_{k+1}}\leq \frac{\Bar{D}_0}{\alpha_k^2}\frac{2\theta }{\Bar{D}_0}\alpha_k=\frac{\Bar{D}_0 }{\alpha_k^2}\frac{\Bar{D}_1 \alpha_k }{4\Bar{D}_0\Bar{D}_3}.
\end{align*}
Therefore, when $\alpha_0\leq 4\Bar{D}_0\Bar{D}_3/\Bar{D}_1$ (cf. Condition \ref{con:stepsize}), we have $\lambda_{k+1}\leq \Bar{D}_0/\alpha_k^2$. Applying $\lambda=\lambda_{k+1}$ in Eq. (\ref{eq:fixlambda}), we have for all $k\geq 0$ that
\begin{align}
    &\mathbb{E}\left[\exp(\lambda_{k+1} M(x_{k+1}-x^*))\mid \mathcal{F}_k\right]\nonumber\\
    \leq\;& \exp\left(\lambda_k M(x_k-x^*)\frac{\lambda_{k+1}}{\lambda_k}\left( 1- \Bar{D}_1\alpha_k + \Bar{D}_2\alpha_k^2+ \Bar{D}_3\lambda_{k+1} \alpha_k^2 \right)\right)\exp(\Bar{D}_4\lambda_{k+1} \alpha_k^2)\nonumber\\
    =\;&\exp\left(\lambda_kN_{1,k} M(x_k-x^*)\right)\exp(N_{2,k}),\label{eq:MGF_before}
\end{align}
where we denote $N_{1,k}=(\lambda_{k+1}/\lambda_k)(1- \Bar{D}_1\alpha_k + \Bar{D}_2\alpha_k^2+ \Bar{D}_3\lambda_{k+1} \alpha_k^2)$ and 
$N_{2,k}=\Bar{D}_4\lambda_{k+1} \alpha_k^2$ for simplicity of notation.
Using the explicit expression of $\lambda_{k+1}$, we have
\begin{align}\label{eq:N2k}
    N_{2,k}=\Bar{D}_4\lambda_{k+1}\alpha_k^2=\Bar{D}_4\theta\frac{\alpha_k^2}{\alpha_{k+1}}\leq 2\Bar{D}_4\theta \alpha_k=\frac{\Bar{D}_1\Bar{D}_4}{4\Bar{D}_3}\alpha_k.
\end{align}
When next bound $N_{1,k}$ in the following:
\begin{align}
    N_{1,k}=\;&\frac{\lambda_{k+1}}{\lambda_k}(1- \Bar{D}_1\alpha_k + \Bar{D}_2\alpha_k^2+ \Bar{D}_3\lambda_{k+1} \alpha_k^2)\nonumber\\
    =\;&\frac{\alpha_k}{\alpha_{k+1}}(1- \Bar{D}_1\alpha_k + \Bar{D}_2\alpha_k^2+ \Bar{D}_3\theta \alpha_k^2/\alpha_{k+1})\nonumber\\
    \leq \;&\frac{\alpha_k}{\alpha_{k+1}}\left( 1- \Bar{D}_1\alpha_k + \Bar{D}_2\alpha_k^2+ \Bar{D}_1 \alpha_k/4 \right)\label{eq:N1k-1}\\
    \leq \;&\frac{\alpha_k}{\alpha_{k+1}}\left( 1- \Bar{D}_1\alpha_k + \Bar{D}_1\alpha_k/4+ \Bar{D}_1 \alpha_k/4 \right)\label{eq:N1k-2}\\
    = \;&\frac{\alpha_k}{\alpha_{k+1}}\left( 1- \Bar{D}_1\alpha_k/2 \right)\label{eq:N1k}.
\end{align}
where Eq. (\ref{eq:N1k-1}) follows from $\alpha_k/\alpha_{k+1}\leq 2$ and $\theta=\Bar{D}_1/(8\Bar{D}_3)$ and Eq. (\ref{eq:N1k-2}) follows from $\alpha_k\leq \alpha_0\leq \Bar{D}_1/(4\Bar{D}_2)$ (cf. Condition \ref{con:stepsize}).
 Using the upper bounds we obtained for $N_{1,k}$ (cf. Eq. (\ref{eq:N1k})) and $N_{2,k}$ (cf. Eq. (\ref{eq:N2k})) in Eq. (\ref{eq:MGF_before}) , we have for all $k\geq 0$ that
\begin{align}\label{eq:prop:additive_recursion}
    &\mathbb{E}\left[\exp(\lambda_{k+1} M(x_{k+1}\!-\!x^*))\mid \mathcal{F}_k\right]\nonumber\\
    \leq\;& \exp\left(\frac{\alpha_k}{\alpha_{k+1}}\left( 1\!-\! \frac{\Bar{D}_1\alpha_k}{2} \right)\lambda_k M(x_k\!-\!x^*)\!+\!\frac{\Bar{D}_1\Bar{D}_4}{4\Bar{D}_3}\alpha_k\right).
\end{align}
Our next step is to first take total expectation and then logarithm on both sides of the previous inequality and then apply Jensen's inequality. To do so, we need to verify 
that $(\alpha_k/\alpha_{k+1})\left( 1- \Bar{D}_1\alpha_k/2 \right)$ belongs to the interval $(0,1)$.

The fact that $(\alpha_k/\alpha_{k+1})\left( 1- \Bar{D}_1\alpha_k/2 \right)>0$ follows from our choice of $\alpha_0\leq 1/\Bar{D}_1$ (cf. Condition \ref{con:stepsize}). To show $(\alpha_k/\alpha_{k+1})\left( 1- \Bar{D}_1\alpha_k/2 \right)<1$, observe that
\begin{align}
    \frac{\alpha_k}{\alpha_{k+1}}\left( 1- \Bar{D}_1\alpha_k/2 \right)
    =\;&\left(\frac{k+h+1}{k+h}\right)^z\left(1-\frac{\Bar{D}_1\alpha}{2(k+h)^z}\right)\nonumber\\
    \leq \;&\left(\frac{k+h+1}{k+h}\right)^z\exp\left(-\frac{\Bar{D}_1\alpha}{2(k+h)^z}\right)\nonumber\\
    = \;&\left[\left(1+\frac{1}{k+h}\right)^{k+h}\right]^{z/(k+h)}\exp\left(-\frac{\Bar{D}_1\alpha}{2(k+h)^z}\right)\nonumber\\
    \leq \;&\exp\left(\frac{z}{k+h}-\frac{\Bar{D}_1\alpha}{2(k+h)^z}\right).\nonumber
\end{align}
When $z=1$, since $\alpha>2/\Bar{D}_1$ (cf. Condition \ref{con:stepsize}), we have 
\begin{align*}
    \frac{\alpha_k}{\alpha_{k+1}}\left( 1- \Bar{D}_1\alpha_k/2 \right)\leq \exp\left(\frac{z}{k+h}-\frac{\Bar{D}_1\alpha}{2(k+h)^z}\right)\leq 1.
\end{align*}
When $z\in (0,1)$, since $h\geq (4z/(\Bar{D}_1\alpha))^{1/(1-z)}$ (cf. Condition \ref{con:stepsize}), we also have 
\begin{align*}
    \frac{\alpha_k}{\alpha_{k+1}}\left( 1- \Bar{D}_1\alpha_k/2 \right)\leq \;&\exp\left(\frac{z}{k+h}-\frac{\Bar{D}_1\alpha}{2(k+h)^z}\right)\\
    = \;&\exp\left(\frac{2z-(k+h)^{1-z}\Bar{D}_1\alpha}{2(k+h)}\right)\\
    \leq  \;&\exp\left(\frac{2z-h^{1-z}\Bar{D}_1\alpha}{2(k+h)}\right)\\
    \leq \;&1.
\end{align*}
Now that we have verified $(\alpha_k/\alpha_{k+1})\left( 1- \Bar{D}_1\alpha_k/2 \right)\in (0,1)$, after first taking total expectation and then logarithm on both sides of Eq. (\ref{eq:prop:additive_recursion}) and then apply Jensen's inequality, we have for all $k\geq 0$ that
\begin{align*}
    Z_{k+1}
    \leq \frac{\alpha_k}{\alpha_{k+1}}(1-\Bar{D}_1\alpha_k/2) Z_k+\frac{\Bar{D}_1\Bar{D}_4}{4\Bar{D}_3}\alpha_k.
\end{align*}
This proves Proposition \ref{prop:additive_recursion}. 
\end{proof}

Repeatedly using Eq. (\ref{prop:additive_recursion-1}) of  Proposition \ref{prop:additive_recursion} yields the following lemma. The proof of Lemma \ref{le:bound_MGF_additive} involves only standard algebra manipulation and is deferred to Appendix \ref{pf:le:additive_solving_recursion}.

\begin{lemma}\label{le:bound_MGF_additive}
    It holds for all $k\geq 0$ that
\begin{align}\label{prop:additive_recursion-2}
    Z_k\leq \begin{dcases}
        Z_0\left(\frac{h}{k+h}\right)^{\Bar{D}_1\alpha/2-1}+\frac{e\Bar{D}_1\Bar{D}_4\alpha}{\Bar{D}_3(\Bar{D}_1\alpha/2-1)},&z=1,\\
        Z_0\left(\frac{k+h}{h}\right)^z\exp\left(-\frac{\Bar{D}_1\alpha}{2(1-z)}((k+h)^{1-z}-h^{1-z})\right)+\frac{\Bar{D}_4}{\Bar{D}_3},&z\in (0,1).
    \end{dcases}
\end{align}
\end{lemma}

\subsection{The Maximal Concentration Inequality}

To establish the maximal concentration inequalities using Proposition \ref{prop:additive_recursion} and Lemma \ref{le:bound_MGF_additive}, we have $2$ approaches. \textit{Approach 1} is based on the construction of a supermartingale and the use of Ville's maximal inequality (which we also used in the proof of Theorem \ref{thm:multi}), and \textit{Approach 2} is based on using Markov inequality together with a telescoping technique. In the case of $z<1$, \textit{Approach 2} gives a better bound. In the case of $z=1$, \textit{Approach 2} gives a better decay rate for the term that involves the initial condition but a worse initial radius on the decaying cone. Theorem \ref{thm:additive} is stated based on  \textit{Approach 2}, the proof of which is presented in the following. The details of \textit{Approach 1} are deferred to Appendix \ref{ap:pf:S+V}. 

For any $\epsilon>0$ and $k\geq 0$, we have
\begin{align*}
    \mathbb{P}\left(\lambda_kM(x_k-x^*)> \epsilon\right)
   = \;&\mathbb{P}\left(\exp\left(\lambda_k M(x_k-x^*)\right)> e^\epsilon\right)\\
   \leq \;&\mathbb{E}\left[\exp\left(\lambda_k M(x_k-x^*)-\epsilon\right)\right]\\
   =\;&\exp\left(Z_k-\epsilon\right).
\end{align*}
Let $\delta=\exp(Z_k-\epsilon)$. Then the previous inequality reads: with probability at least $1-\delta$, we have $M(x_k-x^*)\leq (Z_k+\log(1/\delta))/\lambda_k$. To proceed, given $K>0$, define $\delta_k=K \delta/(k(k+1))$. By union bound, we have for any $\delta>0$ that
\begin{align*}
    \mathbb{P}\left(M(x_k-x^*)\leq \frac{Z_k+\log(1/\delta_k)}{\lambda_k},\;\forall\;k\geq K\right)\leq\;& 1-\sum_{k=K}^\infty\delta_k\\
    =\;& 1-\sum_{k=K}^\infty\frac{K \delta}{k(k+1)}\\
    =\;&1-K\delta\sum_{k=K}^\infty\left(\frac{1}{k}-\frac{1}{k+1}\right)\\
    =\;&1-\delta.
\end{align*}
By first using Lemma \ref{prop:Moreau} to translate $M(\cdot)$ back into $\|\cdot\|_c^2$ and then using the upper bound of $Z_k$ derived in Lemma \ref{le:bound_MGF_additive}, we have the following results.
\begin{enumerate}[(1)]
    \item When $z=1$ and $\alpha>2/\Bar{D}_1$, for any $K\geq 0$, we have with probability at least $1-\delta$ that the following inequality holds for all $k\geq K$:
    \begin{align*}
        \|x_k-x^*\|_c^2\leq\;&\frac{16 \bar{D}_3u_{cM}^2\alpha}{\Bar{D}_1(k+h)}\log(1/\delta)+\frac{u_{cM}^2}{\ell_{cM}^2}\|x_0-x^*\|_c^2\left(\frac{h}{k+h}\right)^{\Bar{D}_1\alpha/2}\\
        &+\frac{16 eu_{cM}^2\Bar{D}_4\alpha^2}{(\Bar{D}_1\alpha/2-1)}\frac{1}{k+h}+\frac{32u_{cM}^2\Bar{D}_3 \alpha }{\Bar{D}_1}\frac{\log((k+1)/K^{1/2})}{k+h}.
    \end{align*}
    \item When $z\in (0,1)$, $\alpha>0$, and $h\geq (\frac{4z}{\Bar{D}_1\alpha})^{1/(1-z)}$, for any $\delta>0$ and $K\geq 0$, the following inequality holds for all $k\geq K$:
    \begin{align*}
        \|x_k-x^*\|_c^2\leq\;&\frac{16 \bar{D}_3u_{cM}^2\alpha}{\Bar{D}_1(k+h)^z}\log(1/\delta)+\frac{16 eu_{cM}^2\Bar{D}_4\alpha}{\bar{D}_1}\frac{1}{(k+h)^z}\\
        &+\frac{u_{cM}^2}{\ell_{cM}^2}\|x_0-x^*\|_c^2\exp\left(-\frac{\Bar{D}_1\alpha}{2(1-z)}((k+h)^{1-z}-h^{1-z})\right)\\
        &+\frac{32u_{cM}^2\Bar{D}_3 \alpha }{\Bar{D}_1}\frac{\log((k+1)/K^{1/2})}{(k+h)^z}.
    \end{align*}
\end{enumerate}

Theorem \ref{thm:additive} follows by using our notation $\bar{c}_1=16 \bar{D}_3u_{cM}^2\alpha/\Bar{D}_1$, $\bar{c}_2=u_{cM}^2/\ell_{cM}^2$, $\bar{c}_3=16 eu_{cM}^2\Bar{D}_4\alpha^2/(\Bar{D}_1\alpha/2-1)$, $\bar{c}_4=32u_{cM}^2\Bar{D}_3 \alpha /\Bar{D}_1$, and $\bar{c}_5=16 eu_{cM}^2\Bar{D}_4\alpha/\bar{D}_1$
to simplify the above expressions.

\section{Applications}\label{sec:applications}
In this section, we discuss the applicability of our theoretical results in the context of linear SA and RL.

\subsection{Linear Stochastic Approximation}\label{subsec:linear_sa}

Consider the problem of solving the linear system of equations $\bar{A}x=\bar{b}$, where $\bar{A}\in\mathbb{R}^{d\times d}$ and $\bar{b}\in\mathbb{R}^d$. This type of problem arises in many realistic applications \cite{benveniste2012adaptive},  typical examples of which involve (1) solving least-square problems, (2) the temporal-difference (TD) learning method for solving the policy evaluation problem in RL (which will be discussed in the next subsection), and (3) linear dynamical systems in control theory, etc.

Suppose that we do not have complete information about the parameters $\bar{A}$ and $\bar{b}$ of the problem but can obtain noisy estimates of them. Then, in this case, the problem is usually solved with SA, which is presented below:
\begin{align}\label{algo:linear_sa}
    x_{k+1}=x_k+\alpha_k(A(Y_k)x_k-b(Y_k)),
\end{align}
where $\{Y_k\}$ (taking values in $\mathcal{Y}$) is a sequence of i.i.d. random variables, and $A:\mathcal{Y}\mapsto \mathbb{R}^{d\times  d}$ and $b:\mathcal{Y}\mapsto \mathbb{R}^d$ are deterministic functions. We impose the following standard assumptions to study $\{x_k\}$ generated by Eq. (\ref{algo:linear_sa}).

\begin{assumption}\label{as:linear_sa_iid}
    It holds that $\mathbb{E}[A(Y_k)]=\bar{A}$ and $\mathbb{E}[b(Y_k)]=\bar{b}$ for all $k\geq 0$. 
\end{assumption}

It is enough to assume that $\{Y_k\}$ has conditionally unbiased perturbations as in Section \ref{subsec:multiplicative}. Here, we make the i.i.d. assumption for ease of exposition.

\begin{assumption}\label{as:linear_sa_boundedness}
 It holds that $\sup_{y\in\mathcal{Y}}\|A(y)\|_2<\infty$ and $\sup_{y\in\mathcal{Y}}\|b(y)\|_2< \infty$.
\end{assumption}

Denote $A_{\max}=\sup_{y\in\mathcal{Y}}\|A(y)\|_2$ and $b_{\max}=\sup_{y\in\mathcal{Y}}\|b(y)\|_2$.
Assumption \ref{as:linear_sa_boundedness} is widely used in studying the asymptotic convergence \cite{bertsekas1996neuro} and finite-sample convergence \cite{srikant2019finite} for linear SA. Note that it is automatically satisfied when $\mathcal{Y}$ is a finite set.

\begin{assumption}\label{as:Hurwitz}
    The matrix $\Bar{A}$ is Hurwitz, i.e., all eigenvalues of $\Bar{A}$ have strictly negative real parts.
\end{assumption}

Assumption \ref{as:Hurwitz} is usually imposed to ensure the stability of the linear SA algorithm presented in Eq. (\ref{algo:linear_sa}) \cite{srikant2019finite}. In fact, consider the ordinary differential equation (ODE) associated with the SA algorithm \cite{borkar2009stochastic}:
\begin{align}\label{ode:linear_sa}
    \dot{x}(t)=\Bar{A}x(t)-\Bar{b}.
\end{align}
When $\Bar{A}$ is Hurwitz, the Lyapunov equation $\Bar{A}^\top P+P\Bar{A}+I_d=0$
has a unique positive definite solution \cite{khalil2002nonlinear}, denoted by $\Bar{P}$. It then follows from the Lyapunov theorem \cite{khalil2002nonlinear} that the unique equilibrium point $x^*=\Bar{A}^{-1}\Bar{b}$ of ODE (\ref{ode:linear_sa}) is exponentially stable \cite{haddad2011nonlinear}, which in turn guarantees the asymptotic convergence of the SA algorithm by the ODE method \cite{borkar2009stochastic}.

Next, we show that the linear SA presented in Eq.\eqref{algo:linear_sa} can be equivalently modeled as a contractive SA in the form of Eq. (\ref{eq:stochastic-approximation}) with multiplicative noise. As a result, Theorem \ref{thm:multi} allows us to establish maximal concentration bounds (with Weibull tails) for linear SA. Let $\beta= \lambda^{-1}_{\max}(\Bar{A}^\top \Bar{P}\Bar{A})/2$. In this work, $\lambda_{\max}(\cdot)$ (respectively, $\lambda_{\min}(\cdot)$) returns the largest (respectively, smallest) eigenvalue of a symmetric matrix. Let $F_\beta:\mathcal{Y}\times \mathbb{R}^d\mapsto\mathbb{R}^d$ be an operator defined as
\begin{align*}
    F_\beta(x,y)=\beta A(y)x-\beta b(y)+x,\quad \forall\;x\in\mathbb{R}^d,y\in\mathcal{Y}.
\end{align*}
Then, Eq. (\ref{algo:linear_sa}) can be equivalently written as
\begin{align*}
    x_{k+1}=x_k+\frac{\alpha_k}{\beta}(F_\beta(x_k,Y_k)-x_k),
\end{align*}
which is in the same form of Algorithm (\ref{eq:stochastic-approximation}) because we can absorb the constant $\beta$ into the stepsize. We next show that Assumptions \ref{as:contraction} -- \ref{as:multi} are satisfied in the context of linear SA. Let $\|\cdot\|_{\Bar{P}}$ be a norm defined as $\|x\|_{\Bar{P}}=(x^\top \Bar{P} x)^{1/2}$ for all $x\in\mathbb{R}^d$, where we recall that $\bar{P}$ is the unique positive definite solution of the Lyapunov equation $\Bar{A}^\top P+P\Bar{A}+I_d=0$. The proof of the following lemma can be found in Appendix \ref{pf:prop:linear_sa}.

\begin{lemma}\label{prop:linear_sa}
    Suppose that Assumptions \ref{as:linear_sa_boundedness} and \ref{as:Hurwitz} are satisfied, and $\beta= \lambda^{-1}_{\max}(\Bar{A}^\top \Bar{P}\Bar{A})/2$. Then, we have the following results.
    \begin{enumerate}[(1)]
        \item There exists $\Bar{\gamma}\in (0,1)$ such that the operator $\Bar{F}_\beta(\cdot)=\mathbb{E}[F_\beta(\cdot,Y_0)]$ is a $\Bar{\gamma}$ -- contraction mapping with respect to $\|\cdot\|_{\Bar{P}}$.
        \item It holds for all $k\geq 0$ that $\mathbb{E}[\Bar{F}_\beta(x_k,Y_k)\mid \mathcal{F}_k]=\Bar{F}_\beta(x_k)$, where $\mathcal{F}_k$ is the $\sigma$-algebra generated by $\{x_0,Y_0,Y_1,\cdots,Y_{k-1}\}$.
        \item There exists $\hat{\sigma}>0$ such that $\|F_\beta(x_k,Y_k)-\Bar{F}_\beta(x_k)\|_{\Bar{P}}\leq \hat{\sigma}(\|x\|_{\Bar{P}}+1)$ for all $k\geq 0$.
    \end{enumerate}
\end{lemma}

Lemma \ref{prop:linear_sa} enables us to apply Theorem \ref{thm:multi} to establish the maximal concentration bound of the linear SA algorithm presented in Eq. (\ref{algo:linear_sa}). The result is presented in the following.

\begin{theorem}\label{thm:linear_sa}
Suppose that Assumptions \ref{as:linear_sa_boundedness} and \ref{as:Hurwitz} are satisfied, and $\alpha_k=\alpha\beta/(k+h)$ with appropriately chosen $\alpha$ and $h$. Then, the same bound in Theorem \ref{thm:multi} holds here. As a result, there exists an integer $m_\ell>0$ such that for any $\epsilon>0$ and $\delta\in (0,1)$, to achieve $\|x_k-x^*\|_{\Bar{P}}\leq \epsilon$ with probability at least $1-\delta$, the iteration complexity is $\tilde{\mathcal{O}}((1+\log^{m_\ell}(1/\delta))\epsilon^{-2})$.
\end{theorem}

Due to the wide applications of linear SA, there are many existing results analyzing its concentration behavior \cite{durmus2021stability,durmus2021tight,durmusAveraged}. However, concentration bounds (with tails decaying faster than polynomials) were established only in the case where the matrix $A(Y_k)$ is a.s. Hurwitz, otherwise only polynomial tail bounds were obtained \cite{durmus2021tight}. Theorem \ref{thm:linear_sa} states that the iterates generated by a linear SA (where the random matrices $\{A(Y_k)\}$ are Hurwitz in expectation but not a.s. Hurwitz) enjoy an $\tilde{\mathcal{O}}(1/k)$ rate of convergence with a Weibull tail, which offers a significant improvement over the results in the literature.

\subsection{Reinforcement Learning} In this section, we discuss the applicability of our results in the context of RL.
In recent years, RL has been deployed to solve many practical problems involving sequential decision-making. An RL problem is usually modeled as a Markov decision process (MDP) \cite{sutton2018reinforcement}. However, the environmental model, including the transition dynamics and the reward function, is unknown to the agent. Therefore, an RL agent has to learn to make decisions by actively interacting with the environment to collect information. 

Mathematically, consider an infinite horizon discounted MDP defined by a finite state space $\mathcal{S}$, a finite action space $\mathcal{A}$, a set of transition probability matrices $\{P_a\in\mathbb{R}^{|\mathcal{S}|\times|\mathcal{S}|}\mid a\in\mathcal{A}\}$, a reward function $\mathcal{R}:\mathcal{S}\times\mathcal{A}\mapsto[0,1]$\footnote{Since we work with a finite MDP, assuming the reward function falls into the interval $[0,1]$ is without loss of generality}, and a discount factor $\gamma\in (0,1)$. Note that in RL, the transition probabilities and the reward function are unknown to the agent. Given a stationary policy $\pi:\mathcal{S}\mapsto\Delta^{|\mathcal{A}|}$, where $\Delta^{|\mathcal{A}|}$ stands for the $|\mathcal{A}|$ -- dimensional probability simplex, its value function $V^\pi:\mathcal{S}\mapsto\mathbb{R}$ and $Q$-function $Q^\pi:\mathcal{S}\times\mathcal{A}\mapsto\mathbb{R}$ are defined as $V^\pi(s)=\mathbb{E}_\pi\left[\sum_{k=0}^\infty\gamma^k\mathcal{R}(S_k,A_k)\;\middle|\;S_0=s\right]$ for all $s$ and $Q^\pi(s,a)=\mathbb{E}_\pi\left[\sum_{k=0}^\infty\gamma^k\mathcal{R}(S_k,A_k)\;\middle|\; S_0=s,A_0=a\right]$ for all $(s,a)$,
where we use the notation $\mathbb{E}_\pi[\,\cdot\,]$ to mean that the actions are selected based on the policy $\pi$. 
Since we work with a finite MDP, the value function can be alternatively viewed as a vector living in the $|\mathcal{S}|$-dimensional Euclidean space, where $|\mathcal{S}|$ stands for the cardinality of the state space $\mathcal{S}$, similarly for the $Q$-functions.

In RL, the prediction problem (also known as the policy evaluation problem) refers to the problem of estimating the value function $V^\pi$ (or the $Q$-function $Q^\pi$) for a given policy $\pi$, and the control problem refers to finding an optimal policy $\pi^*$ so that its value function $V^*$ (or equivalently, its $Q$-function $Q^\pi$) is uniformly maximized, i.e., $V^*(s)\geq V^\pi(s)$ for all $s\in\mathcal{S}$ and policy $\pi$ (or equivalently, $Q^*(s,a)\geq Q^\pi(s,a)$ for all $(s,a)$ and policy $\pi$). From now on, we will study the prediction problem and the control problem in terms of the $Q$-functions. The same results for the value function can be established following a similar approach.

\subsubsection{TD-Learning for Prediction}
We first consider the prediction problem, which is usually solved with TD-learning \cite{sutton1988learning} and its variants. To present the algorithm, we first state an important result in MDP theory -- the Bellman equation, which states that $Q^\pi\in\mathbb{R}^{|\mathcal{S}||\mathcal{A}|}$ is the unique solution to the fixed-point equation $Q=\mathcal{H}^\pi(Q)$, where $\mathcal{H}^\pi:\mathbb{R}^{|\mathcal{S}||\mathcal{A}|}\mapsto\mathbb{R}^{|\mathcal{S}||\mathcal{A}|}$ is the Bellman operator defined as
\begin{align}\label{def:Bellman_operator_prediction}
    [\mathcal{H}^\pi(Q)](s,a)=\mathbb{E}_\pi\left[\mathcal{R}(S_0,A_0)+\gamma Q(S_1,A_1)\mid S_0=s,A_0=a\right]
\end{align}
for all $(s,a)\in\mathcal{S}\times \mathcal{A}$ and $Q\in\mathbb{R}^{|\mathcal{S}||\mathcal{A}|}$.
The fact that $Q^\pi=\mathcal{H}^\pi(Q^\pi)$ can be derived using the definition of the $Q$-function and the Markov property, and the uniqueness of the solution follows from $\mathcal{H}^\pi(\cdot)$ being a contraction mapping \cite{bertsekas1996neuro}. Moreover, it can be easily verified that $Q^\pi$ is the unique solution to the multi-step Bellman equation 
\begin{align}\label{def:n-Bellman_operator_prediction}
Q=\underbrace{\mathcal{H}^\pi\cdot\mathcal{H}^\pi\cdots \mathcal{H}^\pi}_{n\text{ times}}(Q):=\mathcal{H}^\pi_n(Q)
\end{align}
for any $n\geq 1$, which follows from $\mathcal{H}_n^\pi(\cdot)$ being a contraction mapping as long as $\mathcal{H}^\pi(\cdot)$ is a contraction mapping.

In view of the Bellman equation, a natural approach to solve the prediction problem is to perform the fixed-point iteration to solve $Q=\mathcal{H}^\pi(Q)$, the geometric convergence of which is guaranteed by the Banach fixed-point theorem \cite{banach1922operations}. However, carrying out such a fixed-point iteration would require complete knowledge of the model parameters to compute the conditional expectation in Eq. (\ref{def:Bellman_operator_prediction}). To overcome this challenge, TD-learning was proposed as a data-driven SA algorithm to solve the (multi-step) Bellman equation.

Let $\pi_b$ be the policy that the agent uses to interact with the environment to collect samples, commonly referred to as the behavior policy. Note that $\pi_b$ does not necessarily coincide with the target policy $\pi$, whose $Q$-function we aim to estimate. Suppose that $\pi_b(a \mid s) > 0$ for all $(s, a)$ and the induced Markov chain on states has a unique stationary distribution, denoted by $\kappa_b\in\Delta(\mathcal{S})$, which satisfies $\kappa_b(s)>0$ for all $s\in\mathcal{S}$. We impose the following assumption regarding our sample trajectory.
\begin{assumption}\label{as:TD}
The sample trajectory $\{(S_k^0,A_k^0,S_k^1,A_k^1,\cdots,S_k^n,A_k^n)\}_{k\geq 0}$ (where sample is an $n$-tuple) is an i.i.d. sequence such that $S_k^0\sim\kappa_b(\cdot)$, $A_k^i\sim \pi_b(\cdot\mid S_k^i)$, and $S_k^{i+1}\sim P_{A_k^i}(S_k^i,\cdot)$ for all $i\in \{0,1,\cdots,n-1\}$ and $k\geq 0$. 
\end{assumption}

The i.i.d. sampling has been widely employed in the existing literature for analytical tractability \cite{bhandari2018finite,doan2019finite,zeng2024fast,dalal2018finite}. Practically, suppose that the Markov chain $\{S_k\}_{k\geq 0}$ induced by $\pi_b$ is irreducible and aperiodic \cite{zou2019finite,khodadadian2021finite,chenziyi2022sample,chenzy2021sample,xu2021sample,wu2020finite,qiu2019finite} (which implies uniform ergodicity). The i.i.d. sampling can be approximately achieved by first letting the Markov chain $\{S_k\}$ evolve for a short time period (so that the distribution of $\{S_k\}$ is close to $\kappa_b$ due to the geometric mixing) and then collect one sample every once a while (so that the samples are nearly independent). Compared with the more natural Markovian sampling (where the agent uses each sample from a single trajectory of state-action pairs generated by using the behavior policy $\pi_b$ to interact with the environment), the i.i.d. sampling assumption does have limitations. An immediate future direction of this work is to investigate whether maximal concentration bounds with tails decaying faster than polynomials can be achieved for SA with biased updates, which would enable us to study RL with Markovian sampling.

Next, we will present a generic learning-based algorithm for solving the policy evaluation problem, which covers many existing variants of TD-learning as its special cases. Before that, we need to first introduce RL with function approximation, which is used to overcome the curse of dimensionality. 

The key idea of function approximation is to approximate the desired high-dimensional target ($Q^\pi$ in the case of policy evaluation) from a pre-specified function class. In this work, for analytical tractability, we consider linear function approximation. Let $\{\phi_i\}_{1\leq i\leq d}$ be a set of linearly independent basis vectors, where for each $i\in \{1,2,\cdots,d\}$, $\phi_i$ is an $|\mathcal{S}||\mathcal{A}|$-dimensional vector. Let $\phi(s,a)=(\phi_1(s,a),\phi_2(s,a),\cdots,\phi_d(s,a))\in\mathbb{R}^d$, which is usually called the feature associated with the state-action pair $(s,a)$. We assume without loss of generality that the basis vectors are normalized so that $\|\phi(s,a)\|_2\leq 1$ for all $(s,a)$. With the basis vectors introduced, in linear function approximation, the goal is to find a linear combination of the basis vectors $\sum_{i=1}^d\phi_iw_i$ (where $w\in\mathbb{R}^d$) to best approximate $Q^\pi$. As a side note, if $\{\phi_i\}_{1\leq i\leq d}$ are chosen as the canonical basis, then it reduces to the tabular setting.

Now, we are ready to present the generic algorithm. With the i.i.d. sample trajectory $\{(S_k^0,A_k^0,S_k^1,A_k^1,\cdots,S_k^n,A_k^n)\}_{k\geq 0}$ at hand, the agent iteratively updates the weight $w_k$ according to the following formula:
\begin{align}\label{algo:TD}
	w_{k+1}=\;&
	w_k+\alpha_k\phi(S_k,A_k)\sum_{i=k}^{k+n-1}\gamma^{i-k}\prod_{j=k+1}^ic(S_k^j,A_k^j)\nonumber\\
	&\times \left(\mathcal{R}(S_k^i,A_k^i)+\gamma \rho(S_k^j,A_k^j)\phi(S_k^{i+1},A_k^{i+1})^\top w_k-\phi(S_k^i,A_k^i)^\top w_k\right),
\end{align}
where $c,\rho:\mathcal{S}\times \mathcal{A}\mapsto\mathbb{R}$ are called generalized importance sampling factors. To further illustrate the algorithm and provide intuition, consider the following special cases.

\textbf{Case 1.} Suppose that we use the canonical basis vectors and $\pi_b=\pi$ (i.e., the behavior policy and the target policy coincides). Then, by choosing $c(s,a)=\rho(s,a)=1$ for all $(s,a)$, Eq. \eqref{algo:TD} reduces to the update equation for the standard on-policy $n$-step TD-learning, which is an SA algorithm for solving the $n$-step Bellman equation (\ref{def:n-Bellman_operator_prediction}) \cite{sutton2018reinforcement}.

\textbf{Case 2:} Suppose that we use the canonical basis vectors and $\pi_b\neq \pi$, i.e., we are in the off-policy setting. Then, by choosing $c(s,a)=\rho(s,a)=\pi(a\mid s)/\pi_b(a\mid s)$ for all $(s,a)$, Eq. \eqref{algo:TD} reduces to the update equation for the standard off-policy $n$-step TD-learning. Although choosing $c(s,a)=\rho(s,a)=\pi(a\mid s)/\pi_b(a\mid s)$ leads to an unbiased estimator for $Q^\pi$, the algorithm is well-known to suffer from a large variance (due to importance sampling) \cite{glynn1989importance}. Various generalized importance sampling factors are proposed to address this issue, leading to algorithms such as $Q^\pi(\lambda)$ \citep{harutyunyan2016q}, Tree-Backup$(\lambda)$ \citep{precup2000eligibility}, Retrace$(\lambda)$ \citep{munos2016safe}, and $Q$-trace \citep{chen2021finite}, all of which are in the form of Eq. (\ref{algo:TD}).

\textbf{Case 3:} Suppose that we are indeed in the function approximation setting where the span of $\{\phi_i\}_{1\leq i\leq d}$ forms a $d$-dimensional linear subspace of $\mathbb{R}^{|\mathcal{S}||\mathcal{A}|}$. Then, Eq. (\ref{algo:TD}) covers the popular $n$-step TD-learning with linear function approximation in the on-policy and off-policy setting as its special cases \cite{tsitsiklis1997analysis,chen2022finite,chen2022sample}.

To establish the concentration bound of the generic TD-learning algorithm presented in Eq. (\ref{algo:TD}), we show that Eq. (\ref{algo:TD}) can be equivalently formulated as a linear SA. For any $y=(s^0,a^0,s^1,a^1,\cdots,s^n,a^n)\in \mathcal{Y}:=\mathcal{S}^n\times \mathcal{A}^n$, let $A:\mathcal{Y}\mapsto \mathbb{R}^{d\times d}$ be defined as 
\begin{align*}
    A(y)=\phi(s^0,a^0)\sum_{i=0}^{n-1}\gamma^{i}\prod_{j=1}^ic(s^j,a^j)\left(\gamma \rho(s^j,a^j)\phi(s^{i+1},a^{i+1})-\phi(s^i,a^i) \right)^\top,
\end{align*}
and let $b:\mathcal{Y}\mapsto \mathbb{R}^d$ be defined as
\begin{align*}
    b(y)=-\phi(s^0,a^0)\sum_{i=0}^{n-1}\gamma^{i}\prod_{j=1}^ic(s^j,a^j)\mathcal{R}(s^i,a^i).
\end{align*}
Then, the update equation (\ref{algo:TD}) can be equivalently written as
\begin{align*}
    w_{k+1}=w_k+\alpha_k(A(Y_k)w_k-b(Y_k)),
\end{align*}
where $\{Y_k=(S_k^0,A_k^0,\cdots,S_k^n,A_k^n)\}_{k\geq 0}$ is a sequence of i.i.d. random variables. Next, we show that all assumptions needed to apply Theorem \ref{thm:linear_sa} are satisfied in the context of TD-learning. Since Assumptions \ref{as:linear_sa_iid} and \ref{as:linear_sa_boundedness} are automatically satisfied due to the finiteness of the set $\mathcal{Y}$ and the i.i.d. sampling, we only need to verify Assumption \ref{as:Hurwitz}. The proof of the following lemma is presented in Appendix \ref{pf:le:TD}.

\begin{lemma}\label{le:TD}
    The following results hold.
    \begin{enumerate}[(1)]
        \item In the on-policy setting (i.e., $\pi_b=\pi$), the matrix $\bar{A}$ is always Hurwitz.
    \item In the off-policy setting (i.e., $\pi_b\neq \pi$), with appropriately chosen generalized importance sampling ratios  $c(\cdot,\cdot),\rho(\cdot,\cdot)$ and the parameter $n$\footnote{The explicit requirement is presented Appendix \ref{pf:le:TD}.}, the matrix $\bar{A}$ is Hurwitz.
    \end{enumerate}
\end{lemma}

Now, applying Theorem \ref{thm:linear_sa} to the generic TD-learning algorithm presented in Eq. (\ref{algo:TD}), we obtain a maximal concentration bound that has an $\tilde{\mathcal{O}}(1/k)$ rate of convergence and a Weibull tail. 

Due to the popularity of RL, there are many results in the literature studying the mean-square bounds and concentration bounds of on-policy and off-policy TD-learning under either i.i.d. sampling or Markovian sampling \cite{bhandari2018finite,srikant2019finite,dalal2018finite,chandak2023concentration,chandak2022concentration,chen2022sample,chen2021GB}. The closest work to ours are \cite{dalal2018finite,chandak2023concentration,chandak2022concentration}. Specifically, in \cite{dalal2018finite} they obtain maximal concentration bounds for TD$(0)$ with linear function approximation, which has linear multiplicative noise. However, their bounds are only valid for large enough iterates (with the threshold depending on the probability $\delta$), and their convergence rate can be worse than $\tilde{\mathcal{O}}(1/k)$. The authors in \cite{chandak2023concentration,chandak2022concentration} provide maximal concentration bounds for on-policy TD-learning with linear function approximation and Markovian sampling. However, they require $\|F(x_k,Y_k)\|_2\leq c_1\|x\|_2+c_2$ with $c_1\in (0,1)$ and their bound starts to hold only when $K$ is large enough (with the threshold being random). In the off-policy setting, to the best of our knowledge, there are no results on high-probability bounds in the literature.

\subsubsection{$Q$-Learning for Control}\label{subsubsec:Q-learning}

Moving to the control problem, the goal is to find an optimal policy $\pi^*$ such that its associated $Q$-function $Q^*$ is such that $Q^*(s,a)\geq Q^\pi(s,a)$ for all $(s,a)$ and $\pi$. One of the most popular algorithms for solving the control problem is $Q$-learning, which is our focus here.

Let $\mathcal{H}:\mathbb{R}^{|\mathcal{S}||\mathcal{A}|}\mapsto\mathbb{R}^{|\mathcal{S}||\mathcal{A}|}$ be the Bellman optimality operator defined as 
\begin{align*}
    [\mathcal{H}(Q)](s,a)=\mathcal{R}(s,a)+\gamma\mathbb{E}\left[\max_{a'\in\mathcal{A}}Q(S_{k+1},a')\;\middle|\; S_k=s,A_k=a\right]
\end{align*}
for all $Q\in\mathbb{R}^{|\mathcal{S}||\mathcal{A}|}$ and $(s,a)$ \cite{bertsekas1996neuro}. It is well known that $Q^*$ is the unique solution to the fixed-point equation $Q^*=\mathcal{H}(Q^*)$. In addition, once $Q^*$ is obtained, an optimal policy can be obtained by choosing actions greedily based on $Q^*$, i.e., any policy $\pi$ satisfying $\{a\mid \pi(a\mid s)>0\} \subseteq \arg\max_{a'\in\mathcal{A}}Q^*(s,a')$ is an optimal policy. See \cite{bertsekas1996neuro,puterman1995markov} for more details.

The celebrated $Q$-learning algorithm can be viewed as a data-driven SA algorithm designed to solve $Q=\mathcal{H}(Q)$. Let $\pi_b$ be the behavior policy used to collect samples. Suppose that the Markov chain $\{S_k\}$ induced by $\pi_b$ has a unique stationary distribution $\kappa_b\in\Delta(\mathcal{S})$ satisfying $\kappa_b(s)>0$ for all $s\in\mathcal{S}$. Similarly to TD-learning, we impose the following i.i.d. assumption on the sample trajectory.

\begin{assumption}\label{as:Q-learning}
    The sample trajectory $\{(S_k,A_k,S_k')\}_{k\geq 0}$ (where each sample is an $3$-tuple) is an i.i.d. sequence such that $S_k\sim \kappa_b(\cdot)$, $A_k\sim \pi_b(\cdot|S_k)$, and $S_k'\sim P_{A_k}(S_k,\cdot)$ for all $k\geq 0$. 
\end{assumption}

As illustrated in the previous section, the i.i.d. sampling assumption has been commonly imposed in the literature \cite{na2024finite,lee2024final,wainwright2019variance}. Moreover, when the Markov chain $\{S_k\}$ induced by $\pi_b$ is irreducible and aperiodic, hence uniformly ergodic, there is a procedure to approximately achieve i.i.d. sampling. That being said, learning from a single trajectory of Markovian samples is a more practical and natural setup \cite{even2003learning,li2021q,li2024q}, which requires studying SA with biased updates and is a future direction of this work.

With the sample trajectory at hand, the $Q$-learning algorithm iteratively updates an estimate $Q_k$ of $Q^*$ according to the following formula:
\begin{align*}
    Q_{k+1}(S_k,A_k)=Q_k(S_k,A_k)+\alpha_k(\mathcal{R}(S_k,A_k)+\gamma \max_{a'\in\mathcal{A}}Q_k(S_k',a')-Q_k(S_k,A_k))
\end{align*}
for all $k\geq 0$, where $Q_0$ is initialized arbitrarily but satisfies $\|Q_0\|_\infty\leq 1/(1-\gamma)$. 

We next remodel $Q$-learning in the form of the SA algorithm presented in Eq. (\ref{eq:stochastic-approximation}). Let $F:\mathbb{R}^{|\mathcal{S}||\mathcal{A}|}\times \mathcal{S}\times \mathcal{A}\times \mathcal{S}\mapsto \mathbb{R}^{|\mathcal{S}||\mathcal{A}|} $ be an operator defined as
\begin{align*}
    [F(Q,s_0,a_0,s_1)](s,a)
    =\mathds{1}_{\{(s_0,a_0)=(s,a)\}}\big(\mathcal{R}(s_0,a_0)+\gamma \max_{a'\in\mathcal{A}}Q(s_1,a')-Q(s_0,a_0)\big)+Q(s,a)
\end{align*}
for all $(s,a)$ and $(Q,s_0,a_0,s_1)$. Then the update equation of $Q$-learning can be equivalently written as
\begin{align*}
    Q_{k+1}=Q_k+\alpha_k (F(Q_k,S_k,A_k,S_k')-Q_k)
\end{align*}
for all $k\geq 0$,
which is in the same form of the SA algorithm in Eq. (\ref{eq:stochastic-approximation}) with $x_k$ being $Q_k$ and $Y_k$ being the triple $(S_k,A_k,S_k')$.

Next, we will use our results on SA with sub-Gaussian additive noise to establish the maximal concentration bounds of $Q$-learning. We start by verifying in the following lemma that Assumptions \ref{as:contraction}, \ref{as:unbiased}, and \ref{ass:sub-Gaussian} are satisfied in the context of $Q$-learning.
Let $D_b$ be an $|\mathcal{S}||\mathcal{A}|$ by $ |\mathcal{S}||\mathcal{A}|$ diagonal matrix with diagonal components $\{\kappa_b(s)\pi_b(a|s)\}_{(s,a)\in\mathcal{S}\times \mathcal{A}}$. Denote the minimum diagonal entry of  $D_b$ by $D_{b,\min}$. Let $\mathcal{F}_k$ be the $\sigma$-algebra generated by $\{S_i,A_i,S_i'\}_{0\leq i\leq k-1}$. Note that $Q_k$ is measurable with respect to $\mathcal{F}_k$. The proof of the following lemma is presented in Appendix \ref{pf:prop:Q-learning}.

\begin{lemma}\label{prop:Q-learning}
    The operator $\bar{F}(\cdot):=\mathbb{E}[F(\cdot,S_k,A_k,S_k')]$ is explicitly given as $\bar{F}(Q)=D_b\mathcal{H}(Q)+(I_{|\mathcal{S}||\mathcal{A}|}-D_b)Q$ for all $Q\in\mathbb{R}^{|\mathcal{S}||\mathcal{A}|}$.
    In addition, we have the following results.
    \begin{enumerate}[(1)]
        \item $\bar{F}(\cdot)$ is a $\hat{\gamma}_c$-contraction mapping with respect to $\|\cdot\|_\infty$, where $\hat{\gamma}_c=1-D_{b,\min}(1-\gamma)$.
        \item $\bar{F}(Q_k)=\mathbb{E}[F(Q_k,S_k,A_k,S_k')\mid \mathcal{F}_k]$ for all $k\geq 0$.
        \item Assumption \ref{ass:sub-Gaussian} holds with $\bar{\sigma}=4/(1-\gamma)$ and $c_d=1$.
    \end{enumerate}
\end{lemma}

Lemma \ref{prop:Q-learning} enables us to apply Theorem \ref{thm:additive} to get maximal concentration bound of $Q$-learning. The result is presented in the following theorem, the proof of which can be found in Appendix \ref{pf:thm:Q-learning}.

\begin{theorem}\label{thm:Q-learning}
Suppose that Assumption \ref{as:Q-learning} is satisfied and $\alpha_k=\alpha/(k+h)$, where $\alpha>2/(1-\hat{\gamma}_c)$ and $h$ is appropriately chosen. Then, for any $K\geq 0$ and $\delta\in (0,1)$, with probability at least $1-\delta$, we have for all $k\geq K$ that
\begin{align*}
    \|Q_k-Q^*\|_\infty^2\leq\;&c_q\left[\frac{\log(1/\delta)}{k+h}+\left(\frac{h}{k+h}\right)^{(1-\hat{\gamma}_c)\alpha/2}+\frac{1+\log((k+1)/K^{1/2})}{k+h}\right],
\end{align*}
    where $c_q=\frac{\log(|\mathcal{S}||\mathcal{A}|)}{D_{b,\min}^3(1-\gamma)^5}$.
\end{theorem}

In view of Theorem \ref{thm:Q-learning}, the sample complexity to achieve $\|Q_k-Q^*\|_\infty\leq \epsilon$ is $\tilde{\mathcal{O}}(D_{b,\min}^{-3}(1-\gamma)^{-5}\epsilon^{-2})$. Due to performing asynchronous update, the convergence rate naturally depends on the minimum component $D_{b,\min}$ of the stationary distribution on the Markov chain $\{(S_k,A_k)\}$. Such a quantity is at best inverse proportional to the size of the state-action space. The fact that the sample complexity depends on $1/(1-\gamma)$ is intuitive because as $\gamma$ increases, the agent should look further into the future when making decisions, which makes the problem more challenging. Similar dependence has been observed in the existing study of $Q$-learning \cite{chen2023lyapunov,li2024q,lee2024final,even2003learning,beck2012error,beck2013improved,wainwright2019stochastic,wainwright2019variance} and other algorithms such as policy gradient \citep{mei2020global,li2021softmax}.

\section{Conclusion}\label{sec:conclusion}

In this paper, we establish maximal concentration bounds for general contractive SA with additive and multiplicative noise. Specifically, we show that the sample paths remain in a cone (with a decaying radius) with high probability. Moreover, we showcase how these general bounds can be applied to linear SA and various RL algorithms. Methodologically, to overcome the challenge of having unbounded iterates, we develop a novel bootstrapping argument, where we start with a potentially loose bound and iteratively improve it to obtain a tighter one. The key steps involve bounding the log-MGF of a modified version of the generalized Moreau envelope of the convergence error and carefully constructing supermartingales to obtain maximal bounds.

\textit{Future Work.} Our main results require the noise sequence to be conditionally unbiased. However, in practical applications, many times the sample trajectory can have biased perturbations. For example, in RL, suppose that the agent interacts with the environment to collect a single trajectory of samples. Then, the sample trajectory forms a Markov chain, which has biased perturbation. Extending our result to the case where the operator can have biased perturbations is an immediate future direction of this work. Another direction is to extend our result to the more challenging but also practically relevant setting of SA with multiple timescales, i.e., the SA algorithm maintains a set of iterates and updates each one of them using potentially order-wise different stepsizes. On the application side, we would like to see if our results can be applied to other algorithms beyond linear SA and RL.

\section*{Acknowledgement}
We would like to thank Prof. R. Srikant from the University of Illinois at Urbana-Champaign for
the insightful comments about using the telescoping technique to establish maximal concentration
bounds.

\bibliographystyle{apalike}
\bibliography{references}

\newpage
\begin{center}
    {\LARGE\bfseries Appendices}
\end{center}

\appendix

\section{Proof of Technical Results in Support of Theorem \ref{thm:multi}}

\subsection{Proof of Lemma \ref{le:help1}}\label{pf:le:help1}
Since the proof for the case where $\beta_1=0$ is trivial, we will only focus on the case where $\beta_1\neq 0$. It can be easily shown by induction that $w_k> 0$ for all $k\geq 0$. Using the numerical inequality $1+x\leq e^x$ for all $x\in\mathbb{R}$, we have for all $k\geq 0$ that
\begin{align*}
    w_{k+1}\leq e^{\beta_1\alpha_k}w_k+\beta_2\alpha_k.
\end{align*}
Multiplying both sides of the previous inequality by $e^{-\beta_1\sum_{i=0}^k\alpha_i}$ , we obtain
\begin{align*}
    e^{-\beta_1\sum_{i=0}^k\alpha_i}w_{k+1}\leq\;& e^{-\beta_1\sum_{i=0}^{k-1}\alpha_i}w_k+\beta_2\alpha_k e^{-\beta_1\sum_{i=0}^k\alpha_i}\\
    =\;& e^{-\beta_1\sum_{i=0}^{k-1}\alpha_i} w_k+\frac{\beta_2}{\beta_1}(1+\beta_1\alpha_k-1) e^{-\beta_1\sum_{i=0}^k\alpha_i}\\
    \leq\;& e^{-\beta_1\sum_{i=0}^{k-1}\alpha_i}w_k+\frac{\beta_2}{\beta_1}(e^{\beta_1 \alpha_k}-1) e^{-\beta_1\sum_{i=0}^k\alpha_i}\\
    \leq\;& e^{-\beta_1\sum_{i=0}^{k-1}\alpha_i}w_k+\frac{\beta_2}{\beta_1}(e^{-\beta_1\sum_{i=0}^{k-1}\alpha_i}-e^{-\beta_1\sum_{i=0}^k\alpha_i}).
\end{align*}
By telescoping, we have
\begin{align*}
    e^{-\beta_1\sum_{i=0}^{k-1}\alpha_i}w_k\leq w_0+\frac{\beta_2}{\beta_1}(1-e^{-\beta_1\sum_{i=0}^{k-1}\alpha_i}),
\end{align*}
which implies
\begin{align*}
    w_k\leq e^{\beta_1\sum_{i=0}^{k-1}\alpha_i}w_0+\frac{\beta_2}{\beta_1}(e^{\beta_1\sum_{i=0}^{k-1}\alpha_i}-1).
\end{align*}
This proves the case where $\beta_1>0$. When $\beta_1<0$, note that in this case we have $e^{\beta_1\sum_{i=0}^{k-1}\alpha_i}\in (0,1)$. Therefore, we have from the previous inequality that
\begin{align*}
    w_k\leq \;& e^{\beta_1\sum_{i=0}^{k-1}\alpha_i}w_0+\frac{\beta_2}{\beta_1}(e^{\beta_1\sum_{i=0}^{k-1}\alpha_i}-1)\\
= \;&e^{\beta_1\sum_{i=0}^{k-1}\alpha_i}\left(w_0+\frac{\beta_2}{\beta_1}\right)-\frac{\beta_2}{\beta_1}\\
    \leq \;&\begin{dcases}
        w_0,&w_0+\frac{\beta_2}{\beta_1}\geq 0,\\
        -\frac{\beta_1}{\beta_1},&w_0+\frac{\beta_2}{\beta_1}< 0,
    \end{dcases}\\
    \leq \;&w_0-\frac{\beta_2}{\beta_1}.
\end{align*}
The proof is now complete.

\subsection{Proof of Proposition \ref{prop:Bound-log-MGF}}\label{pf:le:multi_solving_recursion}
Repeatedly using Eq. (\ref{prop:Bound-log-MGF-eq1}), we have for all $k\geq 0$ that
\begin{align}
	Z_k\leq\;& \underbrace{\exp\left(-\frac{\alpha D_0/2-1}{\alpha}\sum_{i=0}^{k-1}\alpha_i\right)}_{T_5}Z_0\nonumber\\
	&+2D_2(1+\|x^*\|_c)^2\underbrace{\sum_{i=0}^{k-1}\alpha_i^2\lambda_i\exp\left(-\frac{\alpha D_0/2-1}{\alpha}\sum_{j=i+1}^{k-1}\alpha_j\right)}_{T_6}.\label{eq:multi_T5T6}
\end{align}
It remains to bound the terms $T_5$ and $T_6$. We first consider the term $T_5$. Note that for any non-increasing function $f:\mathbb{R}\mapsto\mathbb{R}$ and $n_1\leq n_2$, we have
\begin{align*}
    \int_{n_1}^{n_2+1}f(x)dx\leq \sum_{i=n_1}^{n_2}f(i)\leq \int_{n_1-1}^{n_2}f(x)dx.
\end{align*}
Therefore, we have
\begin{align*}
	\sum_{i=0}^{k-1}\alpha_i=\sum_{i=0}^{k-1}\frac{\alpha}{i+h}
	\geq \int_{0}^{k}\frac{\alpha}{x+h}dx
	=\alpha\log\left(\frac{k+h}{h}\right).
\end{align*}
It follows that
\begin{align*}
	T_5\leq \exp\left(-(\alpha D_0/2-1)\log\left(\frac{k+h}{h}\right)\right)=\left(\frac{h}{k+h}\right)^{\alpha D_0/2-1}.
\end{align*}
As for the term $T_6$, we have
\begin{align*}
	T_6=\;&\sum_{i=0}^{k-1}\alpha_i^2\lambda_i\exp\left(-\frac{\alpha D_0/2-1}{\alpha}\sum_{j=i+1}^{k-1}\alpha_j\right)\\
	\leq \;&\sum_{i=0}^{k-1}\alpha_i^2\lambda_i\exp\left(-\frac{\alpha D_0/2-1}{\alpha}\alpha\log\left(\frac{k+h}{i+1+h}\right)\right)\\
	=\;&\sum_{i=0}^{k-1}\alpha_i^2\lambda_i\left(\frac{i+1+h}{k+h}\right)^{\alpha D_0/2-1}\\
	=\;&\left(\frac{1}{k+h}\right)^{\alpha D_0/2-1}\sum_{i=0}^{k-1}\alpha_i^2\lambda_i(i+1+h)^{\alpha D_0/2-1}.
\end{align*}
To proceed, since $\lambda_k=\theta\alpha_k^{-1}T_k(\delta)^{-1}$ and $\|x_0-x^*\|_c^2\leq T_0(\delta)\leq T_k(\delta)$, we have
\begin{align*}
	\alpha_k^2\lambda_k=\frac{\theta\alpha_k}{T_k(\delta)}\leq \frac{\theta \alpha_k}{T_0(\delta)}\leq \frac{\theta \alpha_k}{\|x_0-x^*\|_c^2}.
\end{align*}
It follows that
\begin{align*}
	T_6
	\leq \;&\left(\frac{1}{k+h}\right)^{\alpha D_0/2-1}\frac{\theta}{\|x_0-x^*\|_c^2}\sum_{i=0}^{k-1}\frac{\alpha}{i+h}(i+1+h)^{\alpha D_0/2-1}\\
	= \;&\left(\frac{1}{k+h}\right)^{\alpha D_0/2-1}\frac{\theta}{\|x_0-x^*\|_c^2}\sum_{i=0}^{k-1}\frac{\alpha(i+1+h)}{i+h}(i+1+h)^{\alpha D_0/2-2}\\
	\leq  \;&2\alpha\left(\frac{1}{k+h}\right)^{\alpha D_0/2-1}\frac{\theta}{\|x_0-x^*\|_c^2}\sum_{i=0}^{k-1}(i+1+h)^{\alpha D_0/2-2}.
\end{align*}
Similarly, using integration to bound the summation, we have
\begin{align*}
	\sum_{i=0}^{k-1}\frac{1}{(i+1+h)^{2-\alpha D_0/2}}\leq\;& \begin{dcases}
		\frac{1}{\alpha D_0/2-1}(k+h)^{\alpha D_0/2-1},&\alpha \in (2/D_0,4/D_0),\\
		k,&\alpha =4/D_0,\\
		\frac{e}{\alpha D_0/2-1}(k+h)^{\alpha D_0/2-1},&\alpha>4/D_0.
	\end{dcases}
\end{align*}
It follows that
\begin{align*}
	T_6\leq \frac{2\alpha e \theta}{(\alpha D_0/2-1)\|x_0-x^*\|_c^2}.
\end{align*}
Using the upper bounds we obtained for the terms $T_5$ and $T_6$ in Eq. (\ref{eq:multi_T5T6}) , we have
\begin{align*}
	Z_k\leq  Z_0\left(\frac{h}{k+h}\right)^{\alpha D_0/2-1}+\frac{4\alpha e \theta D_2}{\alpha D_0/2-1}\frac{(1+\|x^*\|_c)^2}{\|x_0-x^*\|_c^2}.
\end{align*}
This proves Eq. (\ref{prop:Bound-log-MGF-eq2}) of Proposition \ref{prop:Bound-log-MGF}.

\subsection{Proof of Proposition \ref{prop:ville}}\label{pf:prop:ville}
For any $\epsilon>0$ and $K\geq 0$, we have
\begin{align}
	&\mathbb{P}\left(\sup_{k\geq K}\left\{\lambda_k\mathds{1}_{E_k(\delta)} M(x_k-x^*)-D_3\sum_{i=0}^{k-1}\alpha_i\right\}> \epsilon\right)\nonumber\\
	= \;&\mathbb{P}\left(\sup_{k\geq K}\left\{\exp\left(\lambda_k\mathds{1}_{E_k(\delta)} M(x_k-x^*)-D_3\sum_{i=0}^{k-1}\alpha_i\right)\right\}> e^\epsilon\right)\nonumber\\
	\leq \;&\mathbb{E}\left[\exp\left(\lambda_K\mathds{1}_{E_K(\delta)} M(x_K-x^*)-D_3\sum_{i=0}^{K-1}\alpha_i-\epsilon\right)\right]\label{eq:Ville-1}\\
	\leq \;&\exp\left(Z_0\left(\frac{h}{K+h}\right)^{\alpha D_0/2-1}+\frac{4\alpha e D_2 \theta}{\alpha D_0/2-1}\frac{(1+\|x^*\|_c)^2}{\|x_0-x^*\|_c^2}-D_3\sum_{i=0}^{K-1}\alpha_i-\epsilon\right),\label{eq:Ville-2}
\end{align}
where Eq. (\ref{eq:Ville-1}) follows from Ville’s maximal inequality and Eq. (\ref{eq:Ville-2}) follows from Proposition \ref{prop:Bound-log-MGF}. 
Let 
\begin{align*}
	\delta'=\exp\left(Z_0\left(\frac{h}{K+h}\right)^{\alpha D_0/2-1}+\frac{4\alpha e D_2 \theta}{\alpha D_0/2-1}\frac{(1+\|x^*\|_c)^2}{\|x_0-x^*\|_c^2}-D_3\sum_{i=0}^{K-1}\alpha_i-\epsilon\right),
\end{align*}
which implies
\begin{align*}
	\epsilon=\log(1/\delta')+Z_0\left(\frac{h}{K+h}\right)^{\alpha D_0/2-1}+\frac{4\alpha e D_2 \theta}{\alpha D_0/2-1}\frac{(1+\|x^*\|_c)^2}{\|x_0-x^*\|_c^2}-D_3\sum_{i=0}^{K-1}\alpha_i.
\end{align*}
Therefore, with probability at least $1-\delta'$, we have
\begin{align}
    &\sup_{k\geq K}\left\{\lambda_k\mathds{1}_{E_k(\delta)} M(x_k-x^*)-D_3\sum_{i=0}^{k-1}\alpha_i\right\}\nonumber\\
    \leq \;& \log(1/\delta')+Z_0\left(\frac{h}{K+h}\right)^{\alpha D_0/2-1}+\frac{4\alpha e D_2 \theta}{\alpha D_0/2-1}\frac{(1+\|x^*\|_c)^2}{\|x_0-x^*\|_c^2}-D_3\sum_{i=0}^{K-1}\alpha_i.\label{eq:sup1}
\end{align}
Our last step is to bound $Z_0$. Observe that
\begin{align*}
	Z_0=\lambda_0M(x_k-x^*)
	=\frac{\theta\|x_k-x^*\|_M^2}{2\alpha_0 T_0(\delta)}
	\leq \frac{\theta \|x_0-x^*\|_M^2}{2\alpha_0 \|x_0-x^*\|_c^2}
	\leq \frac{\theta}{2\alpha_0\ell_{cM}^2}
	\leq \frac{D_0}{16\alpha_0D_1\ell_{cM}^2},
\end{align*}
where we used $M(x)=\frac{1}{2}\|x\|_M^2$ (cf. Lemma \ref{prop:Moreau}), $\|x_0-x^*\|_c^2\leq T_0(\delta)$, and $\theta=D_0\|x_0-x^*\|_c^2/[8D_1((1+\|x^*\|_c)^2+\|x_0-x^*\|_c^2)]$. Using the bound we obtained for $Z_0$ in Eq. (\ref{eq:sup1}) yields the result.

we have with probability at least $1-\delta'$:
\begin{align*}
	\sup_{k\geq K}\left\{\lambda_k\mathds{1}_{E_k(\delta)} M(x_k-x^*)\right\}
	\leq \;&\log(1/\delta')+\frac{D_0}{16\alpha_0D_1\ell_{cM}^2}\left(\frac{h}{K+h}\right)^{\alpha D_0/2-1}\\
	&+\frac{4\alpha e D_2 \theta}{\alpha D_0/2-1}\frac{(1+\|x^*\|_c)^2}{\|x_0-x^*\|_c^2}+D_3\sum_{i=K}^{k-1}\alpha_i.
\end{align*}
The last step is to evaluate $\sum_{i=K}^{k-1}\alpha_i$ and connect $M(x_k-x^*)$ with $\|x_k-x^*\|_c^2$. Since
\begin{align*}
	\sum_{i=K}^{k-1}\alpha_i=\sum_{i=K}^{k-1}\frac{\alpha}{i+h}\leq \int_{K-1}^{k-1}\frac{\alpha}{x+h}dx=\alpha \log\left(\frac{k-1+h}{K-1+h}\right)
\end{align*}
and
\begin{align*}
	M(x_k-x^*)=\frac{1}{2}\|x_k-x^*\|_M^2\geq  \frac{1}{2u_{cM}^2}\|x_k-x^*\|_c^2\tag{Lemma \ref{prop:Moreau} (3)},
\end{align*}
we have with probability at least $1-\delta'$ that 
\begin{align*}
	\sup_{k\geq K}\left\{\lambda_k\mathds{1}_{E_k(\delta)} \|x_k-x^*\|_c^2\right\}
	\leq \;&2u_{cM}^2\log(1/\delta')+\frac{2u_{cM}^2D_0}{16\alpha_0D_1\ell_{cM}^2}\left(\frac{h}{K+h}\right)^{\alpha D_0/2-1}\\
	&+\frac{8u_{cM}^2\alpha e D_2 \theta}{\alpha D_0/2-1}\frac{(1+\|x^*\|_c)^2}{\|x_0-x^*\|_c^2}+2\alpha D_3 u_{cM}^2\log\left(\frac{k-1+h}{K-1+h}\right).
\end{align*}

\subsection{Proof of Claim \ref{claim:finish-bootstrapping}}\label{pf:le:first_bootstrap}
Let
\begin{align*}
    D_4=\;&\frac{D_0((1+\|x^*\|_c)^2}{8D_1((1+\|x^*\|_c)^2+\|x_0-x^*\|_c^2)},\;c_1=\frac{(32D_1(1+D_4)u_{cM}^{2})^m(D^2+\sigma^2 D_4)}{D_0^m D^2},\\
    c_2=\;&\frac{D_0}{16\alpha_0D_1\ell_{cM}^2},
    c_3=\frac{8\alpha e D_2D_4}{\alpha D_0-2},\;c_4=\alpha D_3,\;\text{ and }
    c_1'=\frac{32u_{cM}^2D_1 (1+\sigma D_4\alpha^2 )(D_4+1)}{D_0}.
\end{align*}
We start with Proposition \ref{prop:worst_case_bound}, which states that 
\begin{align*}
	\mathbb{P}(\|x_k-x^*\|_c^2\leq B_k(D)^2,\forall\;k\geq 0)=1,
\end{align*}
where
\begin{align*}
	B_k(D)=\begin{dcases}
		\left(\frac{k-1+h}{h-1}\right)^{\alpha D}\left(\|x_0-x^*\|_c+\frac{\sigma(1+\|x^*\|_c)}{D}\right),&D> 0,\\
		\|x_0-x^*\|_c+\sigma(1+\|x^*\|_c)\alpha \log\left(\frac{k-1+h}{h-1}\right),&D=0.
	\end{dcases} 
\end{align*}

\begin{enumerate}[(1)]
    \item We first consider the case where $D>0$. Recall that we have assumed (without loss of generality) that $2\alpha D$ is an integer and defined $m=2\alpha D+1$. Let $\delta_1,\delta_2,\cdots,\delta_m>0$ be such that $\sum_{i=1}^m\delta_i\leq 1$. Repeatedly using Proposition \ref{prop:bootstrapping} for $m$ times , we have
with probability at least $1-\sum_{i=1}^m\delta_i$ that 
\begin{align*}
	\|x_k-x^*\|_c^2
	\leq \;&\frac{\alpha_k^m B_k^2(D)}{\theta^m} \prod_{i=1}^{m-1}\bigg[2u_{cM}^2\log(1/\delta_i)+\frac{2u_{cM}^2D_0}{16\alpha_0D_1\ell_{cM}^2}\nonumber\\
	&+\frac{16u_{cM}^2\alpha e D_2 \theta}{\alpha D_0-2}\frac{(1+\|x^*\|_c)^2}{\|x_0-x^*\|_c^2}+2\alpha D_3 u_{cM}^2\log\left(\frac{k-1+h}{h-1}\right)\bigg]\\
	&\times \bigg[2u_{cM}^2\log(1/\delta_m)+\frac{2u_{cM}^2D_0}{16\alpha_0D_1\ell_{cM}^2}\left(\frac{h}{K+h}\right)^{\alpha D_0/2-1}\nonumber\\
	&+\frac{16u_{cM}^2\alpha e D_2 \theta}{\alpha D_0-2}\frac{(1+\|x^*\|_c)^2}{\|x_0-x^*\|_c^2}+2\alpha D_3 u_{cM}^2\log\left(\frac{k-1+h}{K-1+h}\right)\bigg]\\
	= \;&\frac{2^mu_{cM}^{2m}\alpha_k^m B_k^2(D)}{\theta^m} \prod_{i=1}^{m-1}\bigg[\log(1/\delta_i)+c_2+c_3+c_4 \log\left(\frac{k-1+h}{h-1}\right)\bigg]\\
	&\times \bigg[\log(1/\delta_m)+c_2\left(\frac{h}{K+h}\right)^{\alpha D_0/2-1}+c_3+c_4 \log\left(\frac{k-1+h}{K-1+h}\right)\bigg].
\end{align*}
By choosing $\delta_i=\delta/m$ for all $i\in \{1,2,\cdots,m\}$, the previous inequality reads
\begin{align*}
	\|x_k-x^*\|_c^2
	\leq \;&\frac{2^m u_{cM}^{2m}\alpha_k^m B_k^2(D)}{\theta^m} \bigg[\log(m/\delta)+c_2+c_3+c_4 \log\left(\frac{k-1+h}{h-1}\right)\bigg]^{m-1}\\
	&\times \bigg[\log(m/\delta)+c_2\left(\frac{h}{K+h}\right)^{\alpha D_0/2-1}+c_3+c_4 \log\left(\frac{k-1+h}{K-1+h}\right)\bigg].
\end{align*}
To proceed, observe that
\begin{align*}
    \frac{2^{m} u_{cM}^{2m}\alpha_k^m B_k^2(D)}{\theta^m}
    =\;&\frac{2^{4m} u_{cM}^{2m} D_1^m(1+D_4)^m\alpha_k \alpha^{m-1}}{D_0^m (k+h)^{m-1}}\left(\frac{k-1+h}{h-1}\right)^{m-1}\\
    &\times \left(1+\sigma D_4^{1/2}/D\right)^2\|x_0-x^*\|_c^2\\
    \leq \;&\frac{32^m u_{cM}^{2m} D_1^m(1+D_4)^m \left(1+\sigma^2 D_4/D^2\right)}{D_0^m }\alpha_k\|x_0-x^*\|_c^2\\
    =\;&c_1\alpha_k\|x_0-x^*\|_c^2.
\end{align*}
Therefore, for any $\delta>0$ and $K\geq 0$, we have with probability at least $1-\delta$ that 
\begin{align*}
    \|x_k-x^*\|_c^2
	\leq \;&c_1\alpha_k\|x_0-x^*\|_c^2 \bigg[\log(m/\delta)+c_2+c_3+c_4 \log\left(\frac{k-1+h}{h-1}\right)\bigg]^{m-1}\\
	&\times \bigg[\log(m/\delta)+c_2\left(\frac{h}{K+h}\right)^{\alpha D_0/2-1}+c_3+c_4 \log\left(\frac{k-1+h}{K-1+h}\right)\bigg]
\end{align*}
for all $k\geq K$. The proof of Theorem \ref{thm:multi} (1) is complete.
    \item When $D=0$, using Proposition \ref{prop:bootstrapping} one time , we have with probability at least $1-\delta$ that
\begin{align*}
	\|x_k-x^*\|_c^2\leq \;&\frac{\alpha_kB_k(D)^2}{\theta}\bigg[2u_{cM}^2\log(1/\delta)+\frac{2u_{cM}^2D_0}{16\alpha_0D_1\ell_{cM}^2}\left(\frac{h}{K+h}\right)^{\alpha D_0/2-1}\nonumber\\
	&+\frac{16u_{cM}^2\alpha e D_2D_4}{\alpha D_0-2}+2\alpha D_3 u_{cM}^2\log\left(\frac{k-1+h}{K-1+h}\right)\bigg]\\
	\leq \;&\frac{16D_1 (1+\sigma D_4\alpha^2 )(D_4+1)}{D_0}\|x_0-x^*\|_c^2\alpha_k\log^2\left(\frac{k-1+h}{h-1}\right)\\
	&\times \bigg[2u_{cM}^2\log(1/\delta)+\frac{2u_{cM}^2D_0}{16\alpha_0D_1\ell_{cM}^2}\left(\frac{h}{K+h}\right)^{\alpha D_0/2-1}\nonumber\\
	&+\frac{16u_{cM}^2\alpha e D_2D_4}{\alpha D_0-2}+2\alpha D_3 u_{cM}^2\log\left(\frac{k-1+h}{K-1+h}\right)\bigg]
\end{align*}
for all $k\geq K$. Therefore, using the definition of $\{c_i\}_{2\leq i\leq 4}$ and $c_1'$, the previous inequality reads
\begin{align*}
	\|x_k-x^*\|_c^2
	\leq \;&c_1'\|x_0-x^*\|_c^2\alpha_k\log^2\left(\frac{k-1+h}{h-1}\right)\\
	&\times \left[\log(1/\delta)+c_2\left(\frac{h}{K+h}\right)^{\alpha D_0/2-1}+c_3+c_4 \log\left(\frac{k-1+h}{K-1+h}\right)\right].
\end{align*}
The proof of Theorem \ref{thm:multi} (2) is complete.
\item When $D<0$, using Proposition \ref{prop:bootstrapping} one time , we have with probability at least $1-\delta$ that
\begin{align*}
	\|x_k-x^*\|_c^2\leq \;&\frac{\alpha_kB_k(D)^2}{\theta}\bigg[2u_{cM}^2\log(1/\delta)+\frac{2u_{cM}^2D_0}{16\alpha_0D_1\ell_{cM}^2}\left(\frac{h}{K+h}\right)^{\alpha D_0/2-1}\nonumber\\
	&+\frac{16u_{cM}^2\alpha e D_2D_4}{\alpha D_0-2}+2\alpha D_3 u_{cM}^2\log\left(\frac{k-1+h}{K-1+h}\right)\bigg]\\
	\leq \;&\frac{8\alpha_k D_1((1+\|x^*\|_c)^2+\|x_0-x^*\|_c^2)}{D_0\|x_0-x^*\|_c^2}\left(\|x_0-x^*\|_c-\frac{\sigma(1+\|x^*\|_c)}{D}\right)^2\\
	&\times \bigg[2u_{cM}^2\log(1/\delta)+\frac{2u_{cM}^2D_0}{16\alpha_0D_1\ell_{cM}^2}\left(\frac{h}{K+h}\right)^{\alpha D_0/2-1}\nonumber\\
	&+\frac{16u_{cM}^2\alpha e D_2D_4}{\alpha D_0-2}+2\alpha D_3 u_{cM}^2\log\left(\frac{k-1+h}{K-1+h}\right)\bigg]\\
 \leq \;& \frac{16 (D^2+\sigma^2)D_1}{D_0D^2}\frac{(\|x_0-x^*\|_c^2+(1+\|x^*\|_c)^2)^2}{\|x_0-x^*\|_c^4}\|x_0-x^*\|_c^2\alpha_k \\
 &\times \bigg[2u_{cM}^2\log(1/\delta)+\frac{2u_{cM}^2D_0}{16\alpha_0D_1\ell_{cM}^2}\left(\frac{h}{K+h}\right)^{\alpha D_0/2-1}\nonumber\\
	&+\frac{16u_{cM}^2\alpha e D_2D_4}{\alpha D_0-2}+2\alpha D_3 u_{cM}^2\log\left(\frac{k-1+h}{K-1+h}\right)\bigg]\\
 \leq \;&c_1''\|x_0-x^*\|_c^2\alpha_k\left[\log(1/\delta)+c_2\left(\frac{h}{K+h}\right)^{\alpha D_0/2-1}\right.\\
 &\left.+c_3+c_4 \log\left(\frac{k-1+h}{K-1+h}\right)\right],
\end{align*}
where 
\begin{align*}
    c_1''=\frac{16 (D^2+\sigma^2)D_1}{D_0D^2}\frac{(\|x_0-x^*\|_c^2+(1+\|x^*\|_c)^2)^2}{\|x_0-x^*\|_c^4}.
\end{align*}
The proof of Theorem \ref{thm:multi} (3) is complete.
\end{enumerate}

\subsection{Removing the Product of Logarithmic Factors}\label{ap:removing_log}

\begin{theoremp}{\ref*{thm:multi}$'$}\label{thm:improved}
Under the same assumptions for Theorem \ref{thm:multi} (1), for any $\delta\in (0,1)$ and $K\geq 0$, with probability at least $1-\delta$, we have for all $k\geq K$ that
\begin{align*}
    \|x_k-x^*\|_c^2\leq \;& \frac{c_1''\alpha \|x_0-x^*\|_c^2}{k+h} \left[\left(\log\left(\frac{m+1}{\delta}\right)\right)^m+1\right]\left[\log\left(\frac{m+1}{\delta}\right)\right.\\
    &\left.+c_2\left(\frac{h}{K+h}\right)^{\alpha D_0/2-1}+c_3+c_4\log\left(\frac{k-1+h}{K-1+h}\right)\right].
\end{align*}
\end{theoremp}
\begin{proof}[Proof of Theorem \ref{thm:improved}]
Recall from Proposition \ref{prop:blueprint} that we require the initial bound $\{T_k(\delta)\}$ to be non-decreasing to initialize the bootstrapping argument. However, the right-hand side of of the inequality in Theorem \ref{thm:multi} (2) will eventually be decreasing as the polynomial term will dominate the logarithmic terms when $k$ is large enough. To overcome this difficulty, define
\begin{align*}
    \tilde{T}_k(\delta)=\;&c_1\|x_0-x^*\|_c^2\sup_{0\leq k'\leq k}\left\{\alpha_{k'} \left[\log\left(\frac{m}{\delta}\right)+c_2+c_3+c_4\log\left(\frac{k'-1+h}{h-1}\right)\right]^m\right\}.
\end{align*}
for any $\delta>0$ and $k\geq 0$. Note that $T_k(\delta)$ is (by definition) a non-decreasing sequence, and Theorem \ref{thm:multi} (2) implies that
\begin{align}\label{eq:need_for_last_bootstrap}
    \mathbb{P}\left(\|x_k-x^*\|_c^2\leq  \tilde{T}_k(\delta),\;\forall\;k\geq 0\right)\geq 1-\delta.
\end{align}
Now that $\{\tilde{T}_k(\delta)\}$ is a non-decreasing sequence, performing an additional bootstrapping step to remove the product of the logarithmic terms proves Theorem \ref{thm:improved}. Next, we carry out the details.

We start with Eq. (\ref{eq:need_for_last_bootstrap}) and perform one step of bootstrapping using Proposition \ref{prop:bootstrapping}. For any $\delta,\delta'\in (0,1)$ and $K\geq 0$, with probability at least $1-\delta-\delta'$, we have for all $k\geq K$ that
\begin{align*}
	\|x_k-x^*\|_c^2
	\leq \;&\frac{\alpha_k \tilde{T}_k(\delta)}{\theta} \bigg[2u_{cM}^2\log(1/\delta')+\frac{2u_{cM}^2D_0}{16\alpha_0D_1\ell_{cM}^2}\left(\frac{h}{K+h}\right)^{\alpha D_0/2-1}\nonumber\\
	&+\frac{16u_{cM}^2\alpha e D_2 \theta}{\alpha D_0-2}\frac{(1+\|x^*\|_c)^2}{\|x_0-x^*\|_c^2}+2\alpha D_3 u_{cM}^2\log\left(\frac{k-1+h}{K-1+h}\right)\bigg]\\
	= \;&\frac{16c_1D_1(D_4+1)\alpha_k u_{cM}^2 }{D_0}\|x_0-x^*\|_c^2\\
	&\times \sup_{0\leq k'\leq k}\left\{\alpha_{k'} \left[\log\left(\frac{m}{\delta}\right)+c_2+c_3+c_4\log\left(\frac{k'-1+h}{h-1}\right)\right]^m\right\}\\
	&\times \left[\log(1/\delta')+c_2\left(\frac{h}{K+h}\right)^{\alpha D_0/2-1}+c_3+c_4\log\left(\frac{k-1+h}{K-1+h}\right)\right].
\end{align*}
Using the numerical inequality $(a+b)^{n+1}\leq 2^n(a^{n+1}+b^{n+1})$ for all $n\geq 0$ and $a,b>0$ , we have
\begin{align*}
    &\sup_{0\leq k'\leq k}\left\{\alpha_{k'} \left[\log\left(\frac{m}{\delta}\right)+c_2+c_3+c_4\log\left(\frac{k'-1+h}{h-1}\right)\right]^m\right\}\\
    \leq \;&2^{m-1}\sup_{0\leq k'\leq k}\left\{\alpha_{k'} \left[\left(\log\left(\frac{m}{\delta}\right)\right)^m+\left(c_2+c_3+c_4\log\left(\frac{k'-1+h}{h-1}\right)\right)^m\right]\right\}\\
    \leq \;&2^{m-1}\left(\log\left(\frac{m}{\delta}\right)\right)^m+2^{m-1}\sup_{k'\geq 0}\left\{\alpha_{k'} \left(c_2+c_3+c_4\log\left(\frac{k'-1+h}{h-1}\right)\right)^m\right\}\\
    =\;&2^{m-1}(1+c_5)\left[\left(\log\left(\frac{m}{\delta}\right)\right)^m+1\right],
\end{align*}
where 
\begin{align*}
    c_5:=\sup_{k'\geq 0}\left\{\alpha_{k'} \left(c_2+c_3+c_4\log\left(\frac{k'-1+h}{h-1}\right)\right)^m\right\}.
\end{align*}
As a result, for any $\delta,\delta'>0$ and $K\geq 0$, the following inequality holds with probability at least $1-\delta-\delta'$:
\begin{align*}
	\|x_k-x^*\|_c^2
	\leq  \;&\frac{16c_1D_1(D_4+1)\alpha_k u_{cM}^2 }{D_0}\|x_0-x^*\|_c^2\\
	&\times 2^{m-1}(1+c_5)\left[\left(\log\left(\frac{m}{\delta}\right)\right)^m+1\right]\\
	&\times \left[\log(1/\delta')+c_2\left(\frac{h}{K+h}\right)^{\alpha D_0/2-1}+c_3+c_4\log\left(\frac{k-1+h}{K-1+h}\right)\right]\\
	\leq  \;&\frac{8^{2m+1}D_1^{m+1}(1+D_4)^{m+1}u_{cM}^{2m+2}(1+c_5)}{D_0^{m+1}}\left(1+\frac{\sigma^2 D_4}{D^2}\right)\\
	&\times \alpha_k \|x_0-x^*\|_c^2\left[\left(\log\left(\frac{m}{\delta}\right)\right)^m+1\right]\\
	&\times \left[\log(1/\delta')+c_2\left(\frac{h}{K+h}\right)^{\alpha D_0/2-1}+c_3+c_4\log\left(\frac{k-1+h}{K-1+h}\right)\right]\\
	\leq  \;&c_1''\alpha_k \|x_0-x^*\|_c^2\left[\left(\log\left(\frac{m}{\delta}\right)\right)^m+1\right]\\
	&\times \left[\log(1/\delta')+c_2\left(\frac{h}{K+h}\right)^{\alpha D_0/2-1}+c_3+c_4\log\left(\frac{k-1+h}{K-1+h}\right)\right],
\end{align*}
where
\begin{align*}
    c_1''=\frac{8^{2m+1}D_1^{m+1}(1+D_4)^{m+1}u_{cM}^{2m+2}(1+c_5)}{D_0^{m+1}}\left(1+\frac{\sigma^2 D_4}{D^2}\right).
\end{align*}
The result follows by reassigning $\delta\leftarrow \delta m/(m+1)$ and $\delta'\leftarrow\delta/(m+1)$.
\end{proof}

Note that Theorem \ref{thm:improved} implies that there exists $C_1'>0$ such that
\begin{align*}
    \mathbb{P}(\sqrt{k+h}\;\| x_k - x^* \|_c> \epsilon ) < (m+1)\exp\left(-C_1'\epsilon^{2/(m+1)}\right),\quad \forall\;k\geq 0,\epsilon>0.
\end{align*}
Compared with Corollary \ref{corollary:multi} (2) of Theorem \ref{thm:multi}, we see that the rate is improved by a logarithmic factor but the tail is heavier.

\section{Proof of   Technical Results in Support of Theorem \ref{thm:impossibility}}

\subsection{Proof of Lemma \ref{le:example4}}\label{pf:le:example4}
For any $k\geq 0$, we have
\begin{align*}
	\mathbb{E}\left[\exp\left(\lambda (k+h)^{c_2} x_k^{c_1}\right)\right]=\;&\int_0^\infty\mathbb{P}\left(\exp\left(\lambda (k+h)^{c_2} x_k^{c_1}\right)>x\right)dx\\
	\leq  \;&e+\int_e^\infty\mathbb{P}\left(\exp\left(\lambda (k+h)^{c_2} x_k^{c_1}\right)>x\right)dx\\
	=\;&e+\int_e^\infty\mathbb{P}\left((k+h)^{c_2}x_k^{c_1}>\log(x)/\lambda\right)dx\\
	\leq \;&e+\int_e^\infty C_1 \exp\left(-\frac{C_2\log(x)}{ \lambda }\right)dx\\
	=\;&e+\int_e^\infty C_1x^{-\frac{C_2}{\lambda }}dx\\
	=\;&e+\frac{C_1\lambda }{\lambda -C_2}\left.x^{1-\frac{C_2}{\lambda}}\right|_e^\infty\\
    =\;&e+\frac{C_1\lambda }{C_2-\lambda}e^{1-\frac{C_2}{\lambda}}
\end{align*}
where the last line follows from $\lambda <C_2$.

\subsection{Proof of Lemma \ref{le:numerical}}\label{pf:le:numerical}
The derivative of $\ell(\cdot)$ is given by $\ell'(x)=e^x-c$, which is strictly positive when $x>\ln(c)$ and strictly negative when $x<\ln(c)$. Therefore, $\ell(\cdot)$ is a strictly decreasing function on $[0,\ln(c))$, and a strictly increasing function on $(\ln(c),\infty)$. In addition, since $\ell(0)=0$, $\ell(\ln(c))<0$, and $\lim_{x\rightarrow\infty}\ell(x)=\infty$, there equation $\ell(x)=0$ has a unique solution (denoted by $x_c$) on $(0,\infty)$. As a result, $\ell(x)\leq 0$ on $[0,x_c]$ and $\ell(x)\geq  0$ on $[x_c,\infty)$.

\section{Proof of Technical Results in Support of Theorem \ref{thm:additive}}
\subsection{Proof of Lemma  \ref{le:bound_MGF_additive}}\label{pf:le:additive_solving_recursion}

Repeatedly using Eq. (\ref{prop:additive_recursion-1}), we have for all $k\geq 0$ that
\begin{align}\label{eqeq:12}
    Z_k
    \leq \prod_{j=0}^{k-1} \frac{\alpha_j}{\alpha_{j+1}}\left(1-\frac{\Bar{D}_1\alpha_j}{2}\right)Z_0+\frac{\Bar{D}_1\Bar{D}_4}{4\Bar{D}_3}\sum_{i=0}^{k-1}\alpha_i\prod_{j=i+1}^{k-1} \frac{\alpha_j}{\alpha_{j+1}}\left(1-\frac{\Bar{D}_1\alpha_j}{2}\right).
\end{align}
We next bound the two terms on the right-hand side of the previous inequality. For the first term, we have
\begin{align*}
    \prod_{j=0}^{k-1} \frac{\alpha_j}{\alpha_{j+1}}\left(1-\frac{\Bar{D}_1\alpha_j}{2}\right)
    =\;& \frac{\alpha_0}{\alpha_k}\prod_{j=0}^{k-1} \left(1-\frac{\Bar{D}_1\alpha_j}{2}\right)\\
    =\;& \left(\frac{k+h}{h}\right)^z\prod_{j=0}^{k-1} \left(1-\frac{\Bar{D}_1\alpha}{2(k+h)^z}\right)\\
    \leq\;& \left(\frac{k+h}{h}\right)^z\exp\left(-\sum_{j=0}^{k-1}\frac{\Bar{D}_1\alpha}{2(k+h)^z}\right)\\
    \leq \;&\left(\frac{k+h}{h}\right)^z\exp\left(-\frac{\Bar{D}_1\alpha}{2} \int_0^k\frac{1}{(x+h)^z}dx\right)\\
    =\;&\begin{dcases}
    \left(\frac{h}{k+h}\right)^{\Bar{D}_1\alpha/2-1},&z=1,\\
    \left(\frac{k+h}{h}\right)^z\exp\left(-\frac{\Bar{D}_1\alpha}{2(1-z)}((k+h)^{1-z}-h^{1-z})\right),&z\in (0,1).
    \end{dcases}
\end{align*}
For the second term on the right-hand side of Eq. (\ref{eqeq:12}), 
we have
\begin{align}
    \frac{\Bar{D}_1\Bar{D}_4}{4\Bar{D}_3}\sum_{i=0}^{k-1}\alpha_i\prod_{j=i+1}^{k-1} \frac{\alpha_j}{\alpha_{j+1}}\left(1-\frac{\Bar{D}_1\alpha_j}{2}\right)
    =\;& \frac{\Bar{D}_1\Bar{D}_4}{4\Bar{D}_3} \sum_{i=0}^{k-1}\alpha_i\frac{\alpha_{i+1}}{\alpha_k}\prod_{j=i+1}^{k-1} \left(1-\frac{\Bar{D}_1\alpha_j}{2}\right)\nonumber\\
    \leq \;& \frac{\Bar{D}_1\Bar{D}_4}{4\Bar{D}_3 \alpha_k} \sum_{i=0}^{k-1}\alpha_i^2\prod_{j=i+1}^{k-1} \left(1-\frac{\Bar{D}_1\alpha_j}{2}\right),\label{eq:sum_product}
\end{align}
where the last line follows from $\{\alpha_k\}$ being a decreasing sequence. The term on the right-hand side of Eq. (\ref{eq:sum_product}) appears frequently in existing literature studying iterative algorithms, and has been analyzed in detail. For example, it was shown in \cite[Appendix A.3.7.]{chen2021finite} that
\begin{align*}
    \sum_{i=0}^{k-1}\alpha_i^2\prod_{j=i+1}^{k-1} \left(1-c' \alpha_j\right)\leq \begin{dcases}
    \frac{4e\alpha}{c'\alpha-1}\alpha_k,&z=1,\alpha>1/c',\\
    \frac{2\alpha_k}{c'},&z\in (0,1),\alpha>0, h\geq \left(\frac{4z}{\Bar{D}_1\alpha}\right)^{1/(1-z)}
    \end{dcases}
\end{align*}
for all $k\geq 0$, where $c'\in (0,1/\alpha_0)$ is any constant.
Therefore, we have
\begin{align*}
    \frac{\Bar{D}_1\Bar{D}_4}{4\Bar{D}_3 \alpha_k} \sum_{i=0}^{k-1}\alpha_i^2\prod_{j=i+1}^{k-1} \left(1-\frac{\Bar{D}_1\alpha_j}{2}\right)
    \leq \;&
    \begin{dcases}
    \frac{e\Bar{D}_1\Bar{D}_4\alpha}{\Bar{D}_3(\Bar{D}_1\alpha/2-1)},&z=1,\alpha>2/\Bar{D}_1,\\
    \frac{\Bar{D}_4}{\Bar{D}_3},&z\in (0,1),\alpha>0,
    \end{dcases}.
\end{align*}
Using the upper bounds we obtained for the two terms on the right-hand side of Eq. (\ref{eqeq:12}) , we have
\begin{align*}
    Z_k\leq 
    \begin{dcases}
    Z_0\left(\frac{h}{k+h}\right)^{\Bar{D}_1\alpha/2-1}+\frac{e\Bar{D}_1\Bar{D}_4\alpha}{\Bar{D}_3(\Bar{D}_1\alpha/2-1)},&z=1,\alpha>2/\Bar{D}_1,\\
    Z_0\left(\frac{k+h}{h}\right)^z\!\exp\left(-\frac{\Bar{D}_1\alpha ((k+h)^{1-z}\!-\!h^{1-z})}{2(1-z)}\right)\!+\!\frac{\Bar{D}_4}{\Bar{D}_3},&z\in (0,1), \alpha>0.
    \end{dcases}
\end{align*}
The proof is complete.

\subsubsection{The Approach Based on Supermartingale + Ville's Maximal Inequality}\label{ap:pf:S+V}
Let $\overline{M}(k)=\exp(\lambda_k M(x_k-x^*)-\Bar{D}_5\sum_{i=0}^{k-1}\alpha_i) $
for all $k\geq 0$, where $\Bar{D}_5=\Bar{D}_1\Bar{D}_4/(4\Bar{D}_3)$. Next, we show that $\{\overline{M}(k)\}$ is a supermartingale with respect to the filtration $\mathcal{F}_k$. It is clear that $\overline{M}(k)$ is measurable with respect to $\mathcal{F}_k$, and is finite in expectation (cf. Proposition \ref{prop:additive_recursion}). In addition, for any $k\geq 0$, we have by Eq. (\ref{eq:prop:additive_recursion}) that
\begin{align*}
    \mathbb{E}[\overline{M}(k+1)\mid\mathcal{F}_k]=\;&\mathbb{E}\left[\exp\left(\lambda_{k+1} M(x_{k+1}-x^*)-\Bar{D}_5\sum_{i=0}^k\alpha_i\right)\;\middle|\;\mathcal{F}_k\right]\\
    \leq \;&\exp\left(\lambda_k M(x_k-x^*)+\frac{\Bar{D}_1\Bar{D}_4}{4\Bar{D}_3}\alpha_k-\Bar{D}_5\sum_{i=0}^k\alpha_i\right)\\
    = \;&\exp\left(\lambda_k M(x_k-x^*)-\Bar{D}_5\sum_{i=0}^{k-1}\alpha_i\right)\\
    =\;& \overline{M}(k).
\end{align*}
Therefore, $\{\overline{M}(k)\}$ is a supermartingale.

Now we are ready to use Ville's maximal inequality. For any $K\geq 0$ and $\epsilon>0$, we have
\begin{align*}
    &\mathbb{P}\left(\sup_{k\geq K}\left\{\lambda_kM(x_k-x^*)-\Bar{D}_5\sum_{i=0}^{k-1}\alpha_i\right\}> \epsilon\right)\\
   = \;&\mathbb{P}\left(\sup_{k\geq K}\left\{\exp\left(\lambda_k M(x_k-x^*)-\Bar{D}_5\sum_{i=0}^{k-1}\alpha_i\right)\right\}> e^\epsilon\right)\\
   \leq \;&\mathbb{E}\left[\exp\left(\lambda_K M(x_K-x^*)-\Bar{D}_5\sum_{i=0}^{K-1}\alpha_i-\epsilon\right)\right]\\
   =\;&\exp\left(Z_K-\Bar{D}_5\sum_{i=0}^{K-1}\alpha_i-\epsilon\right).
\end{align*}
Let $\delta=\exp(Z_k-\Bar{D}_5\sum_{i=0}^{K-1}\alpha_i-\epsilon)$, which implies $\epsilon=\log(1/\delta)+Z_K-\Bar{D}_5\sum_{i=0}^{K-1}\alpha_i$.
Then, the previous inequality reads
\begin{align*}
   M(x_k-x^*)\leq \frac{1}{\lambda_k}\left(\log(1/\delta)+Z_K+\Bar{D}_5\sum_{i=K}^{k-1}\alpha_i\right),\;\forall\;k\geq K.
\end{align*}
with probability at least $1-\delta$.
Since
\begin{align*}
    \frac{1}{2\ell_{cM}^2}\|x_k-x^*\|_c^2\geq M(x_k-x^*)=\frac{1}{2}\|x_k-x^*\|_M^2\geq \frac{1}{2u_{cM}^2}\|x_k-x^*\|_c^2
\end{align*}
and
\begin{align*}
    \sum_{i=K}^{k-1}\alpha_i
    \leq\;& \int_{K-1}^{k-1}\frac{\alpha}{(x+h)^z}dx\\
    =\;&\begin{dcases}
    \alpha\log\left(\frac{k-1+h}{K-1+h}\right),&z=1,\\
    \frac{\alpha}{1-z}\left((k-1+h)^{1-z}-(K-1+h)^{1-z}\right),&z\in (0,1),
    \end{dcases}
\end{align*}
by using the explicit upper bound of $Z_K$ established in Proposition \ref{prop:additive_recursion}, we have the following results.
\begin{enumerate}[(1)]
    \item When $z=1$ and $\alpha>2/\Bar{D}_1$, for any $K\geq 0$, we have with probability at least $1-\delta$ that
\begin{align*}
    \|x_k-x^*\|_c^2\leq\;& \frac{16\Bar{D}_3u_{cM}^2\alpha\log(1/\delta)}{\Bar{D}_1(k+h)}+\frac{u_{cM}^2}{\ell_{cM}^2}\|x_0-x^*\|_c^2 \frac{h^{\Bar{D}_1\alpha/2}}{(k+h)(K+h)^{\Bar{D}_1\alpha/2-1}}\\
    &+\frac{16e\Bar{D}_4u_{cM}^2\alpha^2}{(\Bar{D}_1\alpha/2-1)(k+h)}+\frac{16\Bar{D}_3\Bar{D}_5u_{cM}^2\alpha^2}{\Bar{D}_1(k+h)}\log\left(\frac{k-1+h}{K-1+h}\right).
\end{align*}
\item When $z\in (0,1)$, for any $K\geq 0$, we have with probability at least $1-\delta$ that
\begin{align*}
    \|x_k-x^*\|_c^2
    \leq\;& \frac{16\Bar{D}_3u_{cM}^2\alpha\log(1/\delta)}{\Bar{D}_1(k+h)^z}+\frac{16\Bar{D}_4u_{cM}^2\alpha}{\Bar{D}_1(k+h)^z}\\
    &+\frac{u_{cM}^2}{\ell_{cM}^2}\|x_0-x^*\|_c^2\left(\frac{K+h}{k+h}\right)^z\exp\left(-\frac{\Bar{D}_1\alpha}{2(1-z)}((K+h)^{1-z}-h^{1-z})\right)\\
    &+\frac{16\Bar{D}_3\Bar{D}_5u_{cM}^2\alpha^2}{\Bar{D}_1(1-z)}\frac{(k-1+h)^{1-z}-(K-1+h)^{1-z}}{(k+h)^z}.
\end{align*}
\end{enumerate}

\section{Proof of Technical Results in Section \ref{sec:applications}}\label{pf:prop:linear_sa}

\subsection{Proof of Lemma \ref{prop:linear_sa}}
    
\begin{enumerate}[(1)]
\item It is easy to see that $\bar{F}_\beta(x)=\beta (\bar{A}x-\bar{b})+x$. Therefore, for any $x_1,x_2\in\mathbb{R}^d$, we have
    \begin{align*}
        \|\Bar{F}_\beta(x_1)-\Bar{F}_\beta(x_2)\|_{\Bar{P}}=\|(\beta \bar{A}+I)(x_1-x_2)\|_{\bar{P}}\leq \|\beta \Bar{A}+I\|_{\Bar{P}}\|x_1-x_2\|_{\Bar{P}}.
    \end{align*}
    It remains to bound the induced matrix norm $\|\beta \Bar{A}+I\|_{\Bar{P}}$.
    Observe that
    \begin{align}
        \|\beta \Bar{A}+I_d\|_{\Bar{P}}^2=\;&\max_{x:\|x\|_{\Bar{P}}=1}x^\top (\beta\Bar{A}+I_d)^\top \Bar{P} (\beta\Bar{A}+I_d)x\nonumber\\
        =\;&\max_{x:\|x\|_{\Bar{P}}=1}x^\top (\beta^2\Bar{A}^\top \Bar{P}\Bar{A}+\beta \Bar{A}^\top \Bar{P}+\beta\Bar{P}\Bar{A}+\Bar{P})x\nonumber\\
        =\;&1+\max_{x:\|x\|_{\Bar{P}}=1}x^\top (\beta^2\Bar{A}^\top \Bar{P}\Bar{A}-\beta I_d)x\label{eq1:lemma-linearSA}\\
        \leq \;&1+\max_{x:\|x\|_{\Bar{P}}=1}\|x\|_2^2 (\beta^2\lambda_{\max}(\Bar{A}^\top \Bar{P}\Bar{A} )-\beta)\nonumber\\
        =\;&1-\frac{\max_{x:\|x\|_{\Bar{P}}=1}\|x\|_2^2}{4\lambda_{\max}(\Bar{A}^\top \Bar{P}\Bar{A} )}\label{eq2:lemma-linearSA}\\
        \leq \;&1-\frac{1}{4\lambda_{\max}(\bar{P})\lambda_{\max}(\Bar{A}^\top \Bar{P}\Bar{A} )},\label{eq3:lemma-linearSA}
    \end{align}
    where Eq. (\ref{eq1:lemma-linearSA}) follows from the Lyapunov equation $\bar{A}^\top \bar{P}+\bar{P}\bar{A}+I_d=0$ and Eq. (\ref{eq2:lemma-linearSA}) follows from choosing $\beta= \lambda^{-1}_{\max}(\Bar{A}^\top \Bar{P}\Bar{A})/2$. It follows that $\Bar{F}_\beta(\cdot)$ is a contraction mapping with respect to $\|\cdot\|_{\Bar{P}}$, with contraction factor
    \begin{align*}
        \Bar{\gamma}=\left(1-\frac{1}{4\lambda_{\max}(\bar{P})\lambda_{\max}(\Bar{A}^\top \Bar{P}\Bar{A} )}\right)^{1/2}.
    \end{align*}
    \item Since $\{Y_k\}$ is an i.i.d. sequence, we have
    \begin{align*}
        \mathbb{E}[F_\beta(x_k,Y_k)\mid \mathcal{F}_k]=\beta \Bar{A}x_k-\beta \Bar{b}+x_k=\Bar{F}_\beta(x_k).
    \end{align*}
    \item Under Assumption \ref{as:linear_sa_boundedness}, we have for all $k\geq 0$ that
    \begin{align*}
        \|F_\beta(x_k,Y_k)-\Bar{F}_\beta(x_k)\|_2=\;&\beta\| (A(Y_k)-\Bar{A})x_k- (b(Y_k)-\Bar{b})\|_2\\
        \leq \;&\beta (\|A(Y_k)\|_2+\|\Bar{A}\|_2)\|x_k\|_2+\beta(\|b(Y_k)\|_2+\|\Bar{b}\|_2)\\
        \leq \;&2\beta A_{\max }\|x_k\|_2+2\beta b_{\max}.
    \end{align*}
    Therefore, we have for all $k\geq 0$ that
    \begin{align*}
        \|F_\beta(x_k,Y_k)-\Bar{F}_\beta(x_k)\|_{\Bar{P}}\leq\;& \lambda_{\max}(\Bar{P})\|F_\beta(x_k,Y_k)-\Bar{F}_\beta(x_k)\|_2\\
        \leq \;& 2\beta\lambda_{\max}(\Bar{P})(A_{\max}\|x_k\|_2+b_{\max})\\
        \leq \;& 2\beta\lambda_{\max}(\Bar{P})(A_{\max}\|x_k\|_{\Bar{P}}/\lambda_{\min}(\Bar{P})+b_{\max})\\
        \leq \;&2\beta \lambda_{\max}(\Bar{P})(A_{\max}/\lambda_{\min}(\Bar{P})+b_{\max})(\|x_k\|_{\Bar{P}}+1).
    \end{align*}
    The result follows by letting $\hat{\sigma}=2\beta \lambda_{\max}(\Bar{P})(A_{\max}/\lambda_{\min}(\Bar{P})+b_{\max})$.
\end{enumerate}

\subsection{Proof of Lemma \ref{le:TD}}\label{pf:le:TD}
In the on-policy setting, the Hurwitzness of $\bar{A}$ has been shown in \cite{tsitsiklis1997analysis} for the more challenging case of TD$(\lambda)$. The same conclusion holds for $n$-step TD-learning. 

In the off-policy setting, let 
\begin{align*}
    D_{\rho,\max}=\;&\max_{s\in\mathcal{S}}\sum_{a\in\mathcal{A}}\pi_b(a\mid s)\rho(s,a),& D_{\rho,\min}=\;&\min_{s\in\mathcal{S}}\sum_{a\in\mathcal{A}}\pi_b(a\mid s)\rho(s,a),\\
    D_{c,\max}=\;&\max_{s\in\mathcal{S}}\sum_{a\in\mathcal{A}}\pi_b(a\mid s)c(s,a),& D_{c,\min}=\;&\min_{s\in\mathcal{S}}\sum_{a\in\mathcal{A}}\pi_b(a\mid s)c(s,a).
\end{align*}
It has been shown in \cite[Proposition 4.1]{chen2022sample} that as long as $0<c(s,a)\leq \rho(s,a)$ for all $(s,a)$, $ D_{\rho,\max}<1/\gamma$, $\frac{\gamma (D_{\rho,\max}-D_{c,\min})}{(1-\gamma D_{c,\min})\sqrt{\min_{s,a}\kappa_b(s)\pi_b(a\mid s)}}<1$, and the parameter $n$ is chosen such that $1-\frac{(1-\gamma D_{\rho,\max})(1-(\gamma D_{c,\min})^n)}{1-\gamma D_{c,\min}}\leq \min_{s,a}\sqrt{\kappa_b(s)\pi_b(a\mid s)}$, the matrix $\bar{A}$ is Hurwitz.

\subsection{Proof of Lemma \ref{prop:Q-learning}}\label{pf:prop:Q-learning}

We first compute the explicit expression of $\bar{F}(\cdot)=\mathbb{E}[F(\cdot,S,A,S')]$, where $S\sim \kappa_b(\cdot)$, $A\sim \pi_b(\cdot|S)$, and $S'\sim P_A(S,\cdot)$. Using the explicit expression of $F(\cdot)$ , we have
\begin{align*}
    &[\bar{F}(Q)](s,a)\\
    =\;&\mathbb{E}\left[\mathds{1}_{\{(S,A)=(s,a)\}}(\mathcal{R}(S,A)+\gamma \max_{a'\in\mathcal{A}}Q(S',a')-Q(S,A))+Q(s,a)\right]\\
    =\;&\kappa_b(s)\pi_b(a|s)\left(\mathcal{R}(s,a)+\gamma \sum_{s'\in\mathcal{S}}P_a(s,s')\max_{a'\in\mathcal{A}}Q(s',a')\right)\\
    &+(1-\kappa_b(s)\pi_b(a|s))Q(s,a)\\
    =\;&(1-D_b((s,a),(s,a)))[\mathcal{H}(Q)](s,a)+(1-D_b((s,a),(s,a)))Q(s,a),
\end{align*}
where $\mathcal{H}(\cdot)$ is the Bellman optimality operator of the $Q$-function.
It follows that
\begin{align*}
    \bar{F}(Q)=D_b\mathcal{H}(Q)+(I-D_b)Q.
\end{align*}
\begin{enumerate}[(1)]
    \item Due to the $\ell_\infty$-norm contraction property of $\mathcal{H}(\cdot)$. The contraction property of $\bar{F}(\cdot)$ was established in existing literature \cite{chen2021finite}. Specifically, it was shown in \cite[Proposition 3.1]{chen2021finite} that $\bar{F}(\cdot)$ is a contraction operator with respect to $\ell_\infty$-norm, with contraction factor
    \begin{align*}
        \hat{\gamma}_c=1-D_{b,\min}(1-\gamma).
    \end{align*}
    \item The unbiasedness follows from our definition of $\bar{F}(\cdot)$ and the fact that $\{(S_k,A_k,S_k')\}$ is an i.i.d. sequence.
    \item To begin with, it was shown in the literature using an induction argument \cite{gosavi2006boundedness} that the iterates of $Q$-learning admit a deterministic uniform bound: $\|Q_k\|_\infty\leq 1/(1-\gamma)$. Now for any $(s,a)$ and $k\geq 0$, we have
    \begin{align*}
        F(Q_k,S_k,A_k,S_k')-\mathbb{E}[F(Q_k,S_k,A_k,S_k')\mid \mathcal{F}_k]
        = F(Q_k,S_k,A_k,S_k')-\bar{F}(Q_k).
    \end{align*}
    If $(S_k,A_k)\neq (s,a)$, we have
    \begin{align*}
        &|[F(Q_k,S_k,A_k,S_k')](s,a)-[\bar{F}(Q_k)](s,a)|\\
        =\;&|D_b((s,a),(s,a))(Q_k(s,a)-[\mathcal{H}(Q_k)](s,a))|\\
        \leq \;&\|Q_k-\mathcal{H}(Q_k)\|_\infty\\
        \leq \;&\|Q_k-Q^*\|_\infty+\|\mathcal{H}(Q^*)-\mathcal{H}(Q_k)\|_\infty\\
        \leq\;&2\|Q_k-Q^*\|_\infty\tag{the contraction property of $\mathcal{H}(\cdot)$}\\
        \leq \;&\frac{2}{1-\gamma}.
    \end{align*}
    If $(S_k,A_k)= (s,a)$, we have
    \begin{align*}
        &|[F(Q_k,S_k,A_k,S_k')](s,a)-[\bar{F}(Q_k)](s,a)|\\
        =\;&
        |\mathcal{R}(s,a)+\gamma\max_{a'\in\mathcal{A}}Q_k(S_k',a')-D_b((s,a),(s,a))[\mathcal{H}(Q_k)](s,a)\\
        &+(1-D_b((s,a),(s,a)))Q_k(s,a)|\\
        \leq  \;&\frac{4}{1-\gamma}.
    \end{align*}
    Combine these two cases , we have
    \begin{align*}
        \|F(Q_k,S_k,A_k,S_k')-\mathbb{E}[F(Q_k,S_k,A_k,S_k')\mid \mathcal{F}_k]\|_\infty\leq \frac{4}{1-\gamma}.
    \end{align*}
    As a result, for any random vector $v\in\mathbb{R}^d$ that is measurable with respect to $\mathcal{F}_k$, we have
    \begin{align*}
        \mathbb{E}[\langle F(Q_k,S_k,A_k,S_k')-\mathbb{E}[F(Q_k,S_k,A_k,S_k')\mid \mathcal{F}_k],v\rangle\mid \mathcal{F}_k] =0\quad \text{w.p. $1$}
    \end{align*}
    and
    \begin{align*}
        \langle F(Q_k,S_k,A_k,S_k')-\mathbb{E}[F(Q_k,S_k,A_k,S_k')\mid \mathcal{F}_k],v\rangle\leq \frac{4\|v\|_1}{1-\gamma}\quad \text{w.p. $1$}
    \end{align*}
    It then follows from the conditional Hoeffding's lemma that
    \begin{align*}
        \mathbb{E}\left[\exp\left(\lambda \langle F(x_k,Y_k)-\mathbb{E}\left[F(x_k,Y_k) | \mathcal{F}_k \right], v \rangle \right) \middle| \mathcal{F}_k \right]\leq \exp\left(\frac{8\lambda^2\|v\|_1^2}{(1-\gamma)^2}\right)
    \end{align*}
    for all $\lambda>0$.
    Moreover, we have
    \begin{align*}
        \mathbb{E}\left[\exp\left(\lambda\| F(x_k,Y_k)-\mathbb{E}\left[F(x_k,Y_k) | \mathcal{F}_k \right]\|_\infty^2\right) \middle| \mathcal{F}_k \right]\leq\;& \exp\left(\frac{16\lambda}{(1-\gamma)^2}\right)\\
        \leq\;& \left(1-\frac{32\lambda}{(1-\gamma)^2}\right)^{-1/2},
    \end{align*}
    where the last line follows from $(1-2x)^{1/2}\leq e^{-x}$ for all $x\geq 0$.
    Therefore, Assumption \ref{ass:sub-Gaussian} is satisfied with $\bar{\sigma}=4/(1-\gamma)$ and $c_d=1$.
\end{enumerate}

\subsection{Proof of Theorem \ref{thm:Q-learning}}\label{pf:thm:Q-learning}

We only need to find the dependence of the constants $\{\bar{c}_i\}_{1\leq i\leq 4}$ on the size of the state-action space and the effective horizon. Following \cite{chen2020finite} that initially proposed using the generalized Moreau envelope to study SA under arbitrary norm contraction, when
$\|\cdot\|_c=\|\cdot\|_\infty$, we choose $\|\cdot\|_s=\|\cdot\|_p$ with $p=2\log(|\mathcal{S}||\mathcal{A}|)$. It follows that $\ell_{cs}=1/\sqrt{e}$, $u_{cs}=1$, $L=p-1\leq 2\log(d)$, $\mu=((1+\hat{\gamma}_c)/2\hat{\gamma}_c)^2-1$, and $u_{cM,*}=\sqrt{e}$. 

For simplicity of notation, we use $a\lesssim b$ to mean that there exists a numerical constant $\tilde{C}$ such that $a\leq \tilde{C} b$. By definition of $\{\bar{D}_i\}_{1\leq i\leq 5}$, we have
\begin{align*}
    \bar{D}_0
    \lesssim\;&\frac{\mu(1-\gamma)^2}{\log(|\mathcal{S}||\mathcal{A}|)},\;
    \bar{D}_1\geq 1-\hat{\gamma}_c,\;\bar{D}_2
    \lesssim \frac{\log(|\mathcal{S}||\mathcal{A}|)}{1-\hat{\gamma}_c},\;\bar{D}_3\lesssim \frac{1}{(1-\gamma)^2},\\
    \bar{D}_4\lesssim \;&\frac{\log(|\mathcal{S}|
    |\mathcal{A}|)}{(1-\gamma)^2\mu},
\end{align*}
which in turn implies
\begin{align*}
    \bar{c}_1\lesssim\;& \frac{ (1+\mu) }{(1-\gamma)^2},\;\bar{c}_2\lesssim 1,\;\bar{c}_3\lesssim \frac{\log(|\mathcal{S}||\mathcal{A}|) }{(1-\gamma)^2(1-\hat{\gamma}_c)^3},\\
    \bar{c}_4\lesssim\;&\frac{\log(|\mathcal{S}||\mathcal{A}|) }{(1-\gamma)^2(1-\hat{\gamma}_c)^3},\;\bar{c}_5\lesssim\frac{\log(|\mathcal{S}||\mathcal{A}|)}{(1-\gamma)^2(1-\hat{\gamma}_c)^3}.
\end{align*}
The result follows by applying Theorem \ref{thm:additive} (1) to $Q$-learning.

\end{document}